\documentclass[11pt]{article}
\usepackage[utf8]{inputenc}
\usepackage{fullpage}
\usepackage[maxbibnames=99]{biblatex}
\usepackage{todonotes}
\usepackage{verbatim}
\usepackage[margin=1in]{geometry}

\newcommand{\Clearn}{C}

\newcommand{\epsuppbd}{10^{-5}}
\newcommand{\edgelowbd}{1/20}

\title{Chow-Liu++: Optimal Prediction-Centric \\ Learning of Tree Ising Models}

\author{Enric Boix-Adser\`a\thanks{Massachusetts Institute of Technology. Department of EECS. Email: \texttt{eboix@mit.edu}. Supported by an NSF Graduate Fellowship, a Siebel Scholarship, and NSF TRIPODS award 1740751.} \and Guy Bresler\thanks{Massachusetts Institute of Technology. Department of EECS. Email: \texttt{guy@mit.edu}. Supported by MIT-IBM Watson AI Lab and NSF CAREER award CCF-1940205.} \and Frederic Koehler\thanks{Massachusetts Institute of Technology. Department of Mathematics. Email: \texttt{fkoehler@mit.edu}. This work was supported in part by NSF CAREER Award CCF-1453261, NSF Large CCF-1565235 and Ankur Moitra's ONR Young Investigator Award.} }

\usepackage{amsmath}
\usepackage{amsthm}
\usepackage{amssymb}
\usepackage{color}
\usepackage{mathtools}
\usepackage{makecell}
\usepackage{bm}
\usepackage{todonotes}
\usepackage{diagbox}
\usepackage{tablefootnote}
\usepackage{hyperref}
\usepackage[ruled,vlined,linesnumbered]{algorithm2e}

\hypersetup{
  colorlinks   = true, %
  urlcolor     = blue, %
  linkcolor    = blue, %
  citecolor   = red %
}

\DeclareMathOperator{\pa}{pa}

\DeclareMathOperator{\E}{E}
\DeclareMathOperator{\LCA}{\mathsf{LCA}}

\DeclareMathOperator{\sgn}{sgn}

\DeclareMathOperator{\Rad}{Rad}

\newcommand{\LCAi}[3]{\mathsf{LCA}^{#1}_{#2}(#3)}

\renewcommand{\d}[2]{\frac{d #1}{d #2}} %
\let\baraccent=\= %
\renewcommand{\=}[1]{\stackrel{#1}{=}} %

\providecommand{\RR}{\mathbb{R}}

\providecommand{\cV}{\mathcal{V}}
\providecommand{\cX}{\mathcal{X}}

\providecommand{\sT}{\mathsf{T}}
\providecommand{\sT}{\mathsf{T}}
\providecommand{\PP}{\mathbb{P}}
\providecommand{\EE}{\mathbb{E}}

\providecommand{\e}{\epsilon}
\providecommand{\eps}{\epsilon}

\providecommand{\al}{\alpha}

\mathchardef\mhyphen="2D %

\providecommand{\sm}{\setminus}

\newcommand{\interior}[1]{%
  {\kern0pt#1}^{\mathrm{o}}%
}

\newtheorem{theorem}{Theorem}[section]
\newtheorem{lemma}[theorem]{Lemma}

\newtheorem{claim}[theorem]{Claim}
\newtheorem{fact}[theorem]{Fact}
\newtheorem{proposition}[theorem]{Proposition}
\newtheorem{corollary}[theorem]{Corollary}

\theoremstyle{definition}
\newtheorem{definition}[theorem]{Definition}

\newtheorem{remark}[theorem]{Remark}

\newcommand{\loctv}{\mathsf{locTV}}
\newcommand{\loctvk}[1]{\mathsf{locTV}^{(#1)}}
\newcommand{\tv}{\mathsf{TV}}
\newcommand{\muh}[1]{{\hat{\mu}_{#1}}}

\newcommand{\mup}[1]{{\mu'_{#1}}}
\newcommand{\ah}{{\hat a}}
\newcommand{\bh}{{\hat b}}

\newcommand{\TCL}{\sT^{\mathsf{CL}}} %
\newcommand{\TCLprime}{\sT^{\mathsf{CL'}}} %

\newcommand{\ReduceToFerromagnetic}{\textup{\textsc{LearnModel}}}
\newcommand{\LearnFerromagneticModel}{\textup{\textsc{LearnFerroModel}}}
\newcommand{\LearnLowerboundedModel}{\textup{\textsc{LearnLwrBddModel}}}
\newcommand{\Desteinerize}{\textup{\textsc{Desteinerize}}}

\newcommand{\AdditiveMetricReconstruction}{\textsc{AdditiveMetricReconstruction}}
\newcommand{\TreeMetricReconstruction}{\textsc{TreeMetricReconstruction}}

\newcommand{\thetat}{\tilde{\theta}}
\newcommand{\Pt}{\tilde{P}}

\renewcommand{\d}{{\mathsf{d}}}
\renewcommand{\a}{{\mathsf{a}}}
\renewcommand{\e}{{\mathsf{e}}}
\newcommand{\csf}{{\mathsf{c}}}

\newcommand{\cT}{\mathcal{T}}
\newcommand{\mut}[1]{\tilde{\mu}_{{#1}}}
\newcommand{\dt}{\tilde{\d}}
\newcommand{\dtpre}{\dt_{\mathrm{pre}}}
\renewcommand{\dh}{\hat{\d}}
\newcommand{\Th}{\widehat{\mathsf{T}}}
\newcommand{\thetah}[1]{\hat{\theta}_{{#1}}}

\newcommand{\Ph}{\hat{P}}
\newcommand{\pih}{\hat{\pi}}
\newcommand{\pis}{\pi}

\newcommand{\T}{\mathsf{T}}
\newcommand{\Ts}{{\mathsf{T}}}
\newcommand{\thetas}[1]{{\theta_{#1}}}
\newcommand{\mus}[1]{{\mu_{#1}}}
\newcommand{\picl}{{\pi^{\mathsf{CL}}}}
\newcommand{\epslearning}{\eps\mbox{-tree-learning}}

\newcommand{\Vset}{\mathcal{V}}

\SetCommentSty{mycommfont}

\newcommand{\inv}{^{-1}}

\newcommand{\crad}{C_\mathrm{radius}}

\newcommand{\bout}{b^{\mathrm{out}}}
\newcommand{\bin}{b^{\mathrm{in}}}

\newcommand{\Lupp}{L_{\mathrm{upp}}}

\newcommand{\CLpp}{\textup{\textsf{Chow-Liu++}}}

\begin{document}

\maketitle
\begin{abstract}
    We consider the problem of learning a tree-structured Ising model from data, such that subsequent predictions computed using the model are accurate. Concretely, we aim to learn a model such that posteriors $P(X_i|X_S)$ for small sets of variables $S$ are accurate. 
    Since its introduction more than 50 years ago, the Chow-Liu algorithm, which efficiently computes the maximum likelihood tree, has been the benchmark algorithm for learning tree-structured graphical models. 
    A bound on the sample complexity of the Chow-Liu algorithm with respect to the prediction-centric local total variation loss was shown in \cite{breslerkarzand19}. While those results demonstrated that it is possible to learn a useful model even when recovering the true underlying graph is impossible, their bound depends on the maximum strength of interactions and thus does not achieve the information-theoretic optimum. In this paper, we introduce a new algorithm that carefully combines elements of the Chow-Liu algorithm with tree metric reconstruction methods to efficiently and optimally learn tree Ising models under a prediction-centric loss. Our algorithm is robust to model misspecification and adversarial corruptions. In contrast, we show that the celebrated Chow-Liu algorithm can be arbitrarily suboptimal. 
\end{abstract}

\section{Introduction}

\paragraph{}

Undirected graphical models (also known as Markov random fields) are a flexible and powerful paradigm for modeling high-dimensional data, where variables are associated to nodes in a graph and edges indicate interactions between the variables.  
One of the primary strengths of graphical models is that the graph structure often facilitates efficient inference computations \cite{wainwright2008graphical}, including computation of marginals, maximum a posteriori (MAP) assignments, and posteriors. Popular algorithms for these tasks include belief propagation and other message-passing algorithms, variational methods, and Markov chain Monte Carlo. Once a graphical model representing the data of interest has been determined, the prototypical use scenario entails observing values for a subset $S$ of the variables, computing posteriors of the form $P(X_i|X_S)$ for $i$ in some set $I$, and then forming predictions for the variables in $I$. 

In modern applications an approximate graphical model must typically be learned from data, and
it therefore makes sense to evaluate the learned model with respect to accuracy of posteriors $P(X_i|X_S)$ \cite{heinemann2014inferning,wainwright2006estimating,breslerkarzand19}. 
Aside from facilitating computation of posteriors and other inference tasks, the underlying structure of a graphical model is also sometimes of interest in its own right, for instance in certain biological or economic settings. In this context it makes sense to try to reconstruct the graph underlying the model assumed to be generating the data, and most papers on the topic have used the 0-1 loss (equal to zero if the learned graph $\widehat G$ is equal to the true graph $G$ and one otherwise). In this paper we will be primarily concerned with learning a model such that posteriors $P(X_i|X_S)$ are accurate, even if the structure is not, but we will also deduce results for reconstructing the graph.

The posteriors $P(X_i|X_S)$ of interest often have small conditioning sets $S$, for instance in a model capturing a user's preferences (like or dislike) for items in a recommendation system, the set $S$ would consist of items for which a user has already provided feedback.
It turns out that a \emph{local} version of total variation, as studied in \cite{rebeschini2015can,breslerkarzand19}, captures accuracy of posteriors $P(X_i|X_S)$ averaged over the values $X_S$. Specifically, let
\begin{equation}\label{eq:loctv}
    \loctvk{k}(P,Q) = \max_{|A|=k} \tv(P_A,Q_A)
\end{equation}
be the maximum over sets of variables $A$ of size $k$ of the total variation between the marginals $P_A$ and $Q_A$. As noted in Section~\ref{sec:predictionLearning}, this quantity controls accuracy of posteriors $P(X_i|X_S)$ for $|S|\leq k-1$.

Instead of learning a model close in local total variation, one might use (global) total variation as done in the recent papers \cite{bhattacharyya2020near,daskalakis2020tree} which determined the optimal sample complexity for learning tree-structured models to within small total variation. Total variation distance also translates to a guarantee on the accuracy of posteriors $P(X_i|X_S)$, but does not exploit the small size of $S$. Indeed, as we show in Section~\ref{sec:TVlower}, learning to within small total variation has a high sample complexity of at least $\Omega(n\log n)$ for trees on $n$ nodes, and does not yield meaningful implications in the high-dimensional setting of modern interest (where number of samples is much smaller than number of variables). Furthermore, both \cite{bhattacharyya2020near} and \cite{daskalakis2020tree} are based on the Chow-Liu algorithm, and as we will show in this paper, the Chow-Liu algorithm can be arbitrarily suboptimal with respect to local total variation. Thus, accuracy with respect to local versus global total variation results in a completely different learning problem with distinct algorithmic considerations.

There is a large literature on structure learning  (see \cite{ravikumar2010high,vuffray2016interaction,bresler2015efficiently,klivans2017learning,wu2019sparse,bresler2019learning,kelner2019learning,goel2019learning,shah2021learning,} and references therein), but prediction-centric learning remains poorly understood.
In this paper, we focus on tree Ising models. In this case, the minimax optimal algorithm for structure learning is the well-known Chow-Liu algorithm introduced more than 50 years ago (\cite{ChowL,breslerkarzand19,tan2011large}), which efficiently computes the maximum likelihood tree. 
Learning tree Ising models with respect to the prediction-centric $\loctv$ was studied in \cite{breslerkarzand19}, where the main take-away was that it is possible to learn a good model even when there is insufficient data to recover the underlying graph. They proved upper and lower bounds on the sample complexity, but left a sizable gap. Their upper bounds were based on analyzing the Chow-Liu algorithm, and left open several core questions:
Can the optimal sample complexity for learning trees under the $\loctv$ loss be achieved by an efficient algorithm? In particular, the guarantee in \cite{breslerkarzand19} depends on the maximum edge strength of the model generating the data -- is this necessary? Finally, \cite{breslerkarzand19} assumed that the observed data was from a tree Ising model. Can one give sharp learning guarantees under model misspecification, where the learned model is a tree Ising model but the model generating the samples need not be? More broadly, can one robustly learn from data subject to a variety of noise processes, including from adversarially corrupted data? 
An initial foray towards learning with noise was made in \cite{nikolakakis2021predictive}, but their approach is based on modifying the techniques of \cite{breslerkarzand19}, which as discussed below cannot yield sharp results.

\subsection{Our Contribution} In this paper we give an algorithm, \CLpp, that is \emph{simultaneously optimal} for both the prediction-centric $\loctv$ loss defined in \eqref{eq:loctv} and the structure learning $0$-$1$ loss, and that is also robust to model misspecification and adversarial corruptions. Our algorithm takes as input the $n^2$ pairwise correlations between variables and has $O(n^2)$ runtime, linear in the input.

\begin{theorem}[Informal statement of Theorem~\ref{t:CL++}]
If $P$ is the distribution of a tree Ising model $(\Ts, \thetas{})$ on $n$ nodes with correlations $\mu_{ij} = \EE_P[X_i X_j]$, then when given as input approximate correlations $\mut{ij}$ satisfying $|\mut{ij} - \mu_{ij}|\leq \eps$ the
{\CLpp } algorithm outputs a tree Ising model $(\Th,\thetah{})$ with distribution $\Ph$ such that
$
\loctvk2(\Ph,P)\leq C\eps
$. 
\end{theorem}

\begin{corollary} Given data from a distribution $P$, {\CLpp} learns a tree Ising model with distribution $\Ph$ such that $\loctvk k(\Ph,P)$ is within a constant factor of optimal. Moreover
\begin{enumerate}
    \item[(i)] {\CLpp} is robust to model misspecification, or equivalently, adversarial corruptions of a fraction of the data.
    \item[(ii)] In the case that $P$ is a tree Ising model, then {\CLpp} returns the correct tree with optimal number of samples up to a constant factor.
\end{enumerate}

\end{corollary}

These results, as well as others, are stated more formally in Section~\ref{sec:contrib}.

\subsection{Techniques and Main Ideas}\label{ssec:techniques}

Our {\CLpp } algorithm combines the Chow-Liu algorithm with tree metric reconstruction methods, together with new analysis for both. Our main insight is that the Chow-Liu algorithm and tree metric algorithms work in complementary ranges of interaction strengths -- by carefully combining them we obtain an algorithm that works regardless of the minimum and maximum edge strengths. Somewhat counterintuitively, we use the Chow-Liu algorithm indirectly to help us to partition the problem instance and only use actual edges from the algorithm that may well be incorrect.  Below, we highlight some of the main challenges and the ideas we introduce to overcome them. We focus on the case of $k=2$ in the loss $\loctvk k$, since the result for general $k > 2$ can be obtained by means of an inequality for tree-structured distributions proved in \cite{breslerkarzand19} that bounds the $\loctvk{k}$ distance in terms of the $\loctvk{2}$ distance.

\paragraph{Challenge 1: Failure of Chow-Liu.}
The main difficulty in proving our result is that the classical Chow-Liu algorithm, which is well-known to be optimal for structure learning and was recently shown to be optimal for total variation \cite{bhattacharyya2020near,daskalakis2020tree}, does not work for learning to within $\loctv$. Concretely, we show that the Chow-Liu algorithm is not robust to model misspecification -- i.e., if the model is not a tree-structured model, but is instead only \emph{close} in $\loctv$ to a tree model -- then Chow-Liu can fail. We prove this as Theorem~\ref{thm:chowliufailsmodelmisspecification}.

\begin{theorem}[Informal statement of Theorem~\ref{thm:chowliufailsmodelmisspecification}]
Chow-Liu fails with model misspecification.
\end{theorem}

To prove this theorem we construct a distribution that is close to a tree Ising model in \textit{local} total variation, such that even when given population (infinite sample) values for the correlations
the Chow-Liu algorithm makes local errors that accumulate at global scales. This issue does not arise in the analysis of \cite{bhattacharyya2020near,daskalakis2020tree}, since they assume a model is $\eps$-close in total variation to a tree model, and thus the input distribution is extremely close to a tree model in a global sense.
Furthermore, in converting this to a finite sample result they assume access to at least $n\log n$ samples, which results in extremely small local estimation errors that do not appreciably accumulate.

\paragraph{Main Idea 1: Structural lemma for Chow-Liu.} 
When there is insufficient data for the Chow-Liu algorithm (or any other algorithm) to reconstruct the correct tree, it turns out that there is still a precise sense in which the output of the Chow-Liu algorithm \emph{resembles} the correct tree. We will momentarily describe the notion of approximation upon which our structural lemma is based, but we first explain specifically which issue it addresses.

Suppose that all edge correlations of the model generating the samples lie in the interval $[\alpha, 1-\beta]$. 
As shown in~\cite{breslerkarzand19}, learning the structure of tree Ising models requires more samples as 
$\alpha$ or $\beta$ get smaller, and no algorithm can reconstruct the exact structure of the tree given $O((\al\beta)^{-2}\log n)$ samples. Turning to the $\loctv$ loss, our counterexample in Theorem~\ref{thm:chowliufailsmodelmisspecification} shows that strong edges remain problematic for Chow-Liu. 
However, surprisingly, it turns out that Chow-Liu is useful in a precise sense, even in the presence of weak edges that preclude structure learning, as shown by our structural lemma. See Section~\ref{ssec:chowliuoverview} for more details.  

\begin{lemma}[Structural Lemma for Chow-Liu]\label{lem:struct}
Let $\TCL$ be the tree returned by the Chow-Liu algorithm and $W$ be the subset of weak edges with empirical correlation $\leq 1/10$. Let $V_1,\ldots,V_m$ denote the connected components of vertices in $\TCL$ upon removal of edges in $W$. Then
\begin{enumerate}
    \item[(i)] The sets of vertices $V_1,\ldots,V_m$ are each connected within the true tree $\Ts$, and
    \item[(ii)] If we contract the vertex sets $V_1,\ldots,V_m$ in $\TCL$, then we obtain the same tree as when we contract these vertex sets in $\Ts$, as long as all edges have correlation $\Omega(\eps)$.
\end{enumerate}
\end{lemma}

In other words, the weak edges learned by Chow-Liu correctly partition the vertices into connected components. The subgraphs on these connected components can be badly erroneous, and the weak edges themselves may be between incorrect vertices, but they are between the correct \emph{sets} of vertices. The constant $1/10$ in the definition of weak edges is an arbitrary, albeit sufficiently small, constant. Note also that item (ii) of the structural lemma comes with the qualifier that all edges have correlation at least of order $\epsilon$, where $\eps$ as in Theorem~\ref{t:CL++} above denotes the accuracy of approximate correlations. The other case turns out to be easy to deal with via a direct argument, because all correlations along paths that include the edge are very weak.

The next two challenges address separately the problem of global error accumulation due to incorrectly placed weak edges and accurate reconstruction within components.

\paragraph{Challenge 2: Accumulation of errors on long paths.}

The structural lemma that we prove for Chow-Liu has seemingly nothing to do with the objective of reconstructing the tree in $\loctvk{k}$ distance -- it is merely a combinatorial result. It is necessary to leverage this structural result to resolve the core issue of error accumulation. 

\paragraph{Main Idea 2: Error composition theorem for Chow-Liu.} 
We postpone the challenge of accurately reconstructing a model within each component $V_i, i\in [m]$, to be dealt with in a moment.
Our composition lemma states that if for each $i \in [m]$ we have a good $\loctvk 2$ reconstruction $(\T_i', \theta_i')$ of the model on the subtree $\Ts[V_i]$ induced by $V_i$, then we can stitch these together along with the weak edges $W$ of the Chow-Liu tree defined in Lemma~\ref{lem:struct}, in order to obtain a good $\loctvk 2$ reconstruction of the overall tree.

\begin{lemma}[Informal statement of composition lemma for Chow-Liu (Proposition~\ref{prop:preprocessing})]
Suppose that for each $i \in [m]$, $(\T'_i, \theta'_i)$ is $\eps$-close in $\loctvk 2$ distance to the tree model $(\T, \theta)$ on vertex set $V_i$.
Then the tree model $(\Th, \thetah{})$ with tree $\Th = W \cup (\cup_{i \in [m]} \T_i')$ and edge weights $\thetah{} = \thetat_{W} \cup (\cup_{i \in [m]} \theta'_i)$ is $O(\eps)$-close in $\loctvk 2$ distance to $(\Ts, \theta)$.
\end{lemma}
The proof of this composition result requires careful control over accumulation of errors along paths. Consider vertices $u,v$ and the paths $\pih(u,v)$ and $\pis(u,v)$ between these vertices in $\Th$ and $\Ts$, respectively. 
We can decompose each of these paths into subpaths that visit different subtrees on the sets $V_1,\ldots,V_m$. Crucially, the structural lemma guarantees that these decompositions are consistent. For the portion of the path $\pih(u,v)$ that lies within a set $V_i$ of the partition, we can bound its error relative to $\pis(u,v)$
by $\eps$ by using the $\loctvk2$ guarantee between $\T'_i$ and $\Ts[V_i]$ on this subtree. However, a path $\pis(u,v)$ between vertices $u$ and $v$ in $\Ts$ may visit many distinct sets $V_i$ of the partition, but in this case the path must contain weak edges between the sets and this suppresses the errors and counteracts the accumulation that would otherwise occur.

\paragraph{Challenge 3: Dealing with strong edges.}

The error composition lemma for Chow-Liu reduces the learning problem to reconstructing an accurate tree model on each subset $V_i$ of vertices. Note that the edges $W$ whose removal defined the partition have correlation $\leq 1/10$, so assuming the input correlations have accuracy $\eps\leq 1/20$ implies that the true model on $V_i$ has edge strengths lower-bounded by $1/20$. Thus, the reconstruction problem on $V_i$ is potentially simpler.
However, our counterexample construction (which has lower-bounded edge strengths) implies that the Chow-Liu algorithm will not succeed even on such models, so we need a different algorithm. 

\paragraph{Main Idea 3: Tree metric reconstruction.} We show how to adapt and repurpose classical tree metric reconstruction algorithms to solve the learning problem in the case that all edge strengths are lower bounded. 
There are two central difficulties in applying tree metric reconstruction algorithms. 
The first is that the tree metric reconstruction problem formulation entails outputting a tree that has \emph{Steiner nodes}, which in graphical models terminology are latent unobserved variables, and our problem formulation requires outputting a tree model without any Steiner nodes. We overcome this difficulty via a novel desteinerization procedure that removes Steiner nodes and is guaranteed to preserve the quality of approximation.
The second difficulty is that the sort of approximation guarantee produced by tree metric reconstruction algorithms, i.e., small additive error in evolutionary distance (defined in Section~\ref{sec:evol}), is impossible to obtain in our setting. Instead, we devise a variant of the algorithm that obtains a multiplicative approximation to the distance between far away vertices, which turns out to be enough to ensure reconstruction in $\loctv$.

\paragraph{Combining ideas: {\CLpp}.} The {\CLpp} algorithm combines Chow-Liu with an approach based on tree metric reconstruction to produce a tree accurate in $\loctvk{2}$ distance. Our algorithm is simple and robust to noise.
We remark that neither the Chow-Liu algorithm nor the tree metric reconstruction algorithm is sufficient on its own. Recall that the Chow-Liu algorithm fails when there is no upper bound on the strengths of the edges in the tree model (Section~\ref{sec:CLfails}). We show in Appendix~\ref{ssec:treemetricfails} that the tree reconstruction algorithm fails when there is no lower bound on the strengths of the edges in the tree model, since then the obtained approximation of the evolutionary distance is too weak. It is therefore necessary to combine the two approaches in order to obtain the best of both worlds.

\subsection{Related work}
\label{sec:related}

\paragraph{Learning in global TV or KL.}
Several papers have studied the problem of learning graphical models to within small total variation or KL-divergence. As we show in Appendix~\ref{sec:TVlower}, learning tree Ising models to within a given total variation requires $\Omega(n \log n)$ samples.
Thus, guarantees for learning to within total variation (or KL-divergence, which by Pinsker's inequality is even more stringent) are not relevant to the high-dimensional regime (with far fewer samples than number of variables) of interest in modern applications. We summarize the results in this line of work below.

Abbeel et al. \cite{abbeel2006learning} considered the problem of learning general bounded-degree factor graphical models to within small KL-divergence and showed polynomial bounds on time and sample complexity. Devroye et al. \cite{devroye2020minimax} bounded the minimax (information theoretic) rates for learning various classes of graphical models in total variation. For the case of tree Ising models on $n$ nodes, they showed that given $m$ samples the total variation distance of the learned model from the one generating the samples decays at rate $O\Big(\sqrt{\frac{n\log(n)}{m}}\Big)$. In Appendix~\ref{sec:TVlower} of the current paper we show this to be tight. The statistical rates analyzed by \cite{devroye2020minimax} are achieved by exhaustive search algorithms and the task of finding efficient algorithms remained. 
Recently Bhattacharyya et al. \cite{bhattacharyya2020near} and Daskalakis and Pan \cite{daskalakis2020tree} studied exactly this problem, showing that the classical Chow-Liu algorithm achieves the optimal sample complexity $O(n \log n)$. They also showed robustness to model misspecification. As we show in Section~\ref{sec:CLfails}, the Chow-Liu algorithm can be arbitrary suboptimal under the local total variation loss, and moreover, the technical challenges we outlined in Section~\ref{ssec:techniques} are unique to local total variation.

\paragraph{Learning trees with noise.} A number of papers have recently considered the problem of learning Ising models (both trees and general graphs) with either stochastic or adversarial noise. 

In the stochastic noise setting, 
Nikolakakis et al. 
\cite{NikolakakisKS:19treestructures} give guarantees for structure learning using the Chow-Liu algorithm, when the variables are observed after being passed through a binary symmetric channel.
Goel et al.
\cite{goel2019learning} show that an appropriate modification of the Interaction Screening objective from \cite{vuffray2016interaction} yields an algorithm that can learn the structure of Ising models on general graphs of low width (which generalizes node degree) 
when each entry of each of the observed samples is missing with some fixed probability. 
Katiyar et al. \cite{katiyar2020robust} observe that learning the structure of tree Ising models with stochastic noise of unknown magnitude is in general non-identifiable, but they characterize a small equivalence class of models and give an algorithm for learning a member of the equivalence class.

Lindren et al. \cite{lindgren2019robust} give non-matching lower and upper bounds for learning to within small $\ell_2$ error in parameters, where the samples are subjected to adversarial corruptions. As discussed in \cite{klivans2017learning}, this also gives the graph structure under natural assumptions on the magnitudes of weights. 
Prasad et al.
\cite{prasad2020learning} consider learning Ising models under Huber's contamination model (in which the adversary is weaker than in the $\eps$-corruption model). They restrict attention to the high-temperature regime (where Dobrushin's condition holds), and prove bounds on the $\ell_2$ error. 
Diakonikolas et al. \cite{diakonikolas2021outlier} study a noise model with a stronger adversary than Huber's contamination model, but weaker than $\eps$-corruption, and again give bounds on the $\ell_2$ error of the parameters when the original model satisfies Dobrushin's condition. (Under Dobrushin's condition accuracy of parameters in $\ell_2$ translates also to accuracy in total variation.) 

\paragraph{Learning for predictions.}
Wainwright \cite{wainwright2006estimating} observed that when using an approximate inference algorithm such as loopy belief propagation, learning an incorrect graphical model may yield more accurate predictions than in the true model. Heinemann and Globerson
\cite{heinemann2014inferning} considered the task of learning a model from data such that subsequent inference computations are accurate, which is our goal in this paper. They considered the class of models on graphs with high girth, lower-bounded interaction strength, and also sufficiently weak interactions that the model satisfies correlation decay, in which case belief propagation is approximately correct. They
proposed an algorithm which always returns a model from this class and showed finite sample guarantees for its success.

Bresler and Karzand \cite{breslerkarzand19} aimed to remove the conditions in \cite{heinemann2014inferning} on the interaction strengths and formulated the problem of learning a tree Ising model to within local total variation. They proved guarantees on the performance of the Chow-Liu algorithm which were independent of the minimum edge interaction strength, but were suboptimal in their dependence on maximum interaction strength. 
Nikolakakis et al. \cite{nikolakakis2021predictive} generalized the results of \cite{breslerkarzand19} to the setting where each variable is flipped with some known probability $q$. Their algorithm preprocesses the data to estimate the correlations which are then input to the Chow-Liu algorithm.

\paragraph{Phylogenetic reconstruction.}
A related area with a different focus is \emph{phylogenetic tree reconstruction}. In this setting, we assume that the data is generated by an unknown tree graphical model and want to estimate the tree topology given access only to the leaves (corresponding to e.g. extant species). The Steel evolutionary metric (see Section~\ref{sec:evol}) and the closely related \emph{quartet test} have been used extensively in this literature: see for example \cite{steel1994recovering,daskalakis2011evolutionary,mossel2005learning,steel2016phylogeny,erdHos1999few,agarwala1998approximability}. In this setting the main goal is to recover the structure (because it encodes e.g. the history of evolution); furthermore, since most of the nodes in the tree model are latent, it is information-theoretically impossible to achieve meaningful reconstruction in this setting unless edge interactions are sufficiently strong, see e.g. \cite{daskalakis2011evolutionary}. This is not a concern in our setting because there are no latent nodes.
Applications to phylogeny motivated work on the closely related problem of estimating a tree metric (with latent/Steiner nodes) from approximate distances \cite{agarwala1998approximability,ailon2005fitting,farach1995robust}. The most relevant work in this line is \cite{agarwala1998approximability}, which showed how to reconstruct a tree metric given distances within an $\ell_{\infty}$ ball; as we explain later, this does not hold for far away nodes when we estimate the evolutionary metric from $O(\log n)$ samples, so we cannot apply their algorithm. A key insight of theirs which we do build on is a connection between tree and ultrametric reconstruction: see Section~\ref{sec:global-reconstruction} for further discussion.

\subsection{Outline}

The rest of the paper is organized as follows. In Section~\ref{sec:prelim} we give necessary background on tree-structured Ising models, formally define the learning problem we consider, and review the Chow-Liu algorithm and tree metric reconstruction. In Section~\ref{sec:contrib} we state our results. We show that the Chow-Liu algorithm is not robust to model misspecification in Section~\ref{sec:CLfails}. We describe our algorithm, and give the high-level ideas of the proof, in Section~\ref{sec:algorithm}. Most of the proofs appear in Sections~\ref{sec:preprocessing} and~\ref{sec:global-reconstruction}.

\subsection*{Acknowledgment}
GB thanks Mina Karzand for many discussions on topics related to those of this paper. We also thank the anonymous reviewers for their helpful and constructive feedback.
This work was done in part while the authors were participating in the program \textit{Probability, Geometry, and Computation in High Dimensions} at the Simons Institute for the Theory of Computing.

\section{Preliminaries}
\label{sec:prelim}

\subsection{Notation}
We write $[t] = \{1,2,\dots, t\}$. Let $\Rad(q)$ denote the distribution such that $X\sim\mathrm{Rad}(q)$ has $\PP(X=1)=q$ and $\PP(X=-1) = 1-q$. The graphs in this paper are all undirected. For an undirected graph $G = (V,E)$ and two vertices $u,v\in V$, we denote by $(u,v)$ the undirected edge between these vertices. For a subset of vertices $U\subseteq V$, $G[U]=(U,E')$ denotes the subgraph of $G$ induced by $U$, where $E' = \{(u,v)\in E: u,v\in U\}$. Also, let $V(G)$ and $E(G)$ denote the vertex set and edge set of $G$, respectively.

Throughout, $\eps$ will always be a positive real number less than or equal to $\epsuppbd$. We use hat, as in $\Th$, $\thetah{}$, or $\muh{}$, to denote quantities associated to an estimator. We use tilde, as in $\dt$ or $\mut{}$, to denote approximate quantities typically given as inputs to algorithms.

\subsection{Tree-structured Ising models}

In this paper, we consider tree-structured Ising models with no external field. It will be convenient to use a nonstandard parameterization in terms of edge correlations. The models are defined by a pair $(\T, \theta)$, where $\T$ is a tree $\T = (\Vset, E)$ on $n$ nodes and each edge $e \in E$ has corresponding parameter $\theta_e\in [-1,1]$. Each node $i\in \Vset$ is associated to a binary random variable $X_i\in \{-1,+1\}$, and each configuration of the variables $x \in \{-1,+1\}^n$ is assigned probability 
\begin{equation}
\label{eq:isingmodeldef}
P(x) = \frac{1}{Z} \exp\Big(\sum_{(i,j) \in E}\tanh^{-1}(\theta_{ij})  x_ix_j  \Big) %
\end{equation} 
where $Z$ is the normalizing constant.
For each edge $(i,j)\in E$ the correlation between the variables at its endpoints is given by $\EE_P[ X_i X_j ]= \theta_{ij}$. One obtains the standard parameterization $P(x) = Z^{-1} \exp (\sum_{(i,j)\in E} J_{ij} x_i x_j)$ upon substituting $\theta_{ij} = \tanh J_{ij}$. 

In general, we could have an external field term $\sum_{i\in \Vset} \theta_i x_i $ in the exponent of \eqref{eq:isingmodeldef}. The assumption of no external field (i.e., $\theta_i=0$) implies that the nodewise marginals satisfy $P(X_i = +1) = P(X_i = -1) = 1/2$ for all $i$. As in \cite{breslerkarzand19}, this assumption helps to make the analysis tractable and at the same time captures the central features of the problem.

We write $\mu_{uv}$ for the correlation $\EE_P[X_uX_v]$, and also write $\mu_e = \mu_{uv}$ for an edge $e = (u,v) \in E$. It turns out that in tree models each correlation $\mu_{uv}$ is given by the product of correlations along the path connecting $u$ and $v$. This fact is an easy consequence of the Markov property implied by the factorization~\eqref{eq:isingmodeldef}, and will be used often.

\begin{fact}[Correlations in tree models] 
\label{fact:corrPaths}
For any tree Ising model $(\T,\theta)$ with distribution $P$ as in \eqref{eq:isingmodeldef}, the correlation $\mu_{uv}$ between the variables at nodes $u$ and $v$ is given by
$$
\mu_{uv} = \EE_{P}[X_u X_v]= \prod_{e\in \pi(u,v)}\mu_{e} = \prod_{e\in \pi(u,v)}\theta_{e}\,.
$$
Here $\pi(u,v)$ is the path between $u$ and $v$ in $\T$. 
\end{fact}

Since the correlation between a pair of vertices has at most unit magnitude, Fact~\ref{fact:corrPaths} implies a triangle inequality for the correlations:

\begin{fact}[Triangle inequality] \label{fact:triangleineq}
For any tree Ising model $(T,\theta)$ with distribution $P$ as in \eqref{eq:isingmodeldef}, and any vertices $u,v,w$, it holds that $|\mu_{uv}| \leq |\mu_{uw}||\mu_{wv}|$.
\end{fact}

\subsection{Prediction-centric learning}

\label{sec:predictionLearning}
\subsubsection{The basic learning problem} We now formally describe the learning problem in the simpler well-specified case. We are given $m$ samples $X^{(1)},\ldots,X^{(m)}$ generated i.i.d. from a tree-structured Ising model $(\T, \theta)$ on $n$ nodes with distribution $P$. The goal is to use these samples to construct a tree-structured Ising model $(\tilde{\sT}, \tilde{\theta})$ with distribution $\tilde P$ such that the posterior $\tilde P(x_i|x_S)\approx P(x_i|x_S)$ for any $i\in \Vset$ and assignment $x_S$ to a small subset of variables $S$. As discussed in \cite{breslerkarzand19}, the accuracy of posteriors $\Pt(x_i|x_S)$ for $|S|=k-1$ is captured by the \emph{local total variation} distance of order $k$, defined as
$$
\loctvk k(P,Q):=\max_{A:|A|=k}\tv(P_A,Q_A)\,.
$$
(Here $P_A$ is the marginal of $P$ on the set of variables in $A$, and analogously for $Q$.) From the definition of conditional probability that the absolute deviation of a given posterior averaged over assignments is controlled as 
$$
\EE \big| P(x_i|X_S) - Q(x_i|X_S) \big| \leq 2\cdot  \loctvk {|S|+1}(P,Q)\qquad  \text{for any } i\text{ and } S\,.
$$ Therefore, the goal is to learn a tree-structured Ising model $\tilde{P}$ that is close to $P$ in $\loctvk{k}$.

In order to achieve this, we will use the following basic facts. 
One may readily check that $\loctvk k$ for any $k$ satisfies the triangle inequality. 
Additionally, it was shown in \cite{breslerkarzand19} that for tree models $P$ and $Q$, 
\begin{equation}\label{eqn:loctv-k-comparison}
\loctvk k (P,Q)\leq k 2^k \cdot \loctvk 2(P,Q).
\end{equation}
Finally, the following is useful in order to bound $\loctvk{2}$.
\begin{fact}\label{fact:corrLocTV}
The joint distribution of a pair of unbiased $\pm1$ binary variables $x_u,x_v$ is a function of their correlation given by $P(x_u,x_v) = (1+x_u x_v \EE_P[X_u X_v])/4$. It follows that if the singleton marginals of distributions $P$ and $Q$ over $\{-1,+1\}^n$ are unbiased, then $$\loctvk 2(P,Q) = \frac{1}{2} \max_{u,v} \big|\EE_P[X_u X_v] - \EE_Q[X_u X_v]\big|\,.$$
\end{fact}

\subsubsection{Learning from pairwise statistics: the $\epslearning$ problem} We now consider a general estimation scenario that we call the \emph{$\epslearning$ problem}.
We emphasize that this problem is a \emph{fully deterministic} algorithmic problem. Our results in both the well-specified and misspecified settings will follow from our result for this more general problem. 
 
 \begin{definition}[$\epslearning$ problem]
 	An algorithm is given $0 < \eps < 1$ as well as estimates $\mut{ij}$ for $1\leq i,j\leq n$ satisfying $|\mut{ij} - \mus{ij}|\leq \eps$,  for some unknown tree Ising model $(\Ts,\theta)$ with distribution $P$ having correlations $\mus{ij} = \EE_P[X_i X_j]$.
 	The algorithm is said to \emph{solve the $\epslearning$ problem} (with some approximation constant $C$) if it is guaranteed to output a tree Ising model $(\Th,\thetah{})$ with distribution $\Ph$ satisfying $\loctvk2(\Ph, P) \leq C\eps$.
 \end{definition}
 
It is worth emphasizing that these estimates can be chosen adversarially subject to the accuracy constraint, potentially even in a way that is not realizable by any probability distribution. Note that the problem is well-defined: while in general there can be many tree models $P$ with correlations within $\epsilon$ of the given estimates, all are within $\loctvk2$ of at most $2\eps$ of one another by the triangle inequality.
Hence it suffices to have the output be close in $\loctvk2$ to \emph{any} tree Ising model. 

What this means is that exhaustively searching over tree models with an appropriate grid over parameters and outputting one with pairwise correlations close to the given $\mut{}$ estimates constitutes a valid solution of the $\epslearning$ problem. Obviously such a procedure is not computationally feasible, and the primary technical contribution of this paper is to give an efficient algorithm for $\epslearning$. 

An algorithm for $\epslearning$ immediately implies that one can efficiently learn from $O((\log n)/\eps^2)$ samples in the PAC model, even under possible model misspecification.

\begin{claim}[$\epslearning$ implies learning under model misspecification]\label{claim:epslearningrelationtomodelmisspecification} Let $P$ be an arbitrary distribution over $\{+1,-1\}^{\Vset}$, and let $\Delta = \inf_{Q \in \cT} \loctvk{k}(Q,P)$ be the amount of misspecification in the problem, where here the infinimum is over tree Ising models. Given an approximation parameter $\eps > 0$, the misspecification $\Delta > 0$, and $m = O(\log (|\Vset|/\delta)/\eps^2)$ samples from $P$, 
there is a procedure that takes $O(|\Vset|^2m)$ time and one call to an algorithm solving $\epslearning$ with approximation constant $C$, and with probability at least $1-\delta$ yields a tree Ising model $(\Th, \theta{})$ with distribution $\Ph$ satisfying
$$\loctvk{k}(\Ph,P) \leq Ck2^k (\epsilon + \Delta).$$
\end{claim}
\begin{proof}
First, from the samples $x_1,\ldots,x_m$ compute the empirical correlation $\mut{u,v} = \frac{1}{m} \sum_{i=1}^m x_{i,u} x_{i,v}$ for each $u,v \in \Vset$. Condition on the event that $|\mut{u,v} - \EE_P[X_uX_v]| \leq \eps$ for all $u,v \in \Vset$. This occurs with probability $1-\delta$ by a Chernoff bound and a union bound. Let $Q$ be a tree-structured distribution achieving $\loctvk{k}(Q,P) = \Delta$ which exists by compactness of $\cT$. Let $\mus{u,v}$ be the pairwise correlations of $Q$ and note that for all $u,v \in \Vset$ $$|\mut{u,v} - \mus{u,v}| \leq \eps + \Delta/2$$ by Fact~\ref{fact:corrLocTV} and the triangle inequality.
Run the $\epslearning$ algorithm with approximation parameter $\eps + \Delta/2$ and inputs $\mut{}$, to obtain a tree model $(\Th, \thetah{})$ with distribution $\Ph$ that satisfies the guarantee $\loctvk{2}(\Ph,Q) \leq C(\epsilon + \Delta/2)$. By \eqref{eqn:loctv-k-comparison}, this means $\loctvk{k}(\Ph,Q) \leq Ck2^k (\epsilon + \Delta/2)$, so by triangle inequality $\loctvk{k}(\Ph,P) \leq Ck2^k (\epsilon + \Delta/2) + \Delta \leq Ck2^k (\epsilon + \Delta)$, since $Ck2^k \geq 2$.
\end{proof}
We remark that the amount of misspecification $\Delta$ does not need to be known, since we can essentially run an exponential search with an extra runtime factor overhead of $O(\log(1/\Delta))$, and lose only a constant factor in the quality of the reconstruction.
\begin{remark}[Distributional Robustness Interpretation]
The proof of Claim~\ref{claim:epslearningrelationtomodelmisspecification} has a natural interpretation in terms of \emph{distributionally robust optimization} (DRO), see e.g. \cite{wiesemann2014distributionally,lee2017minimax,esfahani2018data}. We started with the empirical measure $P_m$ given by $m$ samples from $P$ and observed that $P$ lies within the \emph{ambiguity set} 
\[ \mathcal{P} = \{R : \loctvk2(P_m, R) \le \epsilon,\inf_{Q \in \mathcal{T}}\loctvk{k}(Q,R) \le \Delta \}. \]
Since the true measure $P$ is unknown, the DRO approach would consider the minimax optimization
\begin{equation}\label{eqn:dro} \inf_{\hat P \in \mathcal{T}} \sup_{R \in \mathcal{P}} \loctvk{k}(R,\hat P).
\end{equation}
The assumptions guarantee, by triangle inequality, that the minimizer $\hat P$ of \eqref{eqn:dro} has objective value at most $O(k2^k \epsilon + \Delta)$. However, because the set of tree measures $\mathcal T$ is nonconvex (e.g. as a subset of the space of all probability measures), there is no obvious computationally efficient way to solve this minimization. Instead of solving it directly, we observed that the combination of an $\epsilon$-learning algorithm and the triangle inequality enable us to compute an approximate minimizer $\hat P$ satisfying
\[ \sup_{R \in \mathcal{P}} \loctvk{k}(R,\hat P) \le Ck2^k(\epsilon + \Delta) \]
which in particular implies the result when we plug in $P$ for $R$. Important to our application, the uncertainty parameter $\epsilon$ can be taken to depend only logarithmically in the dimension $|\Vset|$ while containing the true measure $P$; this contrasts with the usual ``curse of dimensionality'' which requires $\epsilon$ to have an exponential dependence on the dimension (e.g. for Wasserstein uncertainty, see the discussion in \cite{blanchet2019robust}).
\end{remark}

\subsection{The Chow-Liu algorithm} \label{sec:CL}

The Chow-Liu algorithm is a classical algorithm that finds the maximum-likelihood tree model given observations. In the case of Ising models with no external field that we consider, it can be motivated based on the following observation.

\begin{claim}\label{claim:spanningtreetheta}
Given a tree model $(\T,\theta)$ with correlations $\mu$, $\T$ is a maximum weight spanning tree for weights $|\mu|$.
\end{claim}
\begin{proof}
$\T$ is a maximum spanning tree if for any vertices $u,v$ in $\T$, all of the edges $e$ on the path $\pi(u,v)$ from $u$ to $v$ have correlation $|\mu_{e}| \geq |\mu_{u,v}|$. This is the case, since $|\mu_{u,v}| = \prod_{e\in \pi(u,v)}|\mu_{e}|\leq\min_{e\in \pi(u,v)}|\mu_{e}|$. 
\end{proof}

Based on this observation, given estimates $\mut{}$ for the correlations, it makes sense to try to learn a tree based on the max weight spanning tree with weights $|\mut{}|$, and indeed as shown in \cite{breslerkarzand19} the Chow-Liu algorithm (which computes the maximum likelihood tree Ising model with no external field) does exactly this. We take this as the definition.
\begin{definition}
Given approximate correlations $\mut{}$, the Chow-Liu tree $\TCL$ is a maximum spanning tree for weights $|\mut{}|$, breaking ties arbitrarily.
\end{definition}

As we will show, the Chow-Liu algorithm does not solve the $\epslearning$ problem, as it can err when there are edges with strong correlation in the true model. However, it will be an important ingredient of our final algorithm, {\CLpp}. 

\subsection{Evolutionary distance and tree metric reconstruction}
\label{sec:evol}
In the {\CLpp} algorithm, we will also use a connection between tree models and tree metrics via the insight of Steel \cite{steel1994recovering}, that on all tree models there is a natural measure of distance between the vertices called the \emph{evolutionary distance} which, furthermore, can be estimated from samples without knowing the tree topology. 
\begin{definition}[Evolutionary distance] \label{def:evolutionarydistance}
	In the special case of interest (pairwise interactions with unbiased binary spins), the evolutionary distance of a tree model $(\T,\Theta)$ with distribution $P$ is simply
\[\d(u,v) := -\log \EE_P[X_u X_v] = -\log \mu_{uv}\,. \]
\end{definition}

It is well-known that the evolutionary distance is a tree metric, defined formally as follows. 
\begin{definition}[Tree Metric]\label{def:treeMetric}
Let $\T$ be a tree with vertex set $\Vset$. We say that $\d:\Vset\times \Vset\to \mathbb{R}_{\geq 0}$ is a tree metric over $\T$ if for any $u,v \in \Vset$ and $u = u_0,\ldots,u_k = v$ the unique path from $u$ to $v$ in $\T$ that
$\d(u,v) = \sum_{i = 0}^{k - 1} \d(u_i,u_{i + 1})$. Equivalently, $\d$ is the shortest path metric when edges $(u,v)\in E(\T)$ are weighted by $\d(u,v)$.
\end{definition}

We will also consider tree metrics that are allowed to have additional Steiner nodes (i.e. latent nodes in the graphical models terminology). The metric associated to a tree with Steiner nodes is called an additive metric.

\begin{definition}[Additive metric]\label{def:additiveMetric}
Let $\Vset$ be a finite set. We call $\d:\Vset \times \Vset\to \mathbb{R}_{\geq 0}$ an \emph{additive metric} if it is the restriction to $\Vset$ of some tree metric $\d':\cX\times \cX\to \mathbb{R}_{\geq 0}$ over a tree with vertex set $\cX\supseteq \Vset$. The nodes in $\cX\sm \Vset$ are called \emph{Steiner} nodes. 
\end{definition}

\section{Results}
\label{sec:contrib}

Our main result is the following theorem, the main ideas of which are presented in Section~\ref{sec:algorithm}. 
\begin{theorem}[Guarantee for {\CLpp}]\label{t:CL++}
Let $P$ be the distribution of a tree Ising model $(\Ts, \thetas{})$ on $n$ binary $\pm 1$ variables $X_1,\dots, X_n$ with correlations $\mu_{ij} = \EE_P[X_i X_j]$. Given $0 < \eps < \epsuppbd$ and estimates $\mut{ij}$ for $1\leq i,j\leq n$ satisfying $|\mut{ij} - \mu_{ij}|\leq \eps$, {\CLpp} runs in $O(n^2)$ time and solves the $\epslearning$ problem, i.e., there is a universal constant $C$ such that it returns a tree Ising model $(\Th,\thetah{})$ with distribution $\Ph$ such that
$$
\loctvk2(\Ph,P)\leq C\eps\,.
$$
\end{theorem}

By Claim~\ref{claim:epslearningrelationtomodelmisspecification}, this theorem implies the following oracle inequality result (in terms of $\loctvk{k}$) for estimation from i.i.d. samples. 
We emphasize, however, that our main theorem applies even in the case that the samples are not generated i.i.d. according to \emph{any} distribution. 

\begin{corollary}[Oracle inequality for $\loctvk2$]\label{c:oracle2}
Let $P$ be an arbitrary distribution on $n$ binary $\pm1$ variables. Let $\Delta = \inf_{Q\in \cT} \loctvk2 (Q,P),$ where the infimum is over tree Ising models.
There is a numerical constant $C>0$ such that given $m \geq C\log (n/\delta) /\eps^2$ i.i.d. samples from $P$ and the misspecification error $\Delta$, with probability at least $1-\delta$, there is an algorithm that runs in $O(n^2m)$ time and returns a tree Ising model $(\Th,\thetah{})$ with distribution $\Ph$  satisfying
$$
\loctvk{k}(\Ph, P) \leq Ck2^k (\Delta + \eps)\,.
$$
\end{corollary}
The sample-complexity dependence on $n$ and $\eps$ is optimal up to constant factors, by Corollary 3.6 of \cite{breslerkarzand19}, and the runtime complexity matches that of the Chow-Liu algorithm.
Furthermore, we remark that the result can be adapted so that the amount of model misspecification $\Delta$ does not need to be known, as we can run exponential search over $\Delta$, and only invoke an extra runtime factor of $O(\log(1/\Delta))$.

Theorem~\ref{thm:prediction-to-structure} in Appendix~\ref{sec:predImpliesStruct} shows that sufficient accuracy in $\loctvk3$ implies correct structure recovery; thus the learning for predictions problem \emph{generalizes} the structure learning problem. This is used to show Corollary~\ref{cor:minimaxStructure} which states that {\CLpp} is optimal also for structure learning. Thus, given data from a tree Ising model, {\CLpp} learns the correct structure if that is possible, and regardless it learns a model that yields good predictions.

We also prove in the next section the surprising result that the Chow-Liu algorithm fails under model misspecification, even when given population values for the correlations.
\begin{theorem}\label{thm:chowliufailsmodelmisspecification}
For any $\delta > 0$, there is a distribution $P$ such that
\begin{enumerate}
    \item[(i)] there is a tree-structured distribution $P'$ with $\loctvk2(P', P)\leq \delta$, and 
    \item[(ii)] the Chow-Liu algorithm, given population values for the correlations, outputs a tree $\TCL$ such that all $\TCL$-structured distributions $Q$ have $\loctvk2(Q,P)\geq 0.05$.
\end{enumerate}
\end{theorem}

\section{The Chow-Liu Algorithm Fails with Model Misspecification} \label{sec:CLfails}

In this section, we establish Theorem~\ref{thm:chowliufailsmodelmisspecification}, which shows a surprising brittleness of the Chow-Liu algorithm to model misspecification. In particular, suppose that the distribution $P$ is not a tree-structured model, but is $\delta$-close in $\loctvk2$ distance to a tree-structured model $P'$ for some small $\delta > 0$. We show that the Chow-Liu algorithm does not achieve nontrivial reconstruction guarantees even when run with infinitely many samples from $P$ (i.e., with the population values for the correlations). 

Informally, the reason for this failure is that the greedy Chow-Liu algorithm can be tricked into taking too many shortcuts when constructing its tree, resulting in significant distortion of the correlation between distant nodes. Note that we prove a rather strong failure of the Chow-Liu algorithm: not only is the specific distribution prescribed by the Chow-Liu algorithm (given by a tree and edge weights) a poor approximation to the one generating the data, but so too is \emph{any} tree-structured distribution that has the Chow-Liu tree structure.

The proof of Theorem~\ref{thm:chowliufailsmodelmisspecification} is as follows.
The construction of $P$ is given in Section~\ref{ssec:nonrobustnessconstruction}. Its closeness to a tree-structured distribution is shown in Section~\ref{ssec:closetreestructuredconstruction}. The failure of the Chow-Liu algorithm is proved in Section~\ref{ssec:chowliufailureproof}.

\subsection{Construction of distribution $P$}\label{ssec:nonrobustnessconstruction}
Fix a parameter $0 < \delta < 0.1$ and let $n = \lceil 1 / \delta \rceil$. Let $P = P_{\delta}$ denote the distribution over the variables $(X_1,\ldots,X_n,Y_1,\ldots,Y_n)$ that is sampled as follows.
Let $A_1,\ldots,A_n,B,C$ be independent $\{+1,-1\}$-valued random variables such that for all $i \in [n]$ we have $\EE[A_i] = e^{-\delta}$, $\EE[B] = e^{-2\delta}$ and $\EE[C] = 0$.
Now let $$X_i = C  \prod_{t=1}^i A_t \quad \text{and}\quad Y_i = BC \prod_{t=1}^i A_t \quad \text{for all }i \in [n]\,.$$

By a straightforward calculation for each $i,j \in [n]$ the first two moments are:
$$\EE_P[X_i] = \EE_P[Y_i] = 0.$$
$$\EE_P[X_iX_j] = \EE_P[Y_iY_j] = \exp(-\delta|i-j|).$$
$$\EE_P[X_iY_j] = \exp(-\delta|i-j|-2\delta).$$

\subsection{Closeness to tree-structured distribution $P'$}\label{ssec:closetreestructuredconstruction}

We construct a tree-structured distribution $P'$ over the variables $(X_1,\ldots,X_n,Y_1,\ldots,Y_n)$ as follows. Let $W,Z_1,\ldots,Z_{n}$ be independent $\{+1,-1\}$-valued random variables such that for all $i \in [n]$ we have $\EE[Z_i] = e^{-\delta}$ and $\EE[W] = 0$. Let $$X_i = Y_i = W \prod_{t=1}^i Z_t.$$ This can be represented by a tree-structured distribution, where the tree structure is a path in order $X_1,Y_1,X_2,Y_2,\ldots,X_n,Y_n$ (see panel (a) in Figure~\ref{fig:CLfails}). 
For all $i,j \in [n]$, the first two moments are:
$$\EE_{P'}[X_i] = \EE_{P'}[Y_i] = 0,$$
$$ \EE_{P'}[X_i X_j] = \EE_{P'}[X_i Y_j] = \EE_{P'}[Y_i Y_j] = \exp(-\delta|i-j|).$$

\begin{figure}
    \centering
    \begin{tabular}{ccc}
    \includegraphics[scale=0.5]{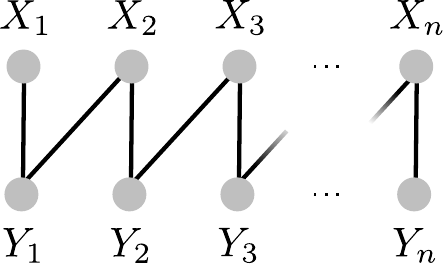} & \qquad\qquad\qquad & \includegraphics[scale=0.5]{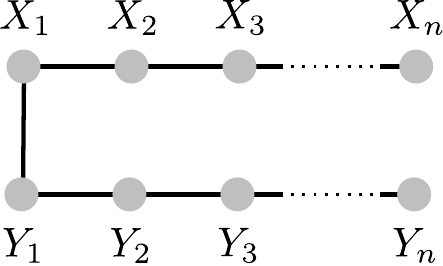} \\
    (a) Tree that is close to $P$ & & (b) Tree returned by Chow-Liu
    \end{tabular}
    \caption{Example capturing the failure of Chow-Liu under model misspecification. The true distribution $P$ is within $\loctvk2$ distance $\delta$ of a tree model $P'$ structured as in (a), however Chow-Liu returns a tree model (b) for which all parameter choices result in $\Omega(1)$ distance in $\loctvk2$ to $P$. We remark that the picture (b) is only one of the possible trees returned by Chow-Liu, since there are many ties in our construction, but we show that all of them fail. }
    \label{fig:CLfails}
\end{figure}

Therefore, the correlations of $P$ and $P'$ match up to error $2\delta$. For any $i,j \in [n]$:
\begin{align*} \EE_P[X_iX_j] = \EE_{P'}[X_iX_j], \quad \EE_P[Y_iY_j] = \EE_{P'}[Y_iY_j]\end{align*}
\begin{align*} &\big|\EE_P[X_iY_{j}] - \EE_{P'}[X_iY_{j}]\big| = \big|e^{-\delta|i-j|-2\delta} - e^{-\delta|i-j|}| \leq 1 - e^{-2\delta} \leq 2\delta\end{align*}

Since the correlations under $P$ and $P'$ are $2\delta$-close and the variables are unbiased, Fact~\ref{fact:corrLocTV} implies \begin{equation}\loctvk2(P,P')\leq 2\delta.\end{equation}

\subsection{Failure of the Chow-Liu algorithm}\label{ssec:chowliufailureproof}
We now show that the Chow-Liu algorithm, when given as input the population correlations under $P$, outputs a tree $\TCL$ such that all $\TCL$-structured distributions are $0.1$-far from $P$ in $\loctvk2$ distance. This shows the failure of the Chow-Liu algorithm. First, define the set $S = \{(X_i,X_{i+1})\}_{i \in [n-1]} \cup \{(Y_i,Y_{i+1})\}_{i \in [n-1]}$. Note that all pairs of variables in $S$ have correlation $e^{-\delta}$ under $P$, and all other pairs have correlation at most $e^{-2\delta}$. Since the edges in $S$ yield no cycles, the Chow-Liu (max-weight spanning) tree $\TCL$ contains the edges $S$. Thus, the Chow-Liu tree $\TCL$ consists of the paths $X_1,\ldots,X_n$ and $Y_1,\ldots,Y_n$ as well as one edge $(X_i,Y_j)$ between these paths.

Fix any $\TCL$-structured distribution $Q$. We wish to show that $\loctvk{2}(Q,P) \geq 0.05$. Suppose that $i \leq n/2$ (a symmetric argument using $\EE_Q[X_1Y_1]$ works if $i > n/2$):
\begin{align*}
\big|\EE_Q[X_nY_n]\big| &= \big|\EE_Q[X_nX_i]\EE_Q[X_iY_n]\big| & \mbox{using Fact~\ref{fact:corrPaths}} \\
&\leq \big|\EE_Q[X_nX_i]\big| 
\\ &\leq \big|\EE_P[X_nX_i]\big| + 2 \loctvk2(Q,P) & \mbox{ using Fact~\ref{fact:corrLocTV}}
\\ &\leq e^{-\delta n /2} + 2 \loctvk2(Q,P), & \mbox{using $i \leq n/2$ and Fact~\ref{fact:corrLocTV}}
\end{align*}
In the other direction, again by Fact~\ref{fact:corrLocTV} $$\big|\EE_Q[X_nY_n]\big| \geq \big|\EE_P[X_nY_n]\big| - 2\loctvk2(Q,P) = e^{-2\delta} - 2\loctvk2(Q,P)\,,$$
so putting these inequalities together $$\loctvk{2}(Q,P) \geq \frac{e^{-2\delta} - e^{-\delta n / 2}}{4} > 0.05,$$ by our assumption that $n \geq 1/\delta$ and $\delta < 0.1$.

\qed

\section{Algorithm}
\label{sec:algorithm}

In this section, we present our robust algorithm, {\CLpp}, for prediction-centric learning. The formal proof of correctness is given in this section, but most of the details are deferred to Sections~\ref{sec:preprocessing}, \ref{sec:global-reconstruction}, and  \ref{sec:desteinerize}.

\subsection{Overview}

The two most popular classical approaches for reconstructing tree models under noisy data observations are (i) the Chow-Liu algorithm, and (ii) tree metric reconstruction based on the evolutionary distance. Surprisingly, neither of these algorithms solves the $\epslearning$ problem by itself. As we have shown in Section~\ref{sec:CLfails}, the Chow-Liu algorithm fails when the edges are allowed to be arbitrarily strong. And, on the other hand, as we show in Section~\ref{ssec:treemetricfails}, the tree metric reconstruction approach fails when the edge strengths are allowed to be arbitrarily weak.

However, by a stroke of luck, the algorithms happen to be effective in contrasting edge strength regimes. The Chow-Liu algorithm is known to work when all edge strengths are \textit{upper-bounded} by a constant \cite{breslerkarzand19}\footnote{Our proof of this fact is different from the proof given by \cite{breslerkarzand19} and we can handle model misspecification.}. And we prove that a tree metric reconstruction approach works when all edge strengths are \textit{lower-bounded} by a constant. By stitching these algorithms together, we obtain the best of both worlds: a single algorithm, {\CLpp}, that succeeds for models with arbitrary edge weights.

For simplicity, we present our algorithm in the restricted case that the model is ferromagnetic (the edge correlations are nonnegative). The following proposition, proved in Appendix~\ref{app:toferromagnetic} shows that restricting to ferromagnetic models comes at no loss of generality: 
\begin{proposition}\label{prop:reducetoferromagnetic} The $\epslearning$ problem for general models on $n$ vertices can be solved in $O(n^2)$ time and one call to an oracle $\LearnFerromagneticModel$ for the $\epslearning$ problem restricted to ferromagnetic models.
\end{proposition}

In Section~\ref{ssec:chowliuoverview},  we present our algorithm $\LearnFerromagneticModel$, which solves the reconstruction problem for ferromagnetic models, and is based off of the Chow-Liu algorithm. This algorithm has a key subroutine, $\LearnLowerboundedModel$, which we present in Section~\ref{ssec:tmroverview}. The subroutine $\LearnLowerboundedModel$ is inspired by the tree metric reconstruction literature, and learns a model assuming that the model's edge strengths are lower-bounded by a constant. The results in Sections~\ref{ssec:chowliuoverview} and \ref{ssec:tmroverview} are summarized by:
\begin{proposition}\label{prop:reducetolowerbounded}\label{prop:preprocessing}
$\LearnFerromagneticModel(\mut,\eps)$ solves the ferromagnetic $\epslearning$ problem in $O(n^2)$ time and $m = O(n)$ calls to $\LearnLowerboundedModel$.
Furthermore, letting $t_1,\ldots,t_m$ denote the number of input vertices in the calls to $\LearnLowerboundedModel$, we have $\sum_{i=1}^m t_i \leq n$.
\end{proposition}
\begin{proposition}
\label{prop:tmr}
The algorithm $\LearnLowerboundedModel$ solves the $\epslearning$ problem on tree models with $n$ variables with edge correlations $\geq \edgelowbd$ in time $O(n^2)$.
\end{proposition}
Combining these propositions gives the {\CLpp} algorithm, and allows us to formally prove Theorem~\ref{t:CL++}.
\begin{proof}[Proof of Theorem~\ref{t:CL++}]
By Proposition~\ref{prop:tmr}, the $\LearnLowerboundedModel$ algorithm is guaranteed to solve the $\epslearning$ problem in $O(n^2)$ time when the edge weights are lower-bounded by $\edgelowbd$. By Proposition~\ref{prop:preprocessing}, we can use it as a subroutine of the $\LearnFerromagneticModel$ algorithm to solve the $\epslearning$ problem in $O(n^2)$ time for ferromagnetic models. And by Proposition~\ref{prop:reducetoferromagnetic} we can use $\LearnFerromagneticModel$ as a subroutine of $\ReduceToFerromagnetic$ (i.e., the {\CLpp} algorithm) given in Appendix~\ref{app:toferromagnetic} to solve the $\epslearning$ problem in $O(n^2)$ time for general models.
\end{proof}

\subsection{Chow-Liu strikes back: handling weak edges}\label{ssec:chowliuoverview}

We present our algorithm $\LearnFerromagneticModel$, which learns any ferromagnetic model (i.e., any model with arbitrary nonnegative edge strengths). The overall scheme of this algorithm is to (i) use the Chow-Liu algorithm to break the learning problem into small subproblems, each corresponding to a subtree with lower-bounded edge strengths, and (ii) solve each subproblem with the subroutine $\LearnLowerboundedModel$. In this scheme, the Chow-Liu algorithm handles the weak edges of the model, and the $\LearnLowerboundedModel$ subroutine handles the strong edges. A surprising takeaway is that Chow-Liu algorithm  redeems itself -- the failure of Chow-Liu proved in Section~\ref{sec:CLfails} is due only to strong edges in the tree model, as Chow-Liu can handle the weak edges well. Pseudocode for $\LearnFerromagneticModel$ is given in Algorithm~\ref{alg:chowliustep}.

\begin{algorithm}\SetAlgoLined\DontPrintSemicolon
    \KwIn{
    \begin{tabular}{l}$0 < \eps < \epsuppbd$ \\  $\eps$-approximate correlations $\mut{}$ for ferromagnetic model $(\Ts,\thetas{})$\end{tabular}
    }
    \KwOut{Model $(\Th,\thetah{})$ that $O(\eps)$-approximates $(\Ts,\thetas{})$ in $\loctvk2$ distance.}
    
    \BlankLine

    \tcp{Run Chow-Liu}
    
    ${\TCL} \gets \mathrm{MaximumSpanningTree}({\mut{}})$
    
    \tcp{Use the result to partition the vertices $\Vset$ into $\sqcup_i V_i$}
    
    $W \gets \{e \in E(\TCL) : \mut{e} \leq 1/10\}$
    
    $V_1,\ldots,V_m \gets \mathrm{ConnectedComponents}({\TCL} \sm W)$. \label{step:componentcreation}

    \tcp{Learn a tree model on each set $V_i$}
    $(\Th^{(i)},\theta^{(i)}) \gets \LearnLowerboundedModel(\{\mut{uv}:u,v\in V_i \}, \eps)$ for each $i \in [m]$

    \tcp{Stitch together the learned models with the weak edges $W$ from Chow-Liu}
    
    $(\Th,\thetah{}) \gets ((\cup_i \Th^{(i)} \cup W), (\cup_i \theta^{(i)} \cup \mut{W}))$
    
    Return $(\Th, \thetah{})$

    \caption{$\LearnFerromagneticModel(\mut,\eps)$}
\label{alg:chowliustep}
\end{algorithm}

In words, $\LearnFerromagneticModel$ proceeds as follows. We first run Chow-Liu, obtaining the tree $\TCL$. Denoting by $W $ the weak edges in $\TCL$ with empirical correlations $\leq 1/10$, we partition the vertices $\Vset$ into sets $V_1,\ldots,V_m$ according to the connected components in the forest with edges $E(\TCL)\setminus W$. Now, perhaps counterintuitively, we keep only the weak edges $W$ in the Chow-Liu tree, and discard the stronger edges whose only purpose was to define the groups of relatively strongly dependent variables. Crucially, we can prove that each of these vertex sets $V_i$ induces a subtree $\Ts[V_i]$ of the original model (or of any model close in $\loctvk2$ in the misspecified case), with edge weights lower-bounded by $1/10 - \eps \geq \edgelowbd$. We then reconstruct the model separately on each subtree $\Ts[V_i]$ using the $\LearnLowerboundedModel$ procedure, taking advantage of the lower bound on the edge strength within the subtree. Finally, we glue together the reconstructed trees on vertex sets $V_1,\ldots,V_m$ using the weak edges $W$ from the Chow-Liu tree.  Figure~\ref{fig:preprocessingdiagram} illustrates the algorithm.

\begin{figure}
    \centering
    \includegraphics[scale=0.75]{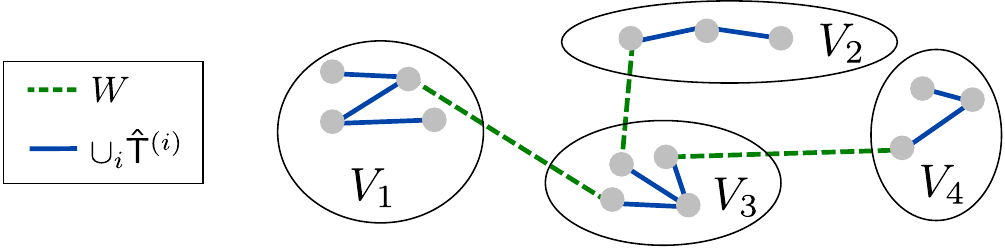}
    \caption{Diagram illustrating the tree $\Th$ returned by $\LearnFerromagneticModel$ (Algorithm~\ref{alg:chowliustep}). The green dashed edges are the weak edges $W$ of the Chow-Liu tree. They connect between sets of vertices $V_1,\ldots,V_m$. The solid blue edges within each set $V_i$ are the tree $\Th^{(i)}$ constructed by the $\LearnLowerboundedModel$ subroutine run on $V_i$.}
    \label{fig:preprocessingdiagram}
\end{figure}

The algorithm's formal correctness guarantee is given in Proposition~\ref{prop:preprocessing}. We sketch the proof, deferring details to Section~\ref{sec:preprocessing}. We wish to prove that for any vertices $u$ and $v$, that the correlation in the reconstructed tree model approximates the true correlation: $$|\muh{u,v} - \mus{u,v}| \leq O(\eps).$$
Our proof separates the problem into two different cases, depending on the edge strengths of the path $\pis(u,v)$ in $\Ts$ between $u$ and $v$. In the first case, all of the edges $e$ of $\pis(u,v)$ have strength at least $\thetas{e} \geq C\epsilon$ for some large enough constant $C$; in the second case there is some edge $e$ with correlation strength $\thetas{e} < C\epsilon$. 

\paragraph{Case 1} Suppose that the edges $e$ in $\pis(u,v)$ all have true correlation $\thetas{e} \geq C\epsilon$. Then we can prove the following critical structural lemma: the reconstructed path $\pih(u,v)$ in $\Th$ between $u$ and $v$ visits sets $V_{l_1},\ldots,V_{l_t}$ in the same order as the path $\pis(u,v)$ in $\Ts$. Specifically, we can decompose the paths $\pis(u,v)$ and $\pih(u,v)$ as a concatenation of paths and edges,
$$\pih(u,v) = \pih(\hat{a}_1,\hat{b}_1) \circ (\hat{b}_1,\hat{a}_2) \circ \pih(\hat{a}_2,\hat{b}_2) \circ \dots \circ (\hat{b}_{t-1},\hat{a}_t) \circ \pih(\hat{a}_t,\hat{b}_t), $$
$$ \pis(u,v) = \pis(a_1,b_1) \circ (b_1,a_2) \circ \pis(a_2,b_2) \circ \dots \circ (b_{t-1},a_t) \circ \pis(a_t,b_t), $$ where for any $i \in [t]$, the paths $\pis(a_i,b_i)$ and $\pih(\hat{a}_i,\hat{b}_i)$ stay within the same set $V_{l_i}$. Here $a_i$ and $b_i$ denote the first and last vertices visited by $\pis(u,v)$ in $V_{l_i}$. And $\hat{a}_i$ and $\hat{b}_i$ are similarly defined for $\pih(u,v)$. (In particular $a_1 = \ah_1 = u$ and $b_t = \bh_t = v$.) See Figure~\ref{fig:errordecaysufficientlybounded} for an illustration.

\begin{figure}[h]
    \centering
    \includegraphics[scale=0.75]{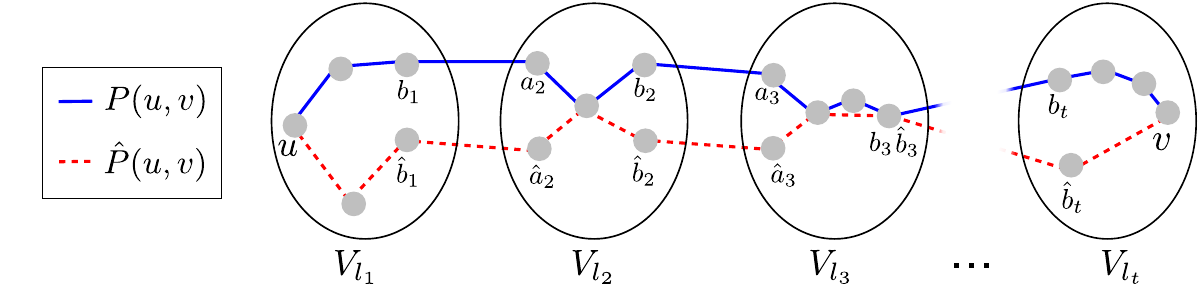} 
  \caption{If all edges $e$ in the true path $\pis(u,v)$ have strength at least $C\epsilon$, then the paths $\pis(u,v)$ and $\pih(u,v)$ visit the sets of the partition $\Vset = \sqcup V_i$ in the same order.}
    \label{fig:errordecaysufficientlybounded}
\end{figure}

Using this decomposition of $\pih(u,v)$ and $\pis(u,v)$ into a concatenation of paths and edges, we can write $\muh{u,v}$ and $\mus{u,v}$ as a product of correlations of the edges and paths in the decomposition:
\begin{align}\muh{u,v} = \muh{\ah_1,\bh_1}\muh{\bh_1, \ah_2} \muh{\ah_2, \bh_2} \dots \muh{\bh_{t-1},\ah_t} \muh{\ah_t,\bh_t}.\label{eq:muhdecomposition}\end{align}
\begin{align}\mus{u,v} = \mus{a_1,b_1} \mus{b_1,a_2} \mus{a_2, b_2} \dots \mus{b_{t-1},a_t} \mus{a_t,b_t},\label{eq:musdecomposition}\end{align}
For our analysis, we also define the quantity $\Upsilon_{u,v}$ by replacing each $\muh{}$ with $\mus{}$ in equation \eqref{eq:muhdecomposition}:
\begin{align}\Upsilon_{u,v} = \mus{\ah_1,\bh_1}\mus{\bh_1, \ah_2} \mus{\ah_2, \bh_2} \dots \mus{\bh_{t-1},\ah_t} \mus{\ah_t,\bh_t}.\label{eq:upsilondecomposition}\end{align}

We are now ready to prove our goal, which is to show the error bound $|\muh{u,v} - \mus{u,v}| \leq O(\eps)$. We bound the error by the triangle inequality, writing $$|\muh{u,v} - \mus{u,v}| \leq |\muh{u,v} - \Upsilon_{u,v}| + |\Upsilon_{u,v} - \mus{u,v}|.$$

This bounds the error by two terms where each term has a nice intuitive interpretation: (i) the first term $|\muh{u,v} - \Upsilon_{u,v}|$ represents the error caused because the edge weights $\thetah{}$ for $\Th$ are based on approximate correlations rather than on the true population correlations, and (ii) the second term $|\Upsilon_{u,v} - \mus{u,v}|$ represents the error caused because the paths $\pih(u,v)$ and $\pis(u,v)$ are different, since the tree $\Th$ might not equal the true tree $\Ts$. We now bound both terms by $O(\eps)$, which will prove the desired result.

\textit{First term}: We show $|\muh{u,v} - \Upsilon_{u,v}| \leq O(\epsilon)$.
For each $i \in [t-1]$, the edge $(\bh_i,\ah_{i+1})$ lies between distinct sets $V_{l_i},V_{l_{i+1}}$, so it belongs to the set $W$ of weak edges of the Chow-Liu tree. So by construction it has weight equal to the empirical correlation: $\muh{\bh_i,\ah_{i+1}} = \mut{\bh_i,\ah_{i+1}}$. Since $\mut{}$ and $\mus{}$ are $\epsilon$-close by assumption, this means $$|\muh{\bh_i,\ah_{i+1}} - \mus{\bh_i,\ah_{i+1}}| \leq O(\eps).$$ Furthermore, for each $i \in [t]$, recall that the path $\pis(a_i,b_i)$ lies fully in $V_{l_i}$, so it lies in a subtree of $\Th$ outputted by the reconstruction subroutine $\LearnLowerboundedModel$. By the correctness guarantee of the subroutine (Proposition~\ref{prop:tmr}), we also have
$$|\muh{\hat{a}_i \hat{b}_i} - \mus{\hat{a}_i \hat{b}_i}| \leq O(\eps).$$
Therefore, comparing formula \eqref{eq:muhdecomposition} for $\muh{u,v}$ and formula \eqref{eq:upsilondecomposition} for $\Upsilon_{u,v}$, we notice that both quantities are a product of $2t-1$ factors, such that each factor in the formula \eqref{eq:muhdecomposition} is $O(\epsilon)$-close to the corresponding factor in the formula \eqref{eq:upsilondecomposition}. Since all factors have magnitude bounded by $1$, this immediately implies that $|\muh{u,v} - \Upsilon_{u,v}| \leq O(\epsilon t)$. Unfortunately, this crude bound is insufficient, since we do not have any control on $t$, the number of vertex sets $V_{l_1},\ldots,V_{l_t}$ visited by the path $\pih(u,v)$. The bound fails because it simply adds up the $O(\epsilon)$ approximation error incurred each time we transition from one set $V_{l_i}$ to the next set $V_{l_{i+1}}$.

Fortunately, a more refined argument can show that the errors do not accumulate linearly in $t$. Notice that for any $i \in [t-1]$ the edge $(\bh_i,\ah_{i+1})$ is one of the weak edges $W$ of the Chow-Liu tree, since $\bh_i$ and $\ah_{i+1}$ are in distinct sets $V_{l_i}$ and $V_{l_{i+1}}$. This means that $\muh{\bh_i,\ah_{i+1}} \leq 1/10$. Hence, $\muh{u,v}$ and $\Upsilon_{u,v}$ decay geometrically in the number of sets $t$ visited by the path: $$\muh{u,v} \leq (1/10)^t, \quad \Upsilon_{u,v} \leq (1/10 + O(\epsilon))^t.$$ Thus, if $t \geq \log(1/\epsilon)$, then $|\muh{u,v} - \Upsilon_{u,v}| \leq |\muh{u,v}| + |\Upsilon_{u,v}| \leq O(\epsilon)$. Combined with the crude error bound of $O(\epsilon t)$ for $t \leq \log(1/\epsilon)$, this shows that $|\muh{u,v} - \Upsilon_{u,v}| \leq O(\epsilon \log(1/\epsilon))$ independently of $t$. In the formal proof, we shave off the factor of $\log(1/\epsilon)$ with a more careful analysis and show $|\muh{u,v} - \Upsilon_{u,v}| \leq O(\epsilon)$.

\textit{Second term}: To show $|\Upsilon_{u,v} - \mus{u,v}| \leq O(\epsilon)$, we perform a similar analysis. We consider the formula \eqref{eq:musdecomposition} for $\mus{u,v}$ and the formula \eqref{eq:upsilondecomposition} for $\Upsilon_{u,v}$ and again compare these formulas factor by factor.
However, this analysis is substantially more delicate than the analysis for the first term. A main obstacle is that some factors in \eqref{eq:musdecomposition} may not $O(\epsilon)$-approximate the corresponding factors of \eqref{eq:upsilondecomposition}. Concretely, we may have the bad case that $|\mus{a_i,b_i} - \mus{\ah_i,\bh_i}| \gg \epsilon$ for some $i \in [t]$. To overcome this obstacle, we prove that if this bad case happens, then $\mus{u,v}$ and $\Upsilon_{u,v}$ decay sufficiently quickly to zero that they are $O(\epsilon)$-close anyway. We defer further details of the bound of this second term to the formal proof in Section~\ref{sec:preprocessing}.

\paragraph{Case 2} Now suppose that the path $\pis(u,v)$ contains a very weak edge $e$ with strength $\thetas{e} < C\epsilon$. In this case, we no longer have the key structural lemma that $\pih(u,v)$ and $\pis(u,v)$ visit vertex sets $V_{l_1},\ldots,V_{l_t}$ in the same order.  For example, if all vertices are perfectly uncorrelated (all edge strengths are $0$), then any true tree $\Ts$ is possible, so such a structural result on the reconstruction $\Th$ is impossible to guarantee.

Although we cannot prove the same structural lemma, we can prove that the algorithm still works. We prove that $\pih(u,v)$ must also have an edge $\hat{e}$ of strength $\thetah{\hat{e}} \leq O(\epsilon)$. Since $\mus{u,v} \leq \thetas{e} < C\eps$ and $\muh{u,v} \leq \thetah{\hat{e}} \leq O(\eps)$,  then $|\mus{u,v} - \muh{u,v}| \leq |\mus{u,v}| + |\muh{u,v}| \leq O(\eps)$, as desired.

\subsection{Tree metric reconstruction: learning lower-bounded models}\label{ssec:tmroverview}
It remains to present the subroutine $\LearnLowerboundedModel$, which solves the learning problem in the case that the model's edge weights are lower bounded by a constant. Pseudocode is given in Algorithm~\ref{alg:learnlowerbounded}.

\begin{algorithm}\SetAlgoLined\DontPrintSemicolon
    \KwIn{\begin{tabular}{l}$0 < \eps < \epsuppbd$ \\  $\eps$-approximate correlations $\mut{}$ for model $(\Ts,\thetas{})$ with edge strengths $\thetas{} \geq \edgelowbd$\end{tabular}}
    \KwOut{Model $(\Th,\thetah{})$ that $O(\eps)$-approximates $(\Ts,\thetas{})$ in $\loctvk2$ distance.}
    
    \BlankLine
    
    \tcp{Construct approximation to evolutionary distance, 
    $\dtpre : \Vset \times \Vset \to [0,\infty]$}
    $\dtpre(u,v) \gets -\log(\max(0,\mut{u,v}-\eps))$ for all $u,v \in \Vset$\label{step:dtinit}

    \tcp{Run tree metric reconstruction on evolutionary distance}
    $(\Th, \dh) \gets \TreeMetricReconstruction(\dtpre, \log(1/\edgelowbd),200000\eps)$ \label{step:treemetricreconstruction}
    
    \tcp{Convert back to Ising model}
    $\thetah{e} \gets \exp(-\dh(e))$ for all edges $e \in E(\Th)$
   
    Return $(\Th, \thetah{})$
   
    \caption{$\LearnLowerboundedModel(\mut,\eps)$}
\label{alg:learnlowerbounded}
\end{algorithm}

The algorithm is based on our solution to a particular tree metric reconstruction problem. The connection to tree metrics is via the beautiful insight of Steel \cite{steel1994recovering}, that on all tree models there is a natural measure of distance between the vertices called the \emph{evolutionary distance} which can be estimated from samples without knowing the tree topology. Recall from Definition~\ref{def:evolutionarydistance} that the evolutionary distance between $u,v \in \Vset$ is given by:
$$\d(u,v) = -\log \mus{u,v}.$$
It is known that $\d$ is a tree metric on the tree $\Ts$ (see Definition~\ref{def:treeMetric}), so this suggests applying a tree metric reconstruction algorithm to learn $\Ts$ given an approximation $\dtpre$ of $\d$. Unfortunately, as we discuss below in Section~\ref{sssec:tmrtheoremoverview}, classical tree metric reconstruction algorithms are insufficient for this task. Instead we provide our own algorithm, which we call $\TreeMetricReconstruction$, to reconstruct the evolutionary distance with the following guarantee.
\begin{theorem}[Tree Metric Reconstruction]\label{thm:treemetricreconstruction}
There exists an absolute constant $C > 0$ such that the following result is true for arbitrary $\epsilon, L > 0$.
Let $(\Ts,\d)$ be an unknown tree metric on vertex set $\Vset$. Suppose that neighboring vertices $u,v$ satisfy $\d(u,v) \le L$, and suppose we have access to a local distance estimate $\dtpre$ such that
\begin{align*}
\d(u,v) &\leq \dtpre(u,v), \qquad & \mbox{ for all } u,v \in \Vset \\
\dtpre(u,v) &\leq \d(u,v) + \eps &\mbox{ for all } u,v \in \Vset \mbox{ such that }\d(u,v) \leq 3\Lupp \end{align*}
Given this information, there exists a $O(n^2)$-time algorithm $\TreeMetricReconstruction$ which constructs a tree ${\Th}$ equipped with tree metric ${\dh}$ such that
\begin{equation}\label{eqn:tmrappxgoal}
    |{\dh}(u,v) -\d(u,v)| \le C (\d(u,v)/L + 1) \epsilon
\end{equation}
for all vertices $u,v$.
\end{theorem}
In words, for vertices $u,v$ that are at a low distance $\d(u,v) = O(L)$, the reconstructed tree metric $\dh$ offers a $O(\eps)$-additive approximation to $\d(u,v)$. For vertices $u,v$ that are far apart, at distance $\d(u,v) = \Omega(L)$, the metric $\dh$ offers a $(1+O(\eps/L))$-multiplicative approximation to $\d(u,v)$. In Section~\ref{sssec:formalproptmr} we prove that Theorem~\ref{thm:treemetricreconstruction} is sufficient to guarantee correctness of the $\LearnLowerboundedModel$ algorithm.
In Section~\ref{sssec:tmrtheoremoverview}, we prove Theorem~\ref{thm:treemetricreconstruction}, relying on results from Sections~\ref{sec:global-reconstruction} and \ref{sec:desteinerize}, and explain how it differs from and extends known tree metric reconstruction algorithms.

\subsubsection{Proof of Proposition~\ref{prop:tmr}, correctness of $\LearnLowerboundedModel$}\label{sssec:formalproptmr}
We prove the correctness of $\LearnLowerboundedModel$ assuming Theorem~\ref{thm:treemetricreconstruction}. The runtime is $O(n^2)$ since the $\TreeMetricReconstruction$ subroutine runs in $O(n^2)$ time by Theorem~\ref{thm:treemetricreconstruction}. It remains to prove that the reconstruction $(\Th, \thetah{})$ returned by the algorithm is accurate in local total variation. In what follows, let $$\Lupp = -\log(\edgelowbd).$$

We prove that the function $\dtpre : \Vset \times \Vset \to [0,\infty]$ constructed by $\LearnLowerboundedModel$ satisfies the following approximation guarantee with respect to $\d$.
\begin{lemma}[Approximation guarantee of $\dtpre$]\label{lem:dtpreappxguarantee} Define $\dtpre : \Vset \times \Vset \to [0,\infty]$ as in Step~\ref{step:dtinit} of $\LearnLowerboundedModel$. Then
\begin{align*}
\d(u,v)& \leq \dtpre(u,v)  & \mbox{ for all } u,v \\
\dtpre(u,v)& \leq \d(u,v) + 200000\eps  &\mbox{ for all } u,v \mbox{ such that }\d(u,v) \leq 3\Lupp \end{align*}
\end{lemma}
\begin{proof}
Recall that $|\mut{u,v} - \mus{u,v}| \leq \eps$ for all $u,v \in \Vset$. Thus, the first inequality follows from $$\d(u,v) = -\log(\mus{u,v}) \leq -\log(\max(0,\mut{u,v} - \eps)) \leq  \dtpre(u,v).$$
For the second inequality, consider $u,v$ such that $\d(u,v) \leq 3\Lupp$. Then $\mus{u,v} = \exp(-\d(u,v)) \geq \exp(-3\Lupp) = (\edgelowbd)^3 > 2 \cdot \epsuppbd > 2\eps$, so $$\dtpre(u,v) = -\log(\mut{u,v} - \eps) \leq -\log(\mus{u,v} - 2\eps) \leq -\log(\mus{u,v}) + \frac{2\eps}{\mus{u,v} - 2\eps} \leq \d(u,v) + 200000\epsilon.$$ The second-to-last inequality comes from $1/(\mus{u,v} - 2\eps)$-Lipschitzness of $\log(t)$ for $t \geq 1/(\mus{u,v} - 2\eps)$. The last inequality comes from $\mus{u,v} > (1/20)^3 > 2 \times \epsuppbd > 2\eps$. 

\end{proof}
We observe that $\Lupp$ is an upper bound on the edge length in the evolutionary tree metric $(\Ts,\d)$, since each edge has correlation at least $\edgelowbd$. Thus, we can plug the approximation guarantee of Lemma~\ref{lem:dtpreappxguarantee} into
Theorem~\ref{thm:treemetricreconstruction}, and ensure\textit{} that the tree metric approximation $(\Th,\dh)$ computed in $\LearnLowerboundedModel$ satisfies, for some constant $C$ independent of $\eps$:
\begin{align}
|\dh(u,v) - \d(u,v)| \leq C(\d(u,v)/\Lupp + 1)\eps \leq C(\d(u,v) + 1)\eps,\label{ineq:thappxguarantee}
\end{align} where we have used that $\Lupp \geq 1$ is a uniformly lower-bounded constant.
This guarantee turns out to be sufficient to solve the problem of learning in total variation, since the model $(\Th, \thetah{})$ outputted by $\LearnLowerboundedModel$, which has pairwise correlations $\muh{u,v} = \exp(-\dh(u,v))$, is guaranteed to be close to the true model by the following lemma:
\begin{lemma}\label{lem:approx-suffices}
Let $\mu_{uv} = e^{- \d(u,v)}$ and $ \muh{uv} = e^{-\dh(u,v)}$. If $\dh(u,v)$ satisfies the approximation guarantee \eqref{ineq:thappxguarantee}
with respect to $\d(u,v)$, then $|\mu_{uv} - \hat{\mu}_{uv}|\leq 2C\epsilon$ .
\end{lemma}

\begin{proof}
Suppose that $2C \epsilon < 1$ without loss of generality, since $\mu_{uv},\muh{uv} \in [0,1]$. From
$|\mu_{uv} - \hat{\mu}_{uv}|= e^{-\d(u,v)} \big|e^{\d(u,v) - \dh(u,v)} - 1\big| $ and the guarantee in \eqref{eqn:tmrappxgoal} we see that
$$
e^{-\d(u,v)}\Big(e^{-(\d(u,v) + 1) \cdot C\epsilon} - 1 \Big)
\leq \mu_{uv} - \muh{uv}\leq e^{-\d(u,v)}\Big(e^{(\d(u,v) + 1) \cdot C\epsilon} - 1 \Big)\,.
$$
Now we differentiate the RHS to get
$$
\frac d{dt}\Big(e^{-t}\big( e^{(t+1)C\epsilon} - 1 \big)\Big) = e^{-t}\Big( 1- (1-C\epsilon) e^{(t+1)C\epsilon}\Big)\,.
$$
This is positive for $t\in [0,\frac1{C\epsilon} \log \frac1{1-C\epsilon} -1)$ and is negative for larger $t$, so it suffices to bound the RHS for $\d(u,v) = \frac1{C\epsilon} \log \frac1{1-C\epsilon} -1$, and it is easily verified that the RHS is at most $2C\epsilon$.

For the LHS, we differentiate to get
$$
\frac d{dt}\Big(e^{-t}\big( e^{-(t+1)C\epsilon} - 1 \big)\Big) = e^{-t}\Big( 1- (1+C\epsilon) e^{-(t+1)C\epsilon}\Big)\,,
$$
and we see that the LHS is decreasing precisely on $[0,\frac{1}{C\epsilon}\log (1+C\epsilon)-1)$. Plugging into the LHS we see that the minimum value it can take is at least $(1+C\epsilon)^{-1}-1 \geq - C\epsilon$. 
\end{proof}
Since all pairwise correlations are $2C\eps$-close, the model $(\Th,\thetah{})$ is $2C\eps$-close to $(\Ts,\thetas{})$ in $\loctvk{2}$, and Proposition~\ref{prop:tmr} follows.
\qed

\subsubsection{Proof of Theorem~\ref{thm:treemetricreconstruction}, tree metric reconstruction algorithm}\label{sssec:tmrtheoremoverview}

Here we present the $\TreeMetricReconstruction$ algorithm and prove the correctness guarantee, Theorem~\ref{thm:treemetricreconstruction}. The pseudocode for $\TreeMetricReconstruction$ is given as Algorithm~\ref{alg:treemetricreconstruction}. In the first step, $\AdditiveMetricReconstruction$ (given in Section~\ref{sec:global-reconstruction}) reconstructs an additive metric approximation. In the second step, $\Desteinerize$ (given in Section~\ref{sec:desteinerize}) converts the reconstructed additive metric to a tree metric without Steiner nodes.

\begin{algorithm}\SetAlgoLined\DontPrintSemicolon
    \KwIn{$\dtpre,L,\eps$ satisfying the conditions of Theorem~\ref{thm:treemetricreconstruction}}
    \KwOut{Tree metric $(\Th, \dh)$ satisfying approximation guarantee \eqref{eqn:tmrappxgoal}}
    
    $(\Th', \dh') \gets \AdditiveMetricReconstruction(\dtpre,L,\eps)$
    
    $(\Th, \dh) \gets \Desteinerize(\Th',\dh')$
    
    Return $(\Th, \dh)$
   
    \caption{$\TreeMetricReconstruction(\dtpre,L,\epsilon)$}
\label{alg:treemetricreconstruction}
\end{algorithm}

Although there is a rich literature that studies tree metric reconstruction problems, known algorithms do not suffice to obtain the guarantee of Theorem~\ref{thm:treemetricreconstruction}. The most relevant result from the literature is the landmark algorithm of Agarwala et al. \cite{agarwala1998approximability} for additive metric approximation (recall from Definition~\ref{def:additiveMetric} that additive metrics are tree metrics with Steiner nodes). Given a $O(\eps)$-approximation $\dt$ to an additive metric $\d$, the algorithm outputs an additive metric $\dh$ that $O(\eps)$-approximates $\d$.
Two major obstacles arise when applying this result. The first obstacle is that the algorithm of Agarwala et al. outputs an additive metric as opposed to a tree metric. If we convert this additive metric into a tree model, then the Steiner nodes in the additive metric correspond to latent variables in the model. This is undesirable, as we wish to learn a model without latent variables. The second obstacle is that it is impossible to compute a $O(\eps)$-additive approximation to the evolutionary distance between all pairs of nodes. The culprits are far-away nodes: as the correlations decrease to 0, the evolutionary distance grows to infinity, and so a $O(\eps)$-additive approximation in the correlations does not yield a $O(\eps)$-additive approximation in the evolutionary distance. 

Nevertheless, although the additive metric reconstruction theorem of Agarwala et al. \cite{agarwala1998approximability} does not directly apply, our proof of Theorem~\ref{thm:treemetricreconstruction} leverages the important connection that Agarwala et al. found between the tree metric reconstruction problem and reconstruction of ultrametrics. Through this connection to ultrametric reconstruction, in Section~\ref{sec:global-reconstruction} we derive the algorithm $\AdditiveMetricReconstruction$, which obtains an additive metric reconstruction of the evolutionary distance satisfying the approximation guarantee \eqref{eqn:tmrappxgoal}, where distances between far-away nodes are multiplicatively-approximated. Then, in Section~\ref{sec:desteinerize} we show how to convert this additive metric to a tree metric with an algorithm $\Desteinerize$ that clusters the Steiner nodes to their closest non-Steiner node. The guarantees that we prove for $\AdditiveMetricReconstruction$ and $\Desteinerize$ are stated below:

\begin{theorem}[Additive Metric Reconstruction]\label{thm:additivemetricreconstruction}
There exists a constant $C > 0$ such that the following result is true for arbitrary $\epsilon > 0$ and any $L \geq 100 \epsilon$.
Let $(\Ts,\d)$ be an unknown tree metric on vertex set $\Vset$. Suppose that neighboring vertices $u,v \in \Vset$ satisfy $\d(u,v) \le L$, and suppose we have access to a distance estimate $\dtpre : \cV \times \cV \to \RR \cup \{\infty\}$ such that
\begin{align*}
\d(u,v) &\leq \dtpre(u,v), \qquad & \mbox{ for all } u,v \in \Vset \\
\dtpre(u,v) &\leq \d(u,v) + \eps &\mbox{ for all } u,v \in \Vset \mbox{ such that }\d(u,v) \leq 3\Lupp \end{align*}
Given this information, there exists a $O(|\Vset|^2)$-time algorithm $\AdditiveMetricReconstruction$ which constructs an additive metric $(\Th, \dh)$ with $|V(\Th)| \leq O(|\Vset|)$ Steiner nodes such that
\begin{equation*}
    |{\dh}(u,v) -\d(u,v)| \le C (\d(u,v)/L + 1) \epsilon
\end{equation*}
for all $u,v \in \Vset$.
\end{theorem}

\begin{theorem}[Removing Steiner nodes]\label{thm:desteinerize}
There exists a constant $C > 0$ such that the following is true. 
Let $L \ge 2\epsilon \ge 0$ and suppose $(\Ts,\d)$ is an unknown tree metric on vertex set $\Vset$ with maximum edge length $L$.
Suppose $\dh$ is an additive metric on $\Vset$ with Steiner tree representation $\Th$ such that
\begin{equation*} |\d(x,y) - \dh(x,y)| \le (\d(x,y)/L + 1)\epsilon \end{equation*}
for all $x,y \in \Vset$. Then \textsc{Desteinerize} with $(\Th,\dh)$ as input runs in $O(|\Vset|^2 + |V(\Th)|)$ time and outputs a tree metric $(\T',\d')$ on $\Vset$ such that
\[ |\d(x,y) - \d'(x,y)| \le (\d(x,y)/L + 1)\epsilon + C\epsilon \]
for all $x,y \in \Vset$.
\end{theorem}

Theorems~\ref{thm:additivemetricreconstruction} and \ref{thm:desteinerize} are proved in Sections~\ref{sec:global-reconstruction} and \ref{sec:desteinerize}, respectively. Combined, these two theorems immediately imply the correctness and runtime of $\TreeMetricReconstruction$, Theorem~\ref{thm:treemetricreconstruction}:

\begin{proof}[Proof of Theorem~\ref{thm:treemetricreconstruction}]
By Theorems~\ref{thm:additivemetricreconstruction} and \ref{thm:desteinerize}, $\TreeMetricReconstruction$ runs in $O(n^2)$ time. Furthermore, there are constants $C',C''$ such that if $\eps > 0$, $L \geq 100\epsilon$, $L \geq 2C'\epsilon$, then the outputted tree metric satisfies the approximation guarantee $$|\d(x,y) - \dh(x,y)| \leq C' (\d(x,y)/L + 1)\epsilon + C'C''\epsilon.$$ Taking $C = \max(100, 2C', C' (C'' + 1))$, we can therefore ensure \eqref{eqn:tmrappxgoal}
for any $\epsilon,L > 0$. (Since if $L < 100\epsilon$ or $L < 2C'\epsilon$, outputting $\dh(x,y) = 0$ for all $x,y$ works.)
\end{proof}

\begin{remark}
Theorem~\ref{thm:treemetricreconstruction} is much easier to prove if we do not require the output $\dh$ to be a tree metric, since essentially looking at the shortest path metric of $\dt$ gives a global metric satisfying a guarantee like \eqref{eqn:tmrappxgoal}; see Lemma~\ref{lem:sp-good} for a proof. However, straightforward approaches to extract a tree from this metric --- e.g. taking a shortest path tree containing all shortest paths from some arbitrary root node $\rho$, or a minimum cost spanning tree --- are easily seen to fail on a path graph with edge distances $\Theta(\epsilon)$. If we use a minimum cost spanning tree then, in the case where $\d$ is the evolutionary metric, the resulting tree is the Chow-Liu Tree and its failure is explained in Section~\ref{sec:CLfails}. If we use a shortest path tree started from the first vertex of the path, then the tree can fail in the same way as Chow-Liu: by modifying all distances by $\Theta(\epsilon)$ we can make the shortest path to the final vertex in the path go through every other vertex in the graph, except possibly at the end, making the tree match the Chow-Liu tree in Figure~\ref{fig:CLfails} and making distances between some neighbors off by the entire diameter of the graph.
\end{remark}

\section{Chow-Liu: handling weak edges}\label{sec:preprocessing}

This section is devoted to the proof of Proposition~\ref{prop:reducetolowerbounded}, which shows correctness of the algorithm $\LearnFerromagneticModel$, based on the Chow-Liu algorithm. Section~\ref{ssec:ProofForLearnFerro} formally proves the proposition, using claims whose proofs are deferred to Section~\ref{ssec:preprocessingauxiliary}.

\subsection{Proof of Proposition~\ref{prop:reducetolowerbounded}}
\label{ssec:ProofForLearnFerro}
First, the output $\Th$ of $\LearnFerromagneticModel$ is a tree, because each $T_i'$ is a tree on the set $V_i$, and by construction the edge set $W$ forms a tree when the sets of the partition $V_1,\ldots,V_m$ are contracted. The runtime guarantee of the algorithm is also straightforward, since the maximum spanning tree and connected components can be found in $O(|\cV|^2)$ time. It remains to show the $\loctvk2$ approximation guarantee. By Fact~\ref{fact:corrLocTV}, this reduces to proving that there is an absolute constant $\Clearn > 0$ such that $$|\muh{uv} - \mus{uv}| \leq \Clearn \eps \mbox{ for all } u,v \in \Vset.$$

To show this, let $\pis(a,b) \subset \Ts$ and $\pih(a,b) \subset \Th$ denote the paths from vertex $a$ to vertex $b$ in $\Ts$ and $\Th$, respectively. Then consider any pair of vertices $u,v \in \Vset$. If there is a very weak edge $e \in \pis(u,v)$ such that $\thetas{e} \leq 4\eps$, then Claim~\ref{claim:pathswithweakedges} below shows that there is an edge $e' \in \pih(u,v)$ such that $\thetah{e'} \leq 5\eps$, and hence 
$$|\muh{uv} - \mus{uv}| \leq |\muh{uv}| + |\mus{uv}| \leq |\thetah{e'}|+ |\thetas{e}|  \leq 9\eps \leq \Clearn \eps\,.$$
So without loss of generality $\thetas{e} > 4\eps$ for all edges $e \in \pis(u,v)$. Under this edge strength lower bound condition, Claim~\ref{claim:pathsagreeglobally} below shows that $\pis(u,v)$ and $\pih(u,v)$ are equal if we contract the vertex sets $V_1,\ldots,V_m$. In other words, we can decompose the paths as $$ \pis(u,v) = \pis(a_1,b_1) \circ (b_1,a_2) \circ \pis(a_2,b_2) \circ \dots \circ (b_{t-1},a_t) \circ \pis(a_t,b_t), $$ $$\pih(u,v) = \pih(a'_1,b'_1) \circ (b'_1,a'_2) \circ \pih(a'_2,b'_2) \circ \dots \circ (b'_{t-1},a'_t) \circ \pih(a'_t,b'_t), $$  where $a_i,b_i,a'_i,b'_i \in V_{\sigma(i)} \mbox{ for all } i \in [t]$ for some permutation $\sigma: [m]\to [m]$. 

Hence, 
$$\mus{uv} = \prod_{i=1}^t \mus{a_i b_i} \prod_{i=1}^{t-1} \mus{b_i a_{i+1}} ,$$ $$\muh{uv} = \prod_{i=1}^t \muh{a'_i b'_i} \prod_{i=1}^{t-1} \muh{b'_i a'_{i+1}}.$$ 
In order to bound the difference between these two quantities, we also define 
$$
\Upsilon_{uv} = \prod_{i=1}^t  \mus{a'_i b'_i} \cdot \prod_{i=1}^{t-1} \mus{b'_i a'_{i+1}},
$$ 
which is the product along the estimated path in $\Th$ of the correlations as computed in $\Ts$.
Now applying the triangle inequality yields 
$$|\muh{uv} - \mus{uv}| \leq |\muh{uv} - \Upsilon_{uv}| + |\Upsilon_{uv} - \mus{uv}|.$$

Claim~\ref{claim:upsilonthetaprimediff} below shows that $|\muh{uv} - \Upsilon_{uv}| \leq C''\eps$ for some absolute constant $C''$. The rough idea is that, for each $i \in [t]$, $\muh{a_i' b_i'}$ is an estimated correlation within the partition $V_{\sigma(i)}$ and $\Ts[V_{\sigma(i)}]$ has all edge weights of magnitude at least $\edgelowbd$ by Claim~\ref{claim:preprocessingconnectedsubtrees} below. Hence $\muh{a_i' b_i'}$ is a
 $C'\eps$-approximation of $\mu_{a_i' b_i'}$ for a universal constant $C'$ by the correctness of \LearnLowerboundedModel. By the fact that $\mut{}$ is an $\eps$-approximation of $\mus{}$, we also know that for each $i \in [t-1]$, $\muh{b_i' a_{i+1}'}$ is a $2\eps$-approximation of $\mu_{b_i' a_{i+1}'}$. Furthermore, as discussed in the overview~\ref{ssec:chowliuoverview}, the errors from each term in the product do not accumulate as $t$ increases, but rather add with geometric attenuation, because  $\mu{b_i' a_{i+1}'} \leq 1/10 + \eps \leq 1/5$ for all $i \in [t-1]$.

Claim~\ref{claim:upsilonthetadiff} shows that $|\Upsilon_{uv} - \mus{uv}| \leq 4\eps$. This claim is proved in two steps. First, it is easy to show that $\Upsilon_{uv} \leq \mus{uv}$ by noting that $\mus{ab} \mus{bc} \leq \mus{ac}$ for any vertices $a,b,c$ in a tree-structured model. Second, to show that $\mus{uv} \leq \Upsilon_{uv} + 4\eps$, we prove that $a_i$ is close to $a_i'$ and $b_i$ is close to $b_i'$ for all $i \in [t]$. This latter step requires more care -- in particular, we use that we may assume that $\thetas{e} > 4\eps$ for all edges $e \in \pi(u,v)$ without loss of generality.

So overall $|\muh{uv} - \mus{uv}| \leq (C'' + 4)\eps \leq \Clearn \eps$ for an absolute constant $\Clearn$. \qed

\subsection{Deferred details}\label{ssec:preprocessingauxiliary}

The claims in this section fill out the details of the proof of Proposition~\ref{prop:reducetolowerbounded}.

\subsubsection{Correctness within each set $V_i$}
First, we show that within each vertex set $V_i$ the correctness of $\LearnLowerboundedModel$ ensures that $\muh{}$ $O(\eps)$-approximates $\mu$.

\begin{claim}[Correlations are weak between sets]\label{claim:distinctfar}
If $u \in V_i, v \in V_j$ for some $i \neq j$, then $\mut{uv} \leq 1/10$.
\end{claim}
\begin{proof}
Since $u$ and $v$ are not in the same set of the partition, there is an edge $e$ (in the set of weak edges $W$) on the path from $u$ to $v$ in ${\TCL}$ with $\mut{e} \leq 1/10$. If $\mut{uv} > 1/10 \geq \mut{e}$, then removing $e$ and adding $(u,v)$ to $\TCL$ gives a spanning tree of strictly higher weight, contradicting maximality of $\TCL$.
\end{proof}

Recall $(\Ts, \thetas{})$ is a tree model with correlations $\mus{}$ within $\eps$ of the given approximate values $\mut{}$. 

\begin{claim}[Each set induces a subtree with lower-bounded edge strengths]\label{claim:preprocessingconnectedsubtrees}
$\Ts[V_i]$ is a subtree of $\Ts$ for each $i \in [m]$. Furthermore, $\thetas{e} \geq 1/10 - \eps$ for all $e\in E(\Ts[V_i])$.
\end{claim}

\begin{proof}
Assume for the sake of contradiction that $\Ts[V_i]$ is not connected. Then there is a vertex $w \not\in V_i$ and partition $V_i = V_{i,1} \sqcup V_{i,2}$ into nonempty sets such that all paths in $\Ts$ from $V_{i,1}$ to $V_{i,2}$ go through $w$. 
So for all $u \in V_{i,1}$ and $v \in V_{i,2}$, 
$$\mus{uv} = \mus{uw} \mus{wv} \leq (\mut{uw} + \eps)(\mut{wv} + \eps) \leq (1/10 + \eps)(1/10 + \eps) \leq 1/20.
$$ %
We have used $\mut{uw}, \mut{wv} \leq 1/10$ by Claim~\ref{claim:distinctfar}, since $w \not\in V_i$, $u,v \in V_i$. 
Thus $\mut{uv} \leq \mus{uv} + \eps \leq 1/10$, for all $u \in V_{i,1}$ and $v \in V_{i,2}$. 
On the other hand, there must be $u \in V_{i,1}, v \in V_{i,2}$ such that $\mut{uv} > 1/10$ (otherwise the only possible edges between $V_{i,1}$ and $V_{i,2}$ in $\TCL$ would be in $W$ and
hence $V_i$ would not have ended up as one part), which is a contradiction.

Furthermore, if there is $e \in E(\Ts[V_i])$ such that $\thetas e < 1/10 - \eps$, then let $V_i = V'_{i,1} \sqcup V'_{i,2}$ be the partition of $V_i$ formed by removing $e$ from $\Ts[V_i]$. For any $u \in V'_{i,1}, v \in V'_{i,2}$, we have $\mut{uv} \leq \mus{uv} + \eps \leq \thetas e + \eps < 1/10,$ which by the argument in the previous paragraph contradicts the construction of $V_i$.
\end{proof}

\begin{claim}[Correlations are correct within each set]\label{claim:correctwithinVi}
There is an absolute constant $C'$ such that if $u,v \in V_i$ for some $i$, then $|\muh{uv} - \mus{uv}| \leq C' \eps$.
\end{claim}
\begin{proof}
Claim~\ref{claim:preprocessingconnectedsubtrees} shows that $\Ts[V_i]$ is a tree with $\thetas e\geq (1/10 - \eps) \geq \edgelowbd$ for each edge $e$, so Proposition~\ref{prop:tmr} applies, showing that $\LearnLowerboundedModel$ solves the $\eps$-reconstruction problem. In other words, there is an absolute constant $C'$ such that $\LearnLowerboundedModel$ returns a tree-structured model $(T'_i, \theta'_i)$ such that all correlations satisfy $|\mu'_{i,uv} - \mus{uv}| \leq C' \eps$ for $u,v\in V_i = V(T'_i)$, where $\mu'_{i,uv}$ is the correlation computed in the model $(T'_i, \theta'_i)$. The proof follows by noting that correlations $\mu'_{i,uv}$ computed within each model $(T'_i, \theta'_i)$ are exactly the same as the correlations $\hat{\mu}_{uv}$ in the final estimated model $(\Th, \thetah)$, since $\Th[V_i]$ is connected and thus correlations within $V_i$ entail only edges within $\Th[V_i] = T'_i$. 
\end{proof}

This latter claim implies the desired local reconstruction error guarantee when $u,v \in V_i$ for some $i \in [m]$.

\subsubsection{Topology roughly correct}
In order to translate the error guarantee within each set $V_i$ (proved by Claim~\ref{claim:correctwithinVi}) into an error guarantee between any pair of distinct sets $V_i$ and $V_j$, we first prove that the topologies of $\Ts$ and $\Th$ are equal (except possibly very weak edges) when the sets $V_i$ are contracted. In other words, except for very weak edges, the edges in $\Th$ between sets $V_1,\dots, V_m$ go between the correct sets. 
\begin{claim}\label{claim:strongTbetweensetsarealsoinTprime}
Let $(u,v) \in E(\Ts)$ such that $u \in V_i,v \in V_j$ for distinct $i,j$. If $\thetas{uv}=\mus{uv} \geq 4\eps$ then there is $(u',v') \in E(\Th)$ such that $u' \in V_i, v' \in V_j$.
\end{claim}
\begin{proof}
Consider the path $\pih(u,v)$ in $\Th$. Let $e = (a,b)$ be the first edge on this path such that $a$ and $b$ are in separate components of $\Ts \sm (u,v)$. We show that $a \in V_i$ and $b \in V_j$, proving the theorem. Otherwise, without loss of generality $a \in V_k$ for some $k \neq i$.
We show that $\tilde{\mu}_{ab} < \tilde{\mu}_{uv}$:
\begin{align*}\mut{ab} &\leq \mus{ab} + \eps \\
&\leq \mus{au}\mus{uv}\mus{vb} + \eps & \mbox{by Fact~\ref{fact:corrPaths}, since $(u,v) \in \pis(a,b)$} \\
&\leq \mus{au}\mus{uv} + \eps & \mbox{using $\mus{vb} \leq 1$}\\
&\leq (\mut{au} + \eps)(\mut{uv} + \eps) + \eps \\
&\leq (1/10 + \eps)(\mut{uv} + \eps) + \eps & \mbox{by Claim~\ref{claim:distinctfar}, since $u \in V_i, a \in V_k$ for $i \neq k$} \\
&\leq (1/10 + \eps) \mut{uv}  + (11/10 + \eps)\eps \\
&= \mut{uv} - (9/10 - \eps)\mut{uv} + (11/10 + \eps)\eps \\
&< \mut{uv}, & \mbox{since $\mut{uv} \geq \mus{uv} - \eps \geq 3\eps$}.
\end{align*}
We now show that $\TCLprime = \TCL \cup (u,v) \sm (a,b)$ is a tree. Since $\TCLprime$ has strictly greater $\mut{}$ weight than $\TCL$, this would contradict maximality of the Chow-Liu tree and prove the claim.  (i) First, $(a,b) \in W \subset E(\TCL)$, since $(a,b) \in E(\Th)$ and $\mut{ab} < \mut{uv} \leq 1/10$ by Claim~\ref{claim:distinctfar}. (ii) Second, $u$ and $v$ are in distinct components of $\TCL \sm (a,b)$, since $u$ is connected to $a$ and $v$ is connected to $b$ in $\TCL \sm (a,b)$. Indeed, for any edge $f = (c,d) \neq (a,b)$ in $\pih(u,v)$ we prove that $c$ and $d$ are connected in $\TCL \sm (a,b)$: either $c \in V_l, d \in V_{l'}$ for distinct $l,l'$, in which case $(c,d) \in W \subset E(\TCL)$; or, if $c,d \in V_l$ for some $l$, then $c$ and $d$ are connected in $\TCL \sm (a,b)$ because $V_l$ is a connected component of $\TCL \sm W$ and $(a,b) \in W$. Combined, (i) and (ii) prove that $\TCLprime$ is a tree.
\end{proof}

This immediately implies the following:
\begin{claim}[Reconstructed topology correct -- not-very-weak edges case]\label{claim:pathsagreeglobally}
Let $u,v \in \Vset$. If all edges $e$ of $\pis(u,v)$ have weight $\thetas e \geq 4\eps$, then we can decompose the paths $\pis(u,v) \subseteq \Ts$, $\pih(u,v) \subseteq \Th$ as $$ \pis(u,v) = \pis(a_1,b_1) \circ (b_1,a_2) \circ \pis(a_2,b_2) \circ \dots \circ (b_{t-1},a_t) \circ \pis(a_t,b_t), $$ $$\pih(u,v) = \pih(a'_1,b'_1) \circ (b'_1,a'_2) \circ \pih(a'_2,b'_2) \circ \dots \circ (b'_{t-1},a'_t) \circ \pih(a'_t,b'_t), $$  where $a_i,b_i,a'_i,b'_i \in V_{\sigma(i)} \mbox{ for all } i \in [t]$ for some permutation $\sigma: [m]\to [m]$. 
\end{claim}
\begin{proof}
Since each set $V_i$ induces a subtree $\Ts[V_i]$ in $\Ts$ (by Claim~\ref{claim:preprocessingconnectedsubtrees}), we can uniquely decompose any path into a union of paths over the subtrees and edges between the subtrees: $$ \pis(u,v) = \pis(a_1,b_1) \circ (b_1,a_2) \circ \pis(a_2,b_2) \circ \dots \circ (b_{t-1},a_t) \circ \pis(a_t,b_t), $$ such that there exists a permutation $\sigma :[m]\to[m]$ with $V(\pis(a_i,b_i)) \subseteq V_{\sigma(i)} \mbox{ for all } i \in [t]$.

Let $a_1' = a_1 = u$, $b_t' = b_t = v$. By Claim~\ref{claim:strongTbetweensetsarealsoinTprime}, for each $i \in [t-1]$ there is an edge $(b_i',a_{i+1}')$ such that $b_i' \in V_{\sigma(i)}, a_{i+1}' \in V_{\sigma(i+1)}$, i.e. there is a primed edge connecting each pair of groups $V_{\sigma(i)},V_{\sigma(i+1)}$ that is connected by a non-primed edge. Because for each $i\in[t]$, $\Th[V_i]$ is a subtree of $\Th$ and hence is connected, we see that $\Th[V_{\sigma(1)}\cup\dots\cup V_{\sigma(t)}]$ together with edges $(b_1',a_2')\cup\dots\cup(b'_{t-1},a'_t)$ forms a connected graph. It follows that $\pih(u,v)$ cannot have any additional edges between $V_i$ groups, i.e.,
the path $\pih(a_i',b_i')$ stays in $V_{\sigma(i)}$ for all $i \in [t]$. Thus $$\pih(a'_1,b'_1) \circ (b'_1,a'_2) \circ \pih(a'_2,b'_2) \circ \dots \circ (b'_{t-1},a'_t) \circ \pih(a'_t,b'_t)$$ is a path from $u$ to $v$ in $\Th$, and it satisfies the required properties.
\end{proof}
The next claim shows that if the path has a weak edge, then the reconstruction will still be reasonable:
\begin{claim}[Reconstructed topology correct -- very weak edge case]\label{claim:pathswithweakedges}
Let $u,v \in \Vset$ such that $\pis(u,v) \subseteq \Ts$ has an edge $e$ with $\thetas e \leq 4\eps$. Then $\pih(u,v) \subseteq \Th$ has an edge $e'$ with $\thetah{e'} \leq 5\eps$.
\end{claim}
\begin{proof}
Consider the tree $\TCL$ constructed in $\LearnFerromagneticModel$ and let $\picl{(u,v)}$ be the path from $u$ to $v$ in $\TCL$. 
 Note that by the construction of $\Th$, the path $\pih{(u,v)}$ includes all edges $e\in \picl(u,v)$ with $\mut{e}\leq 1/10$ as these are in $W$. It follows that if there is some edge $\tilde{e}$ in $\picl(u,v)$ with $\mus{\tilde{e}} \leq 4\eps$, then $\mut{\tilde{e}} \leq \mus{\tilde e} + \eps \leq 5\eps \leq 1/10$, and hence $\tilde{e} \in W$, so $\pih(u,v)$ would contain this edge $\tilde{e}$ and we would be finished. Henceforth assume that $\mus{\tilde{e}} > 4\eps$ for all  $\tilde{e}$ in $\picl(u,v)$. In particular, $\thetas{\tilde{e}} > 4\eps$ for all $\tilde{e}$ in $\picl(u,v)$. 

Now, the hypothesis of the claim is that $\pis(u,v) \subseteq \Ts$ has an edge $e$ of weight $\thetas{e} \leq 4\eps$. This contradicts the fact that $\Ts$ is a maximum spanning tree on the weights $\mus{}$ (Claim~\ref{claim:spanningtreetheta}), because $\picl(u,v)\cup \pis(u,v)$ must form a cycle containing $e$, and removing $e$ from $\Ts$ and adding in an edge of weight greater than $4\eps$ from the cycle would strictly increase the weight of the tree.  
\end{proof}

To summarize, Claims~\ref{claim:pathsagreeglobally} and \ref{claim:pathswithweakedges} show that between the sets $V_1,\ldots,V_m$ of the partition, $\Th$ roughly recovers the topology of $\Ts$.

\subsubsection{Bounding $|\muh{uv} - \Upsilon_{uv}|$}

We will now prove the claimed upper bound on $|\muh{uv} - \Upsilon_{uv}|$, where we recall that $\muh{uv} = \prod_{i=1}^t \muh{a_i' b_i'} \prod_{i=1}^{t-1} \muh{b_i' a_{i+1}'}$ and $\Upsilon_{uv} = \prod_{i=1}^t \mus{a_i' b_i'} \prod_{i=1}^{t-1} \mus{b_i' a_{i+1}'}$ for $a_i',b_i'$ decomposing the path $\pih(u,v)$ as in the proof of Proposition~\ref{prop:preprocessing}. The main result is:
\begin{claim}\label{claim:upsilonthetaprimediff}
If all edges $e$ in the path $\pis(u,v)$ have weight $\thetas{e} \geq 4\eps$, then $|\muh{uv} - \Upsilon_{uv}| \leq (C' + 1)\eps$, where $C'$ is the constant from Claim~\ref{claim:correctwithinVi}.
\end{claim}

The proof of the claim will rely on the following technical bound:

\begin{lemma}[Telescoping bound]\label{lem:technicaltelescopingpreprocessingbound}
Let $0 \leq \rho \leq 1$. For any $\alpha,\beta \geq 0$, $\gamma_1,\ldots,\gamma_t \in [0,\rho]$, and $\gamma'_1,\ldots,\gamma'_t \in [0,\rho]$ such that $|\gamma_i - \gamma'_i| \leq \delta$ for all $i \in [t]$, we have \begin{align*}
\left|\alpha \cdot \prod_{i=1}^{t} \gamma'_i - \beta \cdot \prod_{i=1}^{t} \gamma_i\right| \leq |\alpha-\beta|\rho^t + \alpha t \delta \rho^{t-1}
\end{align*}
\end{lemma}
\begin{proof} 
By the triangle inequality and the assumption that $\gamma_i,\gamma'_i \le \rho$,
\[ \left|\alpha \cdot \prod_{i=1}^{t} \gamma'_i - \beta \cdot \prod_{i=1}^{t} \gamma_i\right| \le |\alpha - \beta| \prod_{i = 1}^t \gamma_i + \alpha \left|\prod_{i=1}^{t} \gamma'_i - \prod_{i=1}^{t} \gamma_i\right| \le |\alpha - \beta| \rho^t + \alpha\left|\prod_{i=1}^{t} \gamma'_i - \prod_{i=1}^{t} \gamma_i\right|. \]
Finally, the last term in the absolute value can be bounded by the triangle inequality by considering intermediate terms $S_j := \left(\prod_{i = 1}^j \gamma'_i\right) \left(\prod_{j + 1}^m \gamma_i\right)$ and using that $|S_j - S_{j + 1}| \le \delta \rho^{m - 1}$.
\end{proof}

We are now ready to prove Claim~\ref{claim:upsilonthetaprimediff}.

\begin{proof}[Proof of Claim~\ref{claim:upsilonthetaprimediff}]
Applying Lemma~\ref{lem:technicaltelescopingpreprocessingbound} with $\alpha = \prod_{i=1}^t \mus{a_i'b_i'}$, $\beta = \prod_{i=1}^t \muh{a_i'b_i'}$, $\gamma_i = \mus{b_i' a_{i+1}'}$, $\gamma_i' = \muh{b_i' a_{i+1}'}$, $\rho = (1/10 + \eps)$,
\begin{equation}\label{ineq:firstfirsttriangle}
    |\muh{uv} - \Upsilon_{uv}| \leq \Big|\prod_{i=1}^t\mus{a_i'b_i'} - \prod_{i=1}^t \muh{a_i'b_i'}\Big| \cdot (1/10 + \eps)^{t-1} + \eps \cdot (t-1) (1/10 + \eps)^{t-2}\,,
\end{equation}
where we have used $\alpha,\beta \leq 1$, $\gamma_i' = \muh{b_i' a_{i+1}'}  = \mut{b_i' a_{i+1}'}\leq 1/10$ (the equality is due to $(b_i',a_{i+1}')$ being a edge of $\Th$ and the inequality by Claim~\ref{claim:distinctfar} since $b_i'$ and $a_{i+1}'$ are in distinct sets $V_{\sigma(i)},V_{\sigma(i+1)}$ of the partition), and $|\gamma_i - \gamma_i'| \leq \eps$ since $\mut{}$ $\eps$-approximates $\mus{}$.

Again applying Lemma~\ref{lem:technicaltelescopingpreprocessingbound}, this time with $\alpha = \beta = 1$, $\gamma_i = \mus{a_i'b_i'}$, $\gamma_i' = \muh{a_i'b_i'}$, $\rho = 1$,
\begin{equation}\label{ineq:firstsecondtriangle}
    \Big|\prod_{i=1}^t \mus{a_i'b_i'} - \prod_{i=1}^t \muh{a_i'b_i'}\Big| \leq t C' \eps\,,
\end{equation}
where we have used $|\gamma_i - \gamma_i'| \leq C' \eps$ by Claim~\ref{claim:correctwithinVi} (approximation guarantee for paths within the set $V_{\sigma(i)}$).
Combining \eqref{ineq:firstfirsttriangle} and \eqref{ineq:firstsecondtriangle} yields $$|\muh{uv} - \Upsilon_{uv}| \leq \eps \cdot (C' \cdot t(1/10 + \eps)^{t-1} + (t-1)(1/10 + \eps)^{t-2}) \leq (C' + 1)\eps\,.$$
The last inequality can be obtained by a simple calculus argument applied to $t x^{t-1}$ showing that for $x\leq 1/e$ the maximum over integers $t$ is obtained at $t=1$, and then separately bounding each of the two terms.
\end{proof}

\subsubsection{Bounding $|\Upsilon_{uv} - \mus{uv}|$}

In this section, we will prove the claimed upper bound on $|\Upsilon_{uv} - \mus{uv}|$, where recall that $\Upsilon_{uv} = \prod_{i=1}^{t} \mus{a_i'b_i'} \prod_{i=1}^{t-1} \mus{b_i'a_{i+1}'}$ and $\mus{uv} = \prod_{i=1}^t \mus{a_i b_i} \prod_{i=1}^{t-1} \mus{b_i a_{i+1}}$ for $a_i',b_i'$ decomposing $\pih(u,v)$ and $a_i,b_i$ decomposing $\pis(u,v)$ as in the proof of Proposition~\ref{prop:preprocessing}. The case that there is a very weak edge $e'$ with $\thetas{e'}\leq 4\eps$ in $\pis(u,v)$ is addressed at the start of the proof of Proposition~\ref{prop:reducetolowerbounded} in Section~\ref{ssec:ProofForLearnFerro}, so as noted there, we assume that all edges in $\pis(u,v)$ have weight at least $4\eps$.

\begin{claim}\label{claim:upsilonthetadiff}
If all edges $e$ in the path $\pis(u,v)$ have weight $\thetas e \geq 4\eps$, then $|\Upsilon_{uv} - \mus{uv}| \leq 4\eps$.
\end{claim}

\begin{proof}
It is clear that $\Upsilon_{uv} \leq \mus{uv}$ via a sequence of applications of the triangle inequality (Fact~\ref{fact:triangleineq}), since $a_1' = u, b_t' = v$. It remains to show that $\Upsilon_{uv} \geq \mus{uv} - 4\eps$. By Claim~\ref{claim:pathsagreeglobally}, $a_i,b_i,a'_i,b'_i \in V_{\sigma(i)}$ for all $i \in [t]$, for some permutation $\sigma : [m]\to [m]$. Applying the triangle inequality (Fact~\ref{fact:triangleineq}) for each $i \in [t]$ we have \begin{equation}\label{ineq:triangspecial}\mus{a'_i b'_i} \geq\mus{a'_i a_i} \mus{a_i b_i} \mus{b_i b'_i}\,.\end{equation} 
Using Fact~\ref{fact:corrPaths}, we can also rewrite $\mus{b'_i a_{i+1}'}$ for any $i \in [t-1]$ as
\begin{equation}\label{eq:triangeq}\mus{b'_i a_{i+1}'} = \mus{b'_i b_i} \mus{b_i a_{i+1}} \mus{a_{i+1} a'_{i+1}}\,,\end{equation} 
since $(b_i,a_{i+1})$ is an edge in $\pis(b'_i, a_{i+1}')$. To see this, note that removing $(b_i,a_{i+1})$ from $\Ts$ splits it into two components, $U \supseteq V_{\sigma(i)} \ni b'_i$ and $U' \supseteq V_{\sigma(i+1)} \ni a_{i+1}'$. Combining these last two equations gives
\begin{align*}\Upsilon_{uv} &\geq \prod_{i=1}^t \mus{a_i b_i} \prod_{i=1}^{t-1} \mus{b'_i a'_{i+1}} \mus{b_i' b_i} \mus{a_{i+1}' a_{i+1}} & \mbox{by \eqref{ineq:triangspecial}, since } a_1' = a_1, b'_t = b_t \\ &= \prod_{i=1}^t \mus{a_i b_i} \prod_{i=1}^{t-1} (\mus{b_i' a_{i+1}'})^2 / \mus{b_i a_{i+1}} & \mbox{by \eqref{eq:triangeq} and }\mus{b_ia_{i+1}} \geq 4\eps > 0  \\ &\geq \prod_{i=1}^t \mus{a_i b_i} \prod_{i=1}^{t-1} (\mus{b_ia_{i+1}} - 4\eps) & \mbox{by deferred Claim~\ref{claim:auxintersetthetadiff} below}. \end{align*}
Because $\Upsilon_{uv} \leq \mus{uv}$, it follows that 
\begin{align*}
|\Upsilon_{uv} - \mus{uv}| &\leq \Big|\prod_{i=1}^t \mus{a_i b_i} \prod_{i=1}^{t-1} (\mus{b_ia_{i+1}} - 4\eps) - \mus{uv}\Big| \\ &= \Big|\prod_{i=1}^t \mus{a_i b_i} \prod_{i=1}^{t-1} (\mus{b_ia_{i+1}} - 4\eps) - \prod_{i=1}^t \mus{a_i b_i} \prod_{i=1}^{t-1} \mus{b_ia_{i+1}}\Big|\,.
\end{align*}
We now apply the bound from Lemma~\ref{lem:technicaltelescopingpreprocessingbound} with $\alpha = \beta = \prod_{i=1}^t \mus{a_ib_i}$, $\gamma_i = \mus{b_i,a_{i+1}}$, $\gamma'_i = \mus{b_i,a_{i+1}} - 4\eps$, and $\rho = 1/10 + \eps$, to obtain
$$|\Upsilon_{uv} - \mus{uv}| \leq (4\eps)(t-1)(1/10 + \eps)^{t-2} \leq 4\eps,$$ where we have used $\alpha \leq 1$ and $\gamma_i,\gamma'_i \in [0,1/10+\eps]$ by Claim~\ref{claim:distinctfar}.
\end{proof}

\begin{claim}\label{claim:auxintersetthetadiff}
In the setup of Claim~\ref{claim:upsilonthetadiff},
$(\mus{b'_ia'_{i+1}})^2 / \mus{b_i a_{i+1}} \geq \mus{b_i a_{i+1}} - 4\eps$ for all $i \in [t-1]$.
\end{claim}
\begin{proof}
First, by the maximum spanning tree construction in $\LearnFerromagneticModel$, 
$$
\mut{b_i'a'_{i+1}} = \max_{w \in V_{\sigma(i)}, w' \in V_{\sigma(i+1)}} \mut{ww'}\,,
$$ 
and, by the fact that $\Ts$ is a maximum spanning tree on $\mus{}$ (Claim~\ref{claim:spanningtreetheta}), 
$$
\mus{b_i,a_{i+1}} = \max_{w \in V_{\sigma(i)}, w' \in V_{\sigma(i+1)}} \mus{ww'}\,.
$$ 
Because $\mut{}$ $\eps$-approximates $\mus{}$ and the maximums of two functions that are point-wise within $\eps$ are also within $\eps$, it follows that $$|\mut{b_i' a'_{i+1}} - \mus{b_i a_{i+1}}| \leq \eps\,.$$ 
As a consequence, 
$$|\mus{b'_i,a'_{i+1}} - \mus{b_i,a_{i+1}}| \leq |\mut{b_i' a'_{i+1}} - \mus{b_i a_{i+1}}| + |\mut{b'_ia'_{i+1}} - \mus{b'_ia_{i+1}'}| \leq 2\eps\,.
$$
(Here for the second term we directly used that $\mut{}$ and $\mus{}$ are within $\eps$.)
Using the last equation together with the assumption $\mus{b_ia_{i+1}} = \thetas{b_ia_{i+1}} \geq 4\eps>2\epsilon$ shows that $(\mus{b_i'a'_{i+1}})^2\geq (\mus{b_ia_{i+1}} - 2\eps)^2$, so 
$$
(\mus{b_i'a'_{i+1}})^2/\mus{b_ia_{i+1}}  \geq (\mus{b_ia_{i+1}} - 2\eps)^2/\mus{b_ia_{i+1}} = \mus{b_ia_{i+1}} - 4\eps + 4\eps^2 / \mus{b_ia_{i+1}} \geq \mus{b_ia_{i+1}} - 4\eps\,,
$$
which proves the claim.
\end{proof}
\section{Additive metric reconstruction: handling strong edges, step 1}
\label{sec:global-reconstruction}
In this section, we present our $\AdditiveMetricReconstruction$ algorithm, which is the first step in the tree metric reconstruction procedure underlying $\LearnLowerboundedModel$ which handles the strong edges of the model (see Section~\ref{ssec:tmroverview}). We prove the following guarantee, recalling from Definition~\ref{def:additiveMetric} that an additive metric is a tree metric with Steiner nodes:
\begin{theorem}[Restatement of Theorem~\ref{thm:additivemetricreconstruction}]
There exists a constant $C > 0$ such that the following result is true for arbitrary $\epsilon > 0$ and any $L \geq 100 \epsilon$.
Let $(\Ts,\d)$ be an unknown tree metric on vertex set $\Vset$. Suppose that neighboring vertices $u,v \in \Vset$ satisfy $\d(u,v) \le L$, and suppose we have access to a distance estimate $\dtpre : \cV \times \cV \to \RR \cup \{\infty\}$ such that
\begin{align}
\d(u,v) &\leq \dtpre(u,v), \qquad & \mbox{ for all } u,v \in \Vset \label{ineq:dprebigger-detailedproof} \\
\dtpre(u,v) &\leq \d(u,v) + \eps &\mbox{ for all } u,v \in \Vset \mbox{ such that }\d(u,v) \leq 3\Lupp \label{ineq:dpresmaller-detailedproof} \end{align}
Given this information, there exists a $O(|\Vset|^2)$-time algorithm $\AdditiveMetricReconstruction$ which constructs an additive metric $(\Th, \dh)$ with $|V(\Th)| \leq O(|\Vset|)$ Steiner nodes such that
\begin{equation}\label{eqn:global-approx-goal-detailedproof}
    |{\dh}(u,v) -\d(u,v)| \le C (\d(u,v)/L + 1) \epsilon
\end{equation}
for all $u,v \in \Vset$.
\end{theorem}

Before proving Theorem~\ref{thm:additivemetricreconstruction}, let us review related results in the literature and state how our theorem differs. Reconstruction of tree metrics is a well-studied problem related to phylogenetic reconstruction (see e.g. \cite{daskalakis2011evolutionary} and references within). In the setting where all pairwise distances are known to be estimated accurately, the additive metric (Steiner node) version of the problem has a known solution by Agarwala et al. \cite{agarwala1998approximability}. This follows from an important connection between the tree metric reconstruction problem and reconstruction of ultrametrics (defined later). This is useful because optimal ultrametric tree reconstruction in $\ell_{\infty}$ can (remarkably) be solved in polynomial time \cite{farach1995robust}; this follows from the existence of \emph{subdominant ultrametrics} which we will discuss later. This reduction leads to the following $3$-approximation algorithm for reconstructing additive metrics (i.e. tree metrics with Steiner nodes):
\begin{theorem}[Agarwala et al. \cite{agarwala1998approximability}]\label{thm:agarwala}
Suppose that $\d$ is an additive metric on $\Vset$ and $\dt$ satisfies $|\dt(x,y) - \d(x,y)| \le \epsilon$ for all $x,y \in \Vset$. Then an additive metric $(\Th,\dh)$ such that
\[ \max_{x,y} |\dh(x,y) - \d(x,y)| \le 3\epsilon\,, \] can be computed from $\dt$ in time $O(|\Vset|^2)$. 
\end{theorem}

Unfortunately, the additive metric reconstruction algorithm of Agarwala et al. \cite{agarwala1998approximability} cannot be used in our setting. In contrast to Theorem~\ref{thm:agarwala}, our Theorem~\ref{thm:additivemetricreconstruction} is not guaranteed that $\dt(u,v)$ is $\eps$-close to $\d(u,v)$ for every pair of vertices $u,v$ -- the approximation only holds for close vertices. Thus, in our setting for two vertices that are far away the distance cannot be estimated up to $O(\eps)$ additive error, and our algorithm settles for reconstructing the distance up to a $(1 + O(\eps))$ multiplicative factor. (In the statistical setting we consider, reconstructing the distances between far-away vertices up to $O(\eps)$ error would require $\Omega(n)$ samples, because the relevant correlations can be as small as $O(1/n)$, whereas we aim to use $O(\log n)$ samples).

We discuss the main idea behind Theorem~\ref{thm:additivemetricreconstruction}. While the additive metric reconstruction algorithm of Agarwala et al.~\cite{agarwala1998approximability} cannot be applied in our setting, the underlying mathematical tools from metric geometry remain useful and we discuss these in the first subsection. At a high level, the key idea we take from \cite{agarwala1998approximability} is that we can reconstruct an additive metric close to $\d$ provided we: (1) compute a close-to-tight upper bound $\dt$ on the true metric $\d$ and (2) show the existence of a ``witness tree'' $\d'$ which is close to $\d$ and matches root-to-leaf distances in $\dt$ from some arbitrary choice of root $\rho$. Here ``close'' needs to be in the same sense as \eqref{eqn:global-approx-goal-detailedproof}, i.e. close nodes are additively close and far-away nodes are multiplicatively close.  
For step (1), it is unacceptable in our setting to use the approximate distances $\dtpre$ directly, as those distances may be very inaccurate between far-away nodes. Instead, we show for (1) that taking $\dt$ as the \emph{shortest-path metric} over \emph{upper (confidence) bounds} on the estimated distances suffices to construct a suitable $\dt$ --- see Lemma~\ref{lem:sp-good}. For (2), we construct a witness tree based on $\d,\dt$ using a new two-step procedure, which carefully balances the opposing goals of matching root-to-leaf distances in $\dt$ while preserving the overall metric approximation of $\d$ --- see Lemma~\ref{lem:witness}. Combining key inputs (1) and (2) lets us use a similar combinatorial algorithm to \cite{agarwala1998approximability}, based on reduction to computation of a \emph{subdominant ultrametric}, to compute an additive metric approximation.

\subsection{Preliminaries from (ultra)metric geometry}\label{sec:prelim-metric}
\paragraph{Pseudometrics.} A tuple $(\Vset,\d)$, where $\d : \Vset \times \Vset \to \mathbb{R}_{\ge 0}$, is a \emph{pseudometric space} if $\d$ satisfies the following axioms:
\begin{enumerate}
    \item $\d(u,u) = 0$ for all $u \in \Vset$,
    \item $\d(u,v) = \d(v,u)$ for all $u,v \in \Vset$, and
    \item $\d(u,v) \le \d(u,w) + \d(w,v)$ for all $u,v,w \in \Vset$.
\end{enumerate}
Normally a pseudometric is called a metric when it satisfies the additional axiom that $\d(u,v) = 0$ iff $u = v$. In what follows this distinction is counterproductive; both terms will refer to a pseudometric unless explicitly noted otherwise.

\paragraph{Ultrametrics.} 
Ultrametrics are a special kind of metric where a very strong form of the triangle inequality holds \eqref{eqn:ultrametric}; they play a fundamental role in metric geometry and its applications to probability theory, physics, computer science, etc. --- see the survey \cite{naor2012introduction}.
\begin{definition}[Ultrametric]
Suppose $(\Vset, \d)$ is a metric space. We say that $\d$ is an \emph{ultrametric} if
\begin{equation}\label{eqn:ultrametric}
\d(u,v) \le \max \{\d(u,w),\d(w,v)\} 
\end{equation}
for all $u,v,w \in \Vset$.
\end{definition}
\begin{definition}[Subdominant ultrametric]
Let $\a : \Vset \times \Vset \to \mathbb{R}_{\ge 0}$ be an arbitrary function. We say that an ultrametric $\d: \Vset \times \Vset \to \mathbb{R}_{\ge 0}$ is the \emph{subdominant ultrametric} of $\a$ if it satisfies the following properties:
\begin{enumerate}
    \item For all $u,v \in \Vset$, $\d(u,v) \le \a(u,v)$.
    \item For any other ultrametric $\e$ such that $\e(u,v) \le \a(u,v)$ for all $u,v \in \Vset$, it holds true that $\e(u,v) \le\d(u,v)$ for all $u,v \in \Vset$.
\end{enumerate}
\end{definition}
\noindent
Remarkably, for any $\a : \Vset \times \Vset \to \mathbb{R}_{\ge 0}$ on a finite\footnote{When $\Vset$ is infinite, subdominant ultrametrics may not exist. See \cite{bayod1990subdominant} for a more general existence result.} set $\Vset$, a subdominant ultrametric always exists and can be efficiently computed. 
This result has a long history and shows up in the statistical physics literature \cite{rammal1985degree,rammal1986ultrametricity}, the literature on numerical taxonomy \cite{jardine1971mathematical}, and studies on hierarchical clustering \cite{murtagh1983survey} among other places. A simple algorithm for computing the subdominant ultrametric based off of minimum spanning trees, which gives the result stated below, appears at least in \cite{farach1995robust} and \cite{rammal1985degree}.
\begin{theorem}[\cite{farach1995robust,rammal1985degree}]\label{thm:subdominant}
Let $\a : \Vset \times \Vset \to \mathbb{R}_{\ge 0}$ be an arbitrary function and suppose $\Vset$ is finite. Then the subdominant ultrametric of $\a$ exists and it can be computed in time $O(|\Vset|^2)$.
\end{theorem}
\paragraph{Ultrametric $=$ Additive metric $+$ Centroid metric.} 
An ultrametric on a finite set $\Vset$ can be visualized as a tree where the leaves are the elements of $\Vset$, the leaves are all equidistant from the root of the tree, and the least common ancestor of $u,v \in \Vset$ is at height $\d(u,v)/2$ above them. This description makes it clear that ultrametrics are a special case of additive metrics. It turns out there is a simple way to convert an additive metric to an ultrametric by adding an appropriate \emph{centroid metric} \cite{barthelemy1991trees,agarwala1998approximability}.
\begin{definition}[Centroid Metric]
Suppose that for all $u \in \Vset$ we are given $\ell_u \ge 0$. Then we say that $\csf : \Vset \times \Vset \to \mathbb{R}_{\ge 0}$ is the corresponding \emph{centroid metric} if $\csf(u,v) = \ell_u + \ell_v$ for $u \ne v$.
\end{definition}
\noindent
Note that the centroid metric is also an additive metric; it can be visualized as the graph metric on a star graph where the leaves are $u \in \Vset$ and the edges have length $\ell_u$.

The following lemma (established in the proof of Claim 3.5B of \cite{agarwala1998approximability}) shows that we can convert any additive metric into an ultrametric by adding an appropriate centroid metric. Equivalently, this ultrametric can be defined by the equation $\e(u,v) = 2 D_{\mathrm{max}} - 2\d(\rho,\LCA_{\rho}(u,v))$ below.
\begin{lemma}[\cite{agarwala1998approximability}]\label{lem:additive-to-ultra}
Suppose $\d$ is an additive metric on $\Vset$, $\rho \in \Vset$, and $\ell_u = D_{\mathrm{max}} - \d(\rho,u)$ for all $u$ where $D_{\mathrm{max}} \ge \max_u \d(\rho,u)$. Let $\csf$ be the centroid metric formed from $\{\ell_u\}_{u \in \Vset}$. Then $\e = \d + \csf$ is an ultrametric such that 
\[ 2D_{\mathrm{max}} \ge \e(u,v) = 2 D_{\mathrm{max}} - 2\d(\rho,\LCA_{\rho}(u,v)) \ge 2 \max\{\ell_u,\ell_v\}\,. \]
\end{lemma}
In the reverse direction, if we have an ultrametric and a centroid metric which satisfy the appropriate conditions, then subtracting the centroid metric from the ultrametric gives an additive metric. This is established in the proof of Claim 3.5A in \cite{agarwala1998approximability}.
\begin{lemma}[\cite{agarwala1998approximability}]\label{lem:ultra-to-additive}
Suppose that $\e$ is an ultrametric on $\Vset$, $\csf(u,v) = \ell_u + \ell_v$ is a centroid metric on $\Vset$, and $\rho \in \Vset$. Let $D_{\mathrm{max}} = \ell_{\rho}$. Suppose that
\[ \e(\rho, u) = 2 D_{\mathrm{max}} \]
for all $u \in \Vset$,
and for all $u,v \in \Vset$
\[ 2D_{\mathrm{max}} \ge \e(u,v) \ge 2 \max\{\ell_u,\ell_v\}. \]
Then $\d = \e - \csf$ is an additive metric on $\Vset$ and $\ell_u = D_{\mathrm{max}} - \d(\rho,u)$ for all $u$.
\end{lemma}

\subsection{Reconstructing an accurate additive metric}
The pseudocode for $\AdditiveMetricReconstruction$ is given as Algorithm~\ref{alg:additivemetricreconstruction}.

\begin{algorithm}\SetAlgoLined\DontPrintSemicolon
    \KwIn{$\dtpre$, $L$, $\eps$ as in Theorem~\ref{thm:additivemetricreconstruction}}
    \KwOut{Additive metric $(\Th,\dh)$ satisfying the approximation guarantee \eqref{eqn:global-approx-goal-detailedproof}}
    
    Fix an arbitrary root $\rho \in \Vset$.
    
    Using Dijkstra's algorithm, compute shortest path distances $\dt(\rho,u)$ for all $u$ from the weighted graph corresponding to $\dtpre$.
    
    Define $D_{\mathrm{max}} := \max_u \dt(\rho,u)$, $\ell_u := D_{\mathrm{max}} - \dt(\rho,u)$ for all $u \in \Vset$, and let $\csf(u,v) = \ell_u + \ell_v$ be the corresponding centroid metric.
    
    Define $\a(u,\rho) := a(\rho,u) = 2D_{\mathrm{max}}$. For all $u,v \in \Vset \setminus \{\rho\}$ such that $\dtpre(u,v) < 3L$, define 
    \[ \a(u,v) := \min \{2D_{\mathrm{max}}, \dtpre(u,v) + \csf(u,v) + 44\epsilon\}. \]
    
    Let $\e$ be the subdominant ultrametric of $\a$, computed using the algorithm from Theorem~\ref{thm:subdominant}.
    
    Return $\dh := \e - \csf$, and the corresponding tree with Steiner nodes $\Th$.
   
    \caption{$\AdditiveMetricReconstruction(\dtpre,L,\epsilon)$}
\label{alg:additivemetricreconstruction}
\end{algorithm}

We prove correctness, Theorem~\ref{thm:additivemetricreconstruction}. In the analysis, we let $\dt$ be the shortest path metric with respect to $\dt$: 
\begin{equation}\label{eq:shortestPath}
\dt(u,v) = \min_{\pi\in\mathrm{paths}(u,v)}\sum_{e=(a,b)\in \pi}\dtpre(a,b)\,.
\end{equation}
This is consistent with the algorithm, although in the algorithm only shortest paths from $\rho$ are actually used or computed; the shortest paths between other pairs of nodes are useful in the analysis. We start with a lemma relating $\d$ and $\dt$.

\begin{lemma}\label{lem:sp-good}
Under the assumptions of Theorem~\ref{thm:additivemetricreconstruction}, for all $u,v \in \Vset$
\[\d(u,v) \le \dt(u,v) \le\d(u,v) + (\d(u,v)/L + 1)\epsilon\,. \]
\end{lemma}
\begin{proof}
By \eqref{ineq:dprebigger-detailedproof} and the triangle inequality, we know $\d(u,v) \le \dt(u,v)$ which proves the lower bound.

The upper bound follows by considering the path from $u$ to $v$ 
in $\T$: since the maximum edge length in the ground truth tree $\T$ is $L$, we can choose vertices $u = w_1,\ldots,w_k$ along the path from $u$ to $v$ such that $\d(w_i,w_{i + 1}) \in [L,2L]$ for all $i < k$ and $\d(w_k, v) \le L$. (Note that $k-1\leq \d(u,v)/L$.)
Now because $\dt$ is the shortest path metric, 
$$
\dt(u,v)\leq \sum_{i < k} \dtpre(w_i,w_{i + 1}) + \dtpre(w_k,v) \stackrel{\eqref{ineq:dpresmaller-detailedproof}}\leq \sum_{i < k} \d(w_i,w_{i + 1}) + \d(w_k,v) + k\eps\,.
$$
Using $\d(u,v) = \sum_{i < k} \d(w_i,w_{i + 1}) + \d(w_k,v)$ due to $\d$ being a tree metric as well as the fact that $k-1\leq \d(u,v)/L$ shows the upper bound. 
\end{proof}

The following lemma shows that estimated root-to-leaf distances are consistent in the following sense: for any node $u$ with nearby descendant $v$, $\dt(\rho,v) - \dt(\rho,u)$ is a $O(\epsilon)$ additive approximation to $\d(u,v)$. 
Note that if $\dt = \d$ then this would be an exact equality. This is used in Lemma~\ref{lem:witness} below to argue that a slightly modified version of $\d$ can match root-to-leaf distances with $\dt$ exactly.

\begin{lemma}\label{lem:sp-consistent}
Under the assumptions of Theorem~\ref{thm:additivemetricreconstruction} the following is true. Choose an arbitrary root $\rho$ for $\T$. 
For any $u,v$ such that $\d(u,v) \le 2L$ and $u$ is an ancestor of $v$ in $\T$,
\[ \dt(\rho,u) +\d(u,v) - 4\epsilon \le \dt(\rho,v) \le \dt(\rho,u) +\d(u,v) + \epsilon\,. \]
\end{lemma}

\begin{proof}
To prove the upper bound, we observe that $\dt(\rho,v) \le \dt(\rho,u) + \dt(u,v)$ by the triangle inequality for $\dt$ and then by \eqref{ineq:dpresmaller-detailedproof} $\dt(u,v) \le \dtpre(u,v) \le\d(u,v) + \epsilon$.

Now we prove the lower bound on $\dt(\rho,v)$. Consider the shortest path $\rho = a_1, a_2, \ldots, a_k = v$ from $\rho$ to $v$ in the weighted graph $\dtpre$. Let $a_i$ be the last node on this path that is not a descendant of $u$ in $\T$. Note that since 
$\d(u,v)\leq 2L$ by assumption, \eqref{ineq:dpresmaller-detailedproof} implies that $\dtpre(u,v)\leq 2L + \eps\leq 3L$. 
By \eqref{ineq:dprebigger-detailedproof} this gives $\d(a_i,a_{i+1}) \le 3L$. Since $a_i$ is not a descendant of $u$, but $a_{i+1}$ is a descendant of $u$, the path in $\T$ between $a_i$ and $a_{i+1}$ must contain $u$, so $\d(u,a_i) \leq 3L$.

By definition of $\dt$ and $a_i$, we have
$\dt(\rho,v) = \dt(\rho,a_i) + \dt(a_i,v)$, hence
\begin{align*} 
\dt(\rho,v) 
\ge \dt(\rho,a_i) + \d(a_i,v)
&= \dt(\rho,a_i) + \d(a_i,u) +\d(u,v) \\
&\geq \dt(\rho,a_i) + \tilde \d(a_i,u) +\d(u,v) - (3L/L + 1)\epsilon
\ge \dt(\rho,u) +\d(u,v) - 4\epsilon
\end{align*}
where in the first step we used Lemma~\ref{lem:sp-good}, in the second step we know $\d(a_i,v) = \d(a_i,u) +\d(u,v)$ because $v$ is an ancestor of $u$ and $a_i$ is not, in the third step we used Lemma~\ref{lem:sp-good} and $\d(u,a_i) \leq 3L$, and in the last step we used the triangle inequality. 
\end{proof}
The following Lemma~\ref{lem:witness} is crucial. It proves that there exists an additive metric $\d'$ which matches the root-to-leaf distances in $\dt$ and is also a good approximation to the ground truth $\d$. This is crucial for applying the tools for ultrametric reconstruction from Section~\ref{sec:prelim-metric}: more specifically, the existence of $\d'$ ensures that when we convert from $\dt$ to $\a$ and extract the subdominant ultrametric $\e$, that the distances in $\e$ are not shrunk by too much (see the proof of Theorem~\ref{thm:additivemetricreconstruction}). %

We prove the Lemma with an explicit construction of $\d'$. This construction assumes knowledge of the ground truth tree $\T$, and therefore cannot be run by the algorithm $\AdditiveMetricReconstruction$, which only has knowledge of $\dtpre$. Nevertheless, the constructive proof suffices to ensure the existence of $\d'$, which we will later show is sufficient for $\AdditiveMetricReconstruction$ to succeed.
\begin{lemma}\label{lem:witness}
Under the assumptions of Theorem~\ref{thm:additivemetricreconstruction} the following is true. 
There exists an additive metric $\d'$ such that
\begin{equation} |\d(u,v) - \d'(u,v)| \le (36 +2\d(u,v)/L)\epsilon \label{eq:globalapproxguaranteeexistence}\end{equation}
for all $u,v \in \Vset$ and
\begin{equation} \d'(\rho,u) = \dt(\rho,u) \label{eq:rhodistancematchingexistence}\end{equation}
for all $u \in \Vset$.
\end{lemma}
\begin{proof}
We construct $\d'$ in two phases, which we now overview. In the first phase, we subdivide $\T$ into a union of subtrees of diameter $\Theta(L)$. The roots of these subtrees are a subset of the nodes $S \subset \Vset$. By shortening or lengthening edges of $\T$ appropriately, we construct a $\T$-structured metric $\d''$ such that the condition $\d''(\rho,v) = \dt(\rho,v)$ holds for all $v \in S$ and such that $\d''$ approximates $\d$ in the sense of \eqref{eq:globalapproxguaranteeexistence}. In the second phase, we create the additive metric $\d'$ by moving each node in $\d''$ by $O(\eps)$ distance.

\textit{First phase}. Consider the ground truth tree $\T$ with edge lengths corresponding to the tree metric $\d$. Run a breadth-first search (BFS) from the root node $\rho$, inductively constructing a set of nodes $S \subset \Vset$, as follows. First let $\rho \in S$. Then, when the BFS reaches node $v$, add $v$ to $S$ if and only if the closest ancestor to $v$ that has been added to $S$ is at distance at least $L$ from $v$.

For any node $w \in \Vset \sm \{\rho\}$, let $\pa(w)$ denote the parent of $w$. Also, let $\pa_S(s)$ denote the closest (strict) ancestor of $w$ that lies in $S$. Note that $\pa_S(w)$ is guaranteed to exist, since the root $\rho$ is in the set $S$. By construction, $\d(\pa_S(s),s) \in [L,2L)$ for any node $s \in S \sm \{\rho\}$, since the maximum edge length is at most $L$ in the tree. Similarly, we have $\d(\pa_S(u),u) \leq 2L$ for any $u \in \Vset \sm \{\rho\}$.

We define $\d''$ by modifying the edge lengths in $\d$. First, shorten certain edges of $\d$: for any edge $(\pa(w),w)$ in $\T$, let 
\begin{equation*}\label{eq:edgeshorteningdprimeprime}
    \d''(\pa(w),w) \gets \begin{cases} \d(\pa(w),w) , & \d(\pa_S(w),\pa(w)) \geq 2\eps \\
    \max(0,\d(\pa(w),w) - 2\eps + \d(\pa_S(w),\pa(w))), &  \d(\pa_S(w),\pa(w)) < 2\eps
    \end{cases}
\end{equation*}
The reason for this shortening step is that for any $s \in S \sm \{\rho\}$, we have $\d''(\pa_S(s),s) = \max(0,\d(\pa_S(s),s) - 2\eps)$. After applying the edge shortening step, also update the length of the parent edge of each $s \in S \sm \{\rho\}$: \begin{equation*}\label{eq:edgelengtheningdprimeprime}\d''(\pa(s),s) \gets \d''(\pa(s),s) + (\dt(\rho,s) - \dt(\rho,\pa_S(s))) - (\d(\pa_S(s),s) - 2\eps)\,.\end{equation*}
This latter update is well-defined since the edge is lengthened (indeed, Lemma~\ref{lem:sp-consistent} implies $\dt(\rho,s) - (\dt(\rho,\pa_S(s)) + \d(\pa_S(s),s) - 2\eps) \geq 0$) and thus the length remains nonnegative. Together, these edge shortenings and lengthenings ensure that for any $v \in \Vset \sm S$, \begin{align}
\d''(\pa_S(v),v) &= \sum_{(\pa(w),w) \in \pi(\pa_S(v),v)} \d''(\pa(w),w) \nonumber \\
&= \max(0,\d(\pa_S(v),v) - 2\eps), \label{eq:nonSguarantee}
\end{align}
and similarly for any $s \in S \sm \{\rho\}$,
\begin{align}
\d''(\pa_S(s),s) &= \max(0,\d(\pa_S(s),s) - 2\eps) + (\dt(\rho,s) - \dt(\rho,\pa_S(s))) - (\d(\pa_S(s),s) - 2\eps) \nonumber \\
&= \dt(\rho,s) - \dt(\rho,\pa_S(s)), \label{eq:dprimeprimepascommas}
\end{align}
where to obtain the last line we use $\d(\pa_S(s),s) \geq L \geq 2\eps$. Since $\d''(\rho,s) = \d''(\rho,\pa_S(s)) + \d''(\pa_S(s),s)$, applying the above equation inductively on the depth of $s$, we obtain \begin{equation} \label{eq:Sguarantee}\d''(\rho,s) = \dt(\rho,s) \mbox{ for all } s \in S.\end{equation}

Furthermore, for any $u \in \Vset \sm \{\rho\}$, we have \begin{equation}|\d''(\pa_S(u),u) - \d(\pa_S(u),u)| \leq 2\eps.\label{eq:dprimeprimelocalguarantee}\end{equation} If $u \in \Vset \sm S$, this follows from \eqref{eq:nonSguarantee}. If $u \in S \sm \{\rho\}$, this follows from \eqref{eq:dprimeprimepascommas}, Lemma~\ref{lem:sp-consistent}, and $\d(\pa_S(u),u) \leq 2L$.

We now prove that for any $u,v \in \Vset$, \begin{equation}|\d''(u,v) - \d(u,v)| \leq (2\d(u,v)/L + 12)\eps.\label{eq:dprimeprimeglobalguarantee}
\end{equation}
Since $\d$ and $\d''$ are both tree metrics on $d$, we must have $\d(u,v) = \d(u,w) + \d(v,w)$ and $\d''(u,v) = \d''(u,w) + \d''(v,w)$ for $w = \LCAi\T\rho{u,v}$. Hence, it suffices to prove that, for any $u,w \in \Vset$ where $w$ is an ancestor of $u$,
\begin{equation}|\d''(u,w) - \d(u,w)| \leq (2\d(u,w)/L + 6)\eps.\label{eq:dprimeprimeglobalguaranteeancestor}
\end{equation} Let $s_1,\ldots,s_k$ be the vertices in $S$ along the path from $w$ to $u$ (i.e., $S \cap \pi(w,u)$). Note that \begin{align*}|\d''(s_1,s_k) - \d(s_1,s_k)| &= |\sum_{i=2}^k \d''(\pa_S(s_i),s_i) - \d(\pa_S(s_i),s_i)| \leq \sum_{i=2}^k |\d''(\pa_S(s_i),s_i) - \d(\pa_S(s_i),s_i)| \\
&\leq 2(k-1)\epsilon \leq 2\d(s_1,s_k)\epsilon/L\end{align*} by \eqref{eq:dprimeprimelocalguarantee} and the fact that $\d(s_1,s_k) \geq (k-1)L$ since each pair $s_i,s_{i-1}$ is at least at distance $L$. Finally, if $w = s_1$, then $\d''(w,s_1) = \d(w,s_1) = 0$ and if $w \neq s_1$, then $\pa_S(w) = \pa_S(s_1)$, so $|\d''(w,s_1) - \d(w,s_1)| = |\d''(\pa_S(s_1),s_1) - \d''(\pa_S(w),w) - \d(\pa_S(s_1),s_1) + \d(\pa_S(w),w)| \leq 4\epsilon$ by \eqref{eq:dprimeprimelocalguarantee}. Similarly, by \eqref{eq:dprimeprimelocalguarantee} we have $|\d''(s_k,u) - \d(s_k,u)| \leq 2\epsilon$. So overall \begin{align*}|\d''(w,u) - \d(w,u)| &\leq |\d''(w,s_1) - \d(w,s_1)| + |\d''(s_1,s_k) - \d(s_1,s_k)| + |\d''(s_k,u) - \d(s_k,u)| \\
&\leq (2\d(s_1,s_k) / L + 6)\epsilon \leq (2\d(w,u) / L + 6)\epsilon,\end{align*} proving \eqref{eq:dprimeprimeglobalguaranteeancestor}, and hence proving \eqref{eq:dprimeprimeglobalguarantee}.

\textit{Second phase}. Construct $\T'$ as follows: first, begin with a copy of $\T$ with edge lengths induced by $\d''$. However, let all the nodes in this copy of $\T$ now be Steiner nodes: i.e., node $u \in \Vset$ in $\T$ corresponds to a Steiner node $u^{\mathrm{stein}}$ in the copy $\T'$. Second, for each $u \in \Vset$, place node $u \in \Vset$ at distance $\dt(\rho,u)$ from $\rho$ as follows: if $\dt(\rho,u) \leq \d''(\rho,u)$, then place $u$ along the path from $\rho$ to $u^{\mathrm{stein}}$. Otherwise, if $\dt(\rho,u) \geq \d''(\rho,u)$, attach $u$ to $u^{\mathrm{stein}}$ with an edge of length $\dt(\rho,u) - \d''(\rho,u)$. This induces a $\T'$-structured tree metric $\d'$, which is an additive metric when restricted to the set $\Vset$. Note that $\d'(u,u^{\mathrm{stein}}) \leq 12\epsilon$ for all $u \in \Vset$, because if $u \in S$ then $\d'(u,u^{\mathrm{stein}}) = 0$ by \eqref{eq:Sguarantee}. And if $u \in \Vset \sm S$, then \begin{align*}\d'(u,u^{\mathrm{stein}}) &= |\d''(\rho,u) - \dt(\rho,u)| = |\d''(\rho,\pa_S(u)) + \d''(\pa_S(u),u) - \dt(\rho,u)|  \\
&= |\dt(\rho,\pa_S(u)) + \d''(\pa_S(u),u) - \dt(\rho,u)|
\\&\leq |\dt(\rho,\pa_S(u)) + \d(\pa_S(u),u) - \dt(\rho,u)| + 10\epsilon \leq 12\epsilon\,,\end{align*}
where in the second line we have used \eqref{eq:Sguarantee} for the first equality, \eqref{eq:dprimeprimeglobalguaranteeancestor} and $\d(\pa_S(u),u) \leq 2L$ for the first inequality, and Lemma~\ref{lem:sp-consistent} for the last inequality.
Hence, between any pair of vertices $u,v \in \Vset$, we have $|\d'(u,v) - \d''(u,v)| \leq \d'(u,u^{\mathrm{stein}}) + \d'(v,v^{\mathrm{stein}}) + |\d'(u^{\mathrm{stein}},v^{\mathrm{stein}}) - \d''(u,v)| = \d'(u,u^{\mathrm{stein}}) + \d'(v,v^{\mathrm{stein}}) \leq 24\epsilon$, proving the lemma when combined with \eqref{eq:dprimeprimeglobalguarantee}.
\end{proof}
We now apply this existence lemma to prove the correctness of $\AdditiveMetricReconstruction$.
\begin{proof}[Proof of Theorem~\ref{thm:additivemetricreconstruction}]
Define $\epsilon' = 44\epsilon$. We claim that $\d'(u,v) \leq \dtpre(u,v) + \epsilon'$ for any $u,v \in \Vset$. Indeed, if $\dtpre(u,v) > 3L$, then $\dtpre(u,v) = \infty$ so the claim holds. Otherwise, if $\dtpre(u,v) \leq 3L$, then $\d(u,v) \leq 3L + \epsilon$, so $\d'(u,v) \leq \d(u,v) + (36 + 2(3L + \epsilon)/L)\epsilon \leq \dtpre(u,v) + 44\epsilon$, by Lemma~\ref{lem:witness} and $\epsilon \leq L$.

This guarantees $\d'(u,v) + \csf(u,v) \le \a(u,v)$ for all $u,v$. Moreover, $\d' + \csf$ is an ultrametric by Lemma~\ref{lem:additive-to-ultra}, since $\ell_u = \dt(\rho,u) = \d'(\rho,u)$ for all $u$. So the subdominant ultrametric $\e$ satisfies $\d'(u,v) + \csf(u,v) \le \e(u,v)$, i.e. $\d'(u,v) \le \dh(u,v)$, which by the guarantee of Lemma~\ref{lem:witness} ensures that 
\[ \dh(u,v) \ge\d(u,v) - (2\d(u,v)/L + 36)\epsilon. \]
Also, $2\max \{\ell_u,\ell_v\} \le \d'(u,v) + \csf(u,v) \le \e(u,v)$ guarantees that we can apply Lemma~\ref{lem:ultra-to-additive} to show $\dh$ is an additive metric. 

On the other hand, by construction $\dh(u,v)$ is a metric such that $\dh(u,v) \le \dtpre(u,v) + \epsilon'$. Therefore applying the upper bound argument from Lemma~\ref{lem:sp-good} shows that 
\[ \dh(u,v) \le \dt(u,v) + (\dh(u,v)/L + 1)\epsilon', \]
so  by Lemma~\ref{lem:sp-good}
\[\dh(u,v) \leq (\dt(u,v) + \epsilon')/(1 - \epsilon'/L) \leq ((\d(u,v)/L + 1)\epsilon + \epsilon')/(1 - \epsilon'/L) \leq (2\d(u,v)/L + 90)\epsilon,\]
where we used used $\epsilon'/L \leq 1/2$ and $\epsilon' = 44\epsilon$. This completes the proof. 
\end{proof}

\section{Desteinerization: handling strong edges, step 2} \label{sec:desteinerize}

In this section we present our $\Desteinerize$ algorithm, which is the second ingredient in the tree metric reconstruction procedure underlying $\LearnLowerboundedModel$ which handles the strong edges of the model (see Section~\ref{ssec:tmroverview}). The essence of the $\Desteinerize$ result is that if we reconstruct an additive metric $\dh$ (with Steiner nodes) which is close to a ground truth tree metric $\d$ \emph{without} Steiner nodes, then it is possible to remove the Steiner nodes from $\dh$ without significantly changing the distances between nodes. 
In the process of proving this, we establish several useful basic lemmas (especially Lemma~\ref{lem:steinernodeclose}) which show that an accurate additive metric must ``essentially'' match the topology of the ground truth tree. In the following theorem, $\Desteinerize$ refers to Algorithm~\ref{alg:desteinerize}.

\begin{theorem}[Restatement of Theorem~\ref{thm:desteinerize}]
There exists a constant $C > 0$ such that the following is true. 
Let $L \ge 2\epsilon \ge 0$ and suppose $(\Ts,\d)$ is an unknown tree metric on vertex set $\Vset$ with maximum edge length $L$.
Suppose $\dh$ is an additive metric on $\Vset$ with Steiner tree representation $\Th$ such that
\begin{equation} |\d(x,y) - \dh(x,y)| \le (\d(x,y)/L + 1)\epsilon \label{ineq:desteinerizeinputguarantee}\end{equation}
for all $x,y \in \Vset$. Then \textsc{Desteinerize} with $(\Th,\dh)$ as input runs in $O(|\Vset|^2 + |V(\Th)|)$ time and outputs a tree metric $(\T',\d')$ on $\Vset$ such that
\[ |\d(x,y) - \d'(x,y)| \le (\d(x,y)/L + 1)\epsilon + C\epsilon \]
for all $x,y \in \Vset$.
\end{theorem}

\begin{remark}
We contrast this theorem with the famous result of \cite{gupta2001steiner}, which showed that any additive metric can be approximated by a tree metric (i.e. an additive metric without steiner nodes) while distorting distances by only a multiplicative factor of $8$. Thus, \cite{gupta2001steiner} showed that if we are satisfied with a constant factor in the distance approximation, then we never need to use Steiner nodes in the tree. In contrast, Theorem~\ref{thm:desteinerize} makes a stronger assumption, that the input additive metric is close to some tree metric, and makes an essentially stronger conclusion: the output $O(\eps/L)$-additively approximates short distances and $(1+O(\eps/L))$ -multiplicatively approximates long distances. This stronger approximation guarantee is needed downstream to reconstruct the tree model in local total variation distance.
\end{remark}

\begin{algorithm}\SetAlgoLined\DontPrintSemicolon
    \KwIn{Additive metric $\dh$ on $\Vset$ with Steiner tree representation $\Th$}
    \KwOut{Tree metric $\d'$ on $\Vset$ with tree representation $\T'$}

    \BlankLine
    
    While there exists a Steiner node of degree $1$ or $2$ in $\Th$, delete this node; this can be done without changing $\dh$ on elements of $\Vset$ by either deleting or merging edges. 
    \label{step:destein0}
    
    \BlankLine
     \tcp{Define clustering of nodes}
      For every node $v$ in $\Th$ define $f(v)$ to be the closest non-Steiner node according to $\dh$, breaking ties in favor of the node with lowest graph distance and then among those ties in favor of the lowest index node. (In Claim~\ref{claim:f-subtree} below we show that for every $w \in \Vset$, $f^{-1}(w)$ induces a subtree in $\Th$.)
    
      \label{step:destein1} 
      
          \BlankLine
        \tcp{Construct new tree on non-Steiner nodes}
      Let $\T'$ have vertex set $\Vset$. For each $u,v\in \Vset$, include edge $(u,v)$ in $E(\T')$ if there is a $u'\in f\inv (u)$ and $v'\in f\inv (v)$ such that $(u',v')\in E(\Th)$.  \label{step:newTree}
      
      Fix an arbitrary root $\rho \in \Vset$ and for all neighbors $u$ of $\rho$ let $\d'(\rho,u) = \dh(\rho, u)$. 
      
      Assign lengths to the edges of $\T'$ in breadth-first order as follows. When assigning length to a given edge $(u,v)$, note that all edges on the path between $\rho$ and $u$ have already been assigned a length, so $\d'(\rho, u)$ is well-defined. Let $\d'(u,v) = \max(0, \dh(\rho,v) - \d'(\rho,u))$. (Observe that $\d'(\rho,v)\geq \dh(\rho,v)$.)
      \label{step:assignLengths}
      
      Return $(\T', \d')$
   
    \caption{$\Desteinerize(\Th,\dh)$}
\label{alg:desteinerize}
\end{algorithm}

We now prove correctness of $\Desteinerize$. In what follows, let $\cV$ denote the non-Steiner nodes of $\Th$, and let $\cX = V(\Th)$ denote the Steiner nodes. Let $\crad$ denote the maximum distance in $\dh$ between a Steiner node to the closest non-Steiner node:
$$\crad = \max_{x \in \cX \sm \Vset} \min_{v \in \Vset} \dh(x,v).$$
Equivalently, $\crad$ is the smallest radius such that for every Steiner node $x\in \cX \sm \Vset$, there exists a non-Steiner node $v \in \Vset$ with $\dh(x,v) \leq \crad$. Theorem~\ref{thm:desteinerize} immediately follows from the next two lemmas. Essentially, Lemma~\ref{lem:desteinerize-guarantee} implies that the desteinerization procedure distorts distances by at most $O(\crad)$. Lemma~\ref{lem:steinernodeclose} guarantees that, since $\dh$ is a good approximation to a tree metric $\d$ without Steiner nodes, every Steiner node has a non-Steiner node at distance $O(\eps)$:
\begin{lemma}\label{lem:desteinerize-guarantee}
\textsc{Desteinerize} with input $\dh$ outputs a tree metric $\d'$ such that $|\d'(x,y) - \d(x,y)| \leq 12\crad$ for all $x,y \in \Vset$. 
\end{lemma}
\begin{lemma}\label{lem:steinernodeclose}
Let $\dh$ be an additive metric that approximates a tree metric $\d$ (without Steiner nodes) as in the setting of Theorem~\ref{thm:desteinerize}. Then, after Step \ref{step:destein0} of $\Desteinerize$, $\crad \leq 30\eps$.
\end{lemma}

\subsection{Proof of Lemma~\ref{lem:desteinerize-guarantee}: distances changed by $O(\crad)$}
Lemma~\ref{lem:desteinerize-guarantee} states that $\Desteinerize$ outputs a tree metric that distorts distances by at most $O(\crad)$. To prove it, we first show that the graph $\T'$ constructed by $\Desteinerize$ is indeed a tree.

\begin{claim}\label{claim:f-subtree}
In Step \ref{step:destein1} of \textsc{Desteinerize}, for any input additive metric $\dh$ and any $w \in \Vset$, the subgraph of $\Th$ induced by $f^{-1}(w)$, $\Th[f^{-1}(w)]$, is a subtree.
\end{claim}

\begin{proof} 
$\Th[f^{-1}(w)]$ is nonempty because $w \in f^{-1}(w)$. So to show that $\Th[f^{-1}(w)]$ is a subtree it suffices to show that this subgraph is connected. To this end, for each node $u\in f^{-1}(w)$, it suffices to show that all nodes on the path $\pih(u,w)\subseteq \Th$ between $u$ and $w$ are in $f^{-1}(w)$. Denote by $a_1=u,a_2,\dots, a_t=w$ the path $\pih(u,w)$ and let $i = \min\{j: f(a_j)\neq w\}$ be the first node that is \emph{not} mapped to $w$. Write $w' = f(a_i)$. Then by the triangle inequality, 
$$
\dh(a_{i-1},w')\leq \dh(a_i,w') + \dh(a_{i-1},a_i) \stackrel{(\star)}\leq \dh(a_i,w) + \dh(a_{i-1},a_i) = \dh(a_{i-1},w)\,,
$$ 
where the inequality $(\star)$ follows by the definition of $f$, and the last equality step uses that $a_{i-1},a_i, \dots, a_t=w$ form a path in $\Th$. If $(\star)$ is strict, then $\dh(a_{i-1},w')< \dh(a_{i-1},w)$ and this contradicts the definition of $i$ as being the first in the path that is not mapped to $w$. Otherwise, if $\dh(a_{i-1},w')= \dh(a_{i-1},w)$, then $(\star)$ is an equality. However, by the definition of $f$, this means that $w'$ is
at least as close as $w$ to $a_i$ in graph distance. If $w'$ is strictly closer than $w$ to $a_i$ in graph distance, then $w'$ is also strictly closer than $w$ to $a_i$ in graph distance, meaning that $f(a_{i-1}) = w'$, a contradiction. Otherwise, if $w'$ is as close as $w$ to $a_i$ in graph distance, then $w' < w$ because $f(a_i) = w'$, and moreover $w'$ is at least as close as $w$ to $a_{i-1}$ in graph distance, so $f(a_{i-1}) = w'$, a contradiction.
\end{proof}

\begin{claim}\label{claim:graphIsTree}
The graph $\T'$ constructed in Step~\ref{step:newTree} is a tree. 
\end{claim}
\begin{proof}
As shown in the previous claim, the sets $f\inv(w)$ for $w\in \Vset$ partition the vertex set of $\Th$ into subtrees $\Th[f\inv(w)]$. The graph $\T'$ is obtained from $\Th$ by contracting the sets $\Th[f\inv(w)]$. In general, contraction of subsets of vertices maintains connectivity. Moreover, considering the contraction of sets performed sequentially, a first such contraction of $\Th[f\inv(w)]$ leading to a cycle would necessarily imply existence of a cycle before contraction due to connectivity of $\Th[f\inv(w)]$, contradicting that $\Th$ is a tree.
\end{proof}

Now that we know by that $\T'$ is a tree and so $\d'$ is a tree metric, it remains to bound the error. We next show that all distances to $\rho$ are accurately approximated. 

\begin{claim}\label{claim:errortorootdesteinerize}
After completing Step~\ref{step:assignLengths}, for all $u\in \Vset$ we have $\d'(\rho,u) - 2\crad\leq  \dh(\rho,u)\leq \d'(\rho,u)$.
\end{claim}
\begin{proof} The bound $\dh(\rho,u)\leq \d'(\rho,u)$ is a consequence of the edge length assignment in step~\ref{step:assignLengths}, chosen specifically to ensure this inequality, so we proceed with proving the other inequality.

Consider an arbitrary node $u\in \Vset$ and denote by $a_0=\rho,\dots, a_t=u$ the path in $\T'$ between $\rho$ and $u$. If $\dh(\rho,u)\geq \d'(\rho,u)$ then there is nothing to prove, so assume that $\dh(\rho,u)< \d'(\rho,u)$.
Let $$i=\min\{i\in [t]: \dh(\rho,a_{j}) < \d'(\rho,a_j)\text{ for all } i\leq j\leq t \}$$
be the earliest index such that $a_i, a_{i+1},\dots, a_t$ are all too far from $\rho$ in metric $\d'$ as compared to $\dh$.
Note that $i\geq 2$ since $\dh(\rho,a_1) = \d'(\rho,a_1)$. Moreover, step~\ref{step:assignLengths} of $\Desteinerize$ assigns
$\d'(a_{j-1},a_{j})=0$ for $j$ in the range $i\leq j\leq t$, so $\d'(\rho,a_{i-1})=\d'(\rho,a_{i})=\cdots =\d'(\rho,a_t) $.

By step~\ref{step:newTree} of $\Desteinerize$, for each $j$, $i\leq j\leq t$, there is an edge $(a_{j-1},a_{j})$ only if the edge $(\bout_{j-1},\bin_{j})$ is in $\Th$ for some $\bout_{j-1}\in f\inv (a_{j-1}),\bin_{j} \in f\inv (a_{j})$, which implies that the path in $\Th$ from $\rho$ to $\bin_t$ passes through the sequence $\bout_{i-1}, \bin_{i},\bout_{i},\dots,\bout_{t-1},\bin_t$, whence $ \dh(\rho,\bout_{i-1})\leq \cdots\leq \dh(\rho,\bin_t)$. 

Finally, the definition of $\crad$ implies that $\dh(v,w) \leq \crad$ for any node $v \in f^{-1}(w)$, it follows that
\begin{align*}
  \dh(\rho,a_t)\geq \dh(\rho,\bin_t) - \crad \geq \dh(\rho,\bout_{i-1}) - \crad  & \geq  \d'(\rho,a_{i-1}) - 2\crad  \\&= \d'(\rho,a_{t}) - 2\crad\,,  
\end{align*}
which proves the claim. 
\end{proof}

To prove that the distance between any pair of nodes $u,v \in \Vset$ is well-approximated, we need to introduce some terminology.

\begin{definition}[Lowest common ancestor (LCA)]
Let $\T$ be a tree, $S$ be a subset of the vertices, and $\rho$ a (root) vertex. We define
$\LCAi \T{\rho}S$ to be the \emph{lowest common ancestor} of $S$ with respect to $\rho$.
More explicitly, it is the furthest vertex from $\rho$ in the intersection of all paths
from $u$ to $\rho$ for $u \in S$.
If $S = \{y,z\}$ is a set of size 2, we write $\LCAi\T\rho{\{y,z\}} = \LCAi\T\rho{y,z}$ as shorthand.
\end{definition}

The following crucial claim shows that if we know the distances between three nodes $u,v$, and $\rho$, then the distance of all of these points to their ``midpoint'' $\LCAi\T\rho{u,v} = \LCAi\T u{\rho,v} = \LCAi\T v{\rho,u}$ is also determined. This allows us to ``triangulate'' the location of nodes in the tree.
\begin{claim}\label{claim:crossing}
Let $\d$ be a metric on tree $\T$ rooted at node $\rho$. Then for any $u,v\in V(\T)$,
\[ \d(\LCAi\T\rho{u,v}, \rho) = \frac{\d(u,\rho) + \d(v,\rho) -\d(u,v)}{2}\,. \]
\end{claim}
\begin{proof}
This follows from $\d(u,\rho) = \d(u,\LCAi\T\rho{u,v}) + \d(\LCAi\T\rho{u,v},\rho)$, the matching expression for $\d(v,\rho)$, and $\d(u,v) = \d(u,\LCAi\T\rho{u,v}) + \d(\LCAi\T\rho{u,v}, v)$.
\end{proof}

We now complete the proof of Lemma~\ref{lem:desteinerize-guarantee} by applying Claim~\ref{claim:crossing} to both $\dh(u,v)$ and $\d'(u,v)$. 
As shown in Claim~\ref{claim:errortorootdesteinerize}, 
for every pair of nodes $u,v \in \Vset$, the constructed tree metric $\d'$ approximates distances $\dh(u,\rho)$ and $\dh(v,\rho)$ up to at most $2\crad$. We also claim that
\begin{align}\big|\dh(\LCAi{\Th}\rho{u,v}, \rho) - \d'(\LCAi{\T'}\rho{u,v}, \rho)\big|\leq 4\crad\,,\label{ineq:desteinerizerootlca}\end{align}
which implies the $|\dh(u,v) - \d'(u,v)| \leq 12\crad$ by Claim~\ref{claim:crossing}.
To show \eqref{ineq:desteinerizerootlca}, first
\begin{align*}
\big|\dh(\LCAi{\Th}\rho{u,v}, \rho) - \d'(\LCAi{\T'}\rho{u,v}, \rho)\big|&\leq \big|\dh(\LCAi{\Th}\rho{u,v}, \rho) - \dh(\LCAi{\T'}\rho{u,v}, \rho)  \big|
\\&\quad + \big| \dh(\LCAi{\T'}\rho{u,v}, \rho)  -\d'(\LCAi{\T'}\rho{u,v}, \rho)\big| \,.
\end{align*}
The second quantity is at most $2\crad$ by Claim~\ref{claim:errortorootdesteinerize} (these are distances from $\rho$ under $\dh$ and $\d'$ of the same node $\LCAi{\T'}\rho{u,v}\in \Vset$). For the first quantity, note that all paths in $\T'$ are the same as in $\Th$ if we contract the sets of vertices $f\inv (w)$ for all $w\in \Vset$. Thus, both  $\LCAi{\Th}\rho{u,v}$ and $\LCAi{\T'}\rho{u,v}$ lie in the same subtree $\Th[f\inv (w)]$ for some $w$, and the bound follows from the fact that these subtrees have diameter at most $2\crad$, since by construction all nodes in $\Th[f\inv (w)]$ are at distance at most $\crad$ from $w$. 
\qed

\subsection{Proof of Lemma~\ref{lem:steinernodeclose}: $\crad$ is small}\label{ssec:desteinerizeadditionallemmas}
Lemma~\ref{lem:steinernodeclose} states that if $\dh$ is a good estimate of a tree metric $\d$ in the sense of Theorem~\ref{thm:desteinerize}, then after Step~\ref{step:destein0} of $\Desteinerize$ every Steiner node of $\dh$ has a $O(\eps)$-close non-Steiner node. First, we prove the following technical claim.
\begin{claim}\label{claim:unique-path}
Suppose that $\d$ is a tree metric on tree $\T$ and that $y$ is on the unique path from $x$ to $z$. Suppose that $y'$ satisfies $|\d(y',x) - \d(y,x)| \le \epsilon$ and $|\d(y',z) - \d(y,z)| \le \epsilon$. Then $\d(y,y') \le \epsilon$.
\end{claim}
\begin{proof}
Consider the paths from $x$ to $y'$ and from $z$ to $y'$. One of these paths must pass through $y$, since their concatenation forms a (non-simple, i.e., with repeated nodes) path from $x$ to $z$. Without loss of generality assume this is the path from $x$ to $y'$. Then
\[ \d(x,y') = \d(x,y) + \d(y,y')\,, \]
which means that $\d(y,y') = \d(x,y') - \d(x,y) \le \epsilon$ using the assumption of the lemma. 
\end{proof}
Now we are ready to prove Lemma~\ref{lem:steinernodeclose}. Let $\T$ be the tree associated to tree metric $\d$ in the statement of Lemma~\ref{lem:steinernodeclose} and let $\Th$ be the tree associated with $\dh$ after step~\ref{step:destein0} of \Desteinerize.

\begin{proof}[Proof of Lemma~\ref{lem:steinernodeclose}]

Fix an arbitrary root $\rho\in \Vset$. Pick an arbitrary Steiner node $v \in \cX$ of $\dh$. The node $v$ must have at least two children $w_1,w_2$ with descendants in $\cV$, since otherwise Step~\ref{step:destein0} of $\Desteinerize$ would have removed $v$. We claim that in $\Th$ there must be at least one descendant $x_1 \in \Vset$ of $w_1$ such that $\dh(v,x_1) \le L + 2\epsilon$. To prove this, take an arbitrary node $x'_1 \in \Vset$ in the $\Th$-subtree under $w_1$, then by assumption there exists a path from $x'_1$ to $\rho$ in $\T$ using edges of length at most $L$, and by assumption these nodes are at most $(L + 2\epsilon)$-far apart in $\Th$; therefore going up this path, the last node which is a $\Th$-descendant of $w_1$ must satisfy the condition to be $x_1$. By the symmetrical argument, we can find a node $x_2$ descending from $w_2$ satisfying $\dh(v,x_2) \leq L+2\epsilon$. By a similar argument, we may find an ancestor $w \in \Vset$ of $v$ in the tree $\Th$ such that $\dh(v,w) \leq L+2\eps$.

Note that by construction $v = \LCAi\Th\rho{x_1,x_2}$, and furthermore that $v = \LCAi\Th w{x_1,x_2}$, since $w$ lies on the path from $\rho$ to $v$ in $\Th$. Define $v' = \LCAi\T w{x_1,x_2}$. By Claim~\ref{claim:crossing}, we have
\begin{equation}
\dh(v, w) = \frac{\dh(x_1,w) + \dh(x_2,w) - \dh(x_1,x_2)}{2} \mbox{ and } \d(v',w) = \frac{\d(x_1,w) + \d(x_2,w) - \d(x_1,x_2)}{2}.\label{eq:crossingidentities}\end{equation}
Now use the fact that for any $a,b \in \{x_1,x_2,w\}$, we have $\dh(a,b) \leq 2L + 4\eps$ by the triangle inequality. By applying \eqref{ineq:desteinerizeinputguarantee}, since $\eps / L \leq 1/2$ we obtain
$$\d(a,b) \leq (2L + 5\eps)/ (1 - \eps/L) \leq 4L + 10\eps,$$ and consequently applying \eqref{ineq:desteinerizeinputguarantee} again $$|\d(a,b) - \dh(a,b)| \leq ((4L + 10\eps)/L + 1)\epsilon \leq 10\epsilon.$$
Plugging this approximation guarantee into \eqref{eq:crossingidentities} yields
\begin{equation*}
|\dh(v,w) - \d(v',w)| \leq 15\eps
\end{equation*}

So $\d(v',w) \leq 15\eps + L + 2\eps \leq 2L$, and thus by \eqref{ineq:desteinerizeinputguarantee}, $|\d(v',w) - \dh(u,v)| \leq 4\eps$, which means that
\begin{equation}
|\dh(v,w) - \dh(v',w)| \leq 20\eps .
\label{eqn:v-rho-close}
\end{equation}

Note that $v'$ is on the $\T$ path from $x_1$ to $w$, so $\d(x_1,w) = \d(x_1,v') + \d(v',w),$ and also $v$ is on the $\Th$ path from $x_1$ to $w$, so $\dh(x_1,v) = \dh(x_1,v) + \dh(v,w)$. Combined with the fact that $|\dh(x_1,w) - \d(x_1,w)| \leq 10\eps$ and  \eqref{eqn:v-rho-close}, this means
\begin{equation}\label{eqn:v-y-close}
|\dh(v,x_1) - \dh(v',x_1)| \leq 30\eps.
\end{equation}
Therefore, applying Claim~\ref{claim:unique-path} with \eqref{eqn:v-rho-close} and \eqref{eqn:v-y-close} proves $\dh(v',v) \leq 30\epsilon.$ Since $v$ was arbitrary, $\crad \leq 30\epsilon$.
\end{proof}

\section{Discussion}

While this paper settles the problem of efficiently learning tree Ising models to within local total variation, including with model misspecification or adversarial modification of the data, a variety of questions remain. For one, we think of $k$ in the order of $\loctvk k$ as constant, and we leave open the task of efficiently learning a tree achieving optimal dependence on $k$ in the $\loctvk k$ error. 

We assumed an Ising model without external field and it would be interesting to generalize the results to allow for arbitrary external field or to arbitrary higher cardinality alphabets and continuous models such as Gaussian models. Beyond the binary alphabet case, the component of {\CLpp} based on tree metric reconstruction algorithms may need some new ingredients.

Perhaps the most compelling open direction is developing efficient algorithms for estimating non-tree graphical models within local total variation.

\appendix

\section{Proof of Proposition~\ref{prop:reducetoferromagnetic}}\label{app:toferromagnetic}

We prove Proposition~\ref{prop:reducetoferromagnetic}, which states that learning general models reduces to learning ferromagnetic models.
The reduction is given in Algorithm~\ref{alg:toferromagnetic}. 

\begin{algorithm}\SetAlgoLined
    \KwIn{$\eps$-approximate correlations $\mut{}$ for an unknown model $(\Ts,\thetas{})$, and $\eps > 0$}
    \KwOut{Model $(\Th,\thetah{})$ that $O(\eps)$-approximates $(\Ts,\thetas{})$ in $\loctvk2$ distance.}
    
    \BlankLine
    
    $(\Th,\theta') \gets \LearnFerromagneticModel(|\muh{}|,\eps)$
    
    $\thetah{e} \gets \sgn(\mut{e}) \cdot \theta'$ for all $e \in E(\Th)$

    \Return $(\Th, \thetah{})$
    \caption{$\ReduceToFerromagnetic(\mut,\eps)$, i.e., the {\CLpp} algorithm}
\label{alg:toferromagnetic}
\end{algorithm}
\begin{proposition}[Restatement of Proposition~\ref{prop:reducetoferromagnetic}] The $\epslearning$ problem for general models on $n$ vertices can be solved in $O(n^2)$ time and one call to an oracle $\LearnFerromagneticModel$ for the $\epslearning$ problem restricted to ferromagnetic models.
\end{proposition}

\begin{proof}[Proof of Proposition~\ref{prop:reducetoferromagnetic}]
The runtime guarantee is straightforward, as there are $n^2$ correlations $\mut{uv}$. For correctness, we begin by noting that $\big||\mut{uv}| - |\mus{uv}|\big|\leq |\mut{uv} -\mus{uv}|\leq \eps$ for each $u,v$, and that the ferromagnetic model $(\T, |\thetas{}|)$ has correlations $|\mus{uv}|$. In words, the entries of $|\mut{}|$ are within $\eps$ of the correlations for $(\T, |\thetas{}|)$ and therefore by the guarantee of $\LearnFerromagneticModel$ the returned model $(\Th,\theta')$ has pairwise correlations $\mup{}$ such that $\max_{u,v} \big||\mup{uv}| - |\mus{uv}|\big| \leq C\eps$ for some absolute constant $C$.

Let $\muh{}$ denote the correlations in the final model $(\Th,\thetah{})$.  
Since $|\thetah{}| = |\theta'|$ and hence $|\muh{}| = |\mup{}|$, we have 
\begin{equation}\label{eq:absCorr}
\max_{u,v} \big||\muh{uv}| - |\mus{uv}|\big| \leq C\eps.
\end{equation}
It remains to show that $\muh{uv}$ has roughly the same sign as $\mus{uv}$ for any $u,v$. We consider two cases. 

\paragraph{Case 1.} If $|\muh{uv}| \leq (C+1)\eps$, then we do not need the correct sign, since by the triangle inequality and \eqref{eq:absCorr} 
$$|\muh{uv} - \mus{uv}| \leq |\muh{uv}| + |\mus{uv}| \leq 2|\muh{uv}| + C\eps \leq (3C+2)\eps\,.$$

\paragraph{Case 2.} If $|\muh{uv}| > (C+1)\eps$, then it suffices to prove that $\sgn(\mus{uv}) = \sgn(\muh{uv})$, because then $$|\muh{uv} - \mus{uv}| = \big||\muh{uv}| - |\mus{uv}|\big| \leq C\eps\,.$$
We consider the vertices $w_1,\ldots,w_t$ on the path from $u$ to $v$ in $\Th$, where $u = w_1$ and $v = w_t$. We first show that the true correlations on this path are not too small: for any $i \in [t-1]$,
\begin{equation}\label{ineq:pathcorrlowerbound}
|\mus{w_i w_{i+1}}| \geq |\muh{w_i w_{i+1}}| - C\eps \geq |\muh{uv}| - C\eps > \eps\,.
\end{equation} 
The first inequality follows from \eqref{eq:absCorr}, and the second uses $|\muh{uv}| = \prod_{j=1}^{t-1} |\muh{w_j w_{j+1}}| \leq |\muh{w_i w_{i+1}}|$. 
Inequality~\eqref{ineq:pathcorrlowerbound} implies that the approximate correlations $\mut{}$ have the correct sign, that is,
\begin{equation} \label{eq:signsonpathequal}
\sgn(\mus{w_i w_{i+1}}) = \sgn(\mut{w_i w_{i+1}})\,.
\end{equation}
This yields
\begin{align*}\sgn(\mus{uv}) &= \prod_{i=1}^{t-1} \sgn(\mus{w_iw_{i+1}}) & \mbox{by Claim~\ref{claim:signsmultiply} since model is tree-structured} \\ 
&= \prod_{i=1}^{t-1} \sgn(\mut{w_i w_{i+1}}) & \mbox{by~\eqref{eq:signsonpathequal}} \\
&= \prod_{i=1}^{t-1} \sgn(\thetah{w_i w_{i+1}}) = \sgn(\muh{uv}) & \mbox{by construction of } (\Th,\thetah{})\,.
\end{align*}
We see that in both cases $|\mup{uv} - \mus{uv}| \leq (3C+2)\eps$.
\end{proof}

\begin{claim}\label{claim:signsmultiply}
Let $\mu$ the pairwise correlations for tree-structured model $(\T, \theta)$ with vertex set $\Vset$. Let $w_1,\ldots,w_t \in \Vset$ be some sequence of vertices. If $|\mu_{w_i w_{i+1}}| > 0$ for all $i \in [t-1]$, then 
$$
\sgn(\mu_{w_1 w_t}) = \prod_{i=1}^{t-1} \sgn(\mu_{w_i w_{i+1}})\,.
$$
\end{claim}

We emphasize that it is not the case that $\mu_{w_1 w_t} = \prod \mu_{w_i w_{i+1}}$ (or there would be nothing to prove).

\begin{proof}
For each $i \in [t-1]$, define $\pi_i \subset \T$ to be the path between $w_i$ and $w_{i+1}$. Let $\pi \subset \T$ be the path between $w_1$ and $w_t$. 
The key observation is that set $E(\pi)$ of edges in $\pi$ is the set of edges that occurs in an odd number of the paths $\pi_1,\ldots,\pi_{t-1}$. 
From this we obtain the desired statement, $$\sgn\left(\prod_{i=1}^{t-1}\mu_{w_i w_{i+1}}\right) = \sgn\left(\prod_{i=1}^{t-1} \prod_{e \in E(\pi_i)} \theta_e\right) =  \sgn\left(\prod_{e\in E(\pi)} \theta_e\right) = \sgn(\mu_{w_1 w_t})\,.$$
\end{proof}

\section{Prediction-Centric Learning Generalizes Structure Learning}
\label{sec:predImpliesStruct}
As discussed in the introduction, it is not possible to guarantee structure
recovery when the strength of interactions is unbounded: if we allow arbitrarily strong interactions, then any
spanning tree can perfectly represent the uniform distribution over the set of $n$ all-equal bits (when $n = 3$, the model $X \sim \text{Uni} \{(1,1,1),(-1,-1,-1)\}$); if we allow arbitrarily weak interactions, then very weak edges cannot be distinguished from non-edges with any bounded number of samples. 

In this section only, we consider the case where interactions are bounded, so structure recovery is possible. We show in this setting that prediction-centric learning automatically implies we get the correct structure. This justifies the claim that prediction-centric learning is the natural generalization of structure learning to models with arbitrarily weak or strong interactions, including settings with hard constraints or nearly independent variables. Afterwards, we show combining the reduction with our main result gives an algorithm with optimal sample complexity for structure learning.

Formally, we show the following:
\begin{theorem}\label{thm:prediction-to-structure}
Suppose that $(\T,\theta)$ is an Ising model $P$ on the tree $\T$ with edge weights $\theta$, and that for all edges $i \sim j$ that $|\theta_{ij}| \in [\alpha,1 - \beta]$ with $\alpha,\beta \in (0,1)$. Let $\epsilon = \alpha \beta/8$ and suppose  $(T', \theta')$ is any tree Ising model $P'$ with
$\loctvk{3}(P,P') < \epsilon$, then $\T = \T'$. 
\end{theorem}
\begin{proof}
We use the following notational convention: random vector $X$ denotes a sample from the Ising model $(\T,\theta)$ and random vector $X'$ denotes a sample from the Ising model $(\T',\theta')$. 

Suppose that $i \sim j$ in $\T$ but they are not neighbors in $\T'$, so there exists at least one other node $k$ which lies on the path from $i$ to $j$. Also suppose without loss of generality that $\E[X_i X_j] \ge 0$ (otherwise, flip the sign of $X_j$).
Using the local TV assumption, by Fact~\ref{fact:corrLocTV}, we have
\begin{equation}\label{eqn:structure-eqn}
\E[X_i X_j] - 2\epsilon \le \E[X'_i X'_j] = \E[X'_i X'_k] \E[X'_k X'_j] \le (\E[X_i X_k] + 2\epsilon)(\E[X_k X_j] + 2\epsilon). \end{equation}
In $\T$, either the path from $k$ to $j$ passes through $i$ or the path from $k$ to $i$ passes through $j$, depending on which node $k$ is closer to. Assume without loss of generality (by symmetry) that the former case holds: then $\E[X_j X_k] = \E[X_i X_j]\E[X_i X_k]$. Expanding the RHS of \eqref{eqn:structure-eqn} then gives
\[ (\E[X_i X_k] + 2\epsilon)(\E[X_i X_j]\E[X_i X_k] + 2\epsilon) = \E[X_i X_k]^2\E[X_i X_j] + 2\epsilon(\E[X_i X_k] + \E[X_i X_j]) + 4\epsilon^2 \]
and combining with \eqref{eqn:structure-eqn} we see that
\[ \E[X_i X_j] \le \E[X_i X_k]^2 \E[X_i X_j] + 8\epsilon \]
and so
\[ \alpha \le \E[X_i X_j] \le \frac{8 \epsilon}{1 - \E[X_i X_k]^2} \le \frac{8 \epsilon}{1 - |\E[X_i X_k]|} \le \frac{8 \epsilon}{\beta} \]
which gives a contradiction using $\epsilon < \alpha \beta/8$. 
\end{proof}
Recall from \eqref{eqn:loctv-k-comparison} that $\loctvk{3}$ and $\loctvk{2}$ are equivalent up to constants for tree models. Using this, we establish optimality of the result.
\begin{remark}[Optimality]
The following example shows the optimality, up to constants, of Theorem~\ref{thm:prediction-to-structure}. Suppose in $\T$ that $i$ is the parent of $j$ and $k$, $\theta_{ik} = 1 - \beta$ and $\theta_{ij} = \alpha$. Suppose that in $\T'$, $k$ is the parent of $i$ and $j$ with $\theta'_{ik} = 1 - \beta$ and $\theta'_{kj} = \alpha$. Then the two models match in all degree $2$ moments except two: $\E[X_iX_j] = \alpha$ whereas $\E[X'_i X'_j] = (1 - \beta) \alpha$, and $\E[X_j X_k] = (1 - \beta)\alpha$ whereas $\E[X'_j X'_k] = \alpha$. Since the two cases are symmetrical, by Fact~\ref{fact:corrLocTV} we see
\[ \loctvk{2}(P,P') = (\alpha - (1 - \beta)\alpha)/2 = \alpha \beta/2. \]
\end{remark}

Theorem~\ref{thm:prediction-to-structure} above combined with our {\CLpp} guarantee in Corollary~\ref{c:oracle2} implies that {\CLpp} successfully recovers the structure of a tree Ising model at the same optimal sample complexity as the Chow-Liu algorithm, up to constant factors. 

\begin{corollary}\label{cor:minimaxStructure}
Let $(\T,\theta)$ be a tree Ising model with edge correlations $|\theta_{ij}| \in [\alpha,1 - \beta]$ with $\alpha,\beta \in (0,1)$. There is a constant $C$ such that given $m\geq C \log (n/\delta) / \alpha^2\beta^2$ samples, {\CLpp} returns the correct tree $\T$ with probability at least $1-\delta$. 
\end{corollary}
\begin{proof} Let $\eps = \alpha \beta / 8$ and suppose the constant $C$ describing $m$ in the statement is sufficiently large.
By Corollary~\ref{c:oracle2}, the learned tree model $\Ph$ satisfies $\loctvk 3(\Ph, P)< c k 2^k \eps < \eps$, where $c$ can be chosen to be an arbitrarily small constant by growing $C$ and $P$ is the distribution of the original model $(\T,\theta)$. Theorem~\ref{thm:prediction-to-structure} now implies that the tree underlying $\Ph$ is correct. 
\end{proof}

Theorem 3.1 of \cite{breslerkarzand19}, when converted to our parameterization, shows that $(\log n)/\al^2\beta^2$ samples are necessary to reconstruct the correct tree with constant probability. Thus the guarantee in the corollary just above proves the sample complexity is optimal up to constant factors.

\section{Failure of pure tree metric reconstruction approach}
\label{ssec:treemetricfails}
As explained in Section~\ref{ssec:chowliuoverview} above, our final learning algorithm combines the tree metric reconstruction based algorithm with the Chow-Liu algorithm, where the Chow-Liu algorithm is used as a kind of ``clustering algorithm'' which handles edges with weak correlation (large evolutionary distance). We already saw in Section~\ref{sec:CLfails} that purely using Chow-Liu fails (in a quite strong sense) to prove a guarantee similar to Theorem~\ref{t:CL++}. In this section, we give a complementary result, showing that the general tree metric reconstruction approach outlined in Section~\ref{ssec:tmroverview} cannot by itself obtain the guarantee of Theorem~\ref{t:CL++} --- in other words, the Chow-Liu preprocessing step is truly necessary to handle weak/long edges. 

There are obvious technical issues trying to adapt the analysis of \textsc{AdditiveMetricReconstruction} to deal with misspecification, since $\delta$-misspecification (or equivalent finite sample limitations) mean that evolutionary distances of size $\Omega(\log(1/\delta))$ are not at all accurate compared to the target tree model. Beyond those issues in the analysis, we next explain two different ways in which the pure tree metric reconstruction algorithm will fail, even with an infinite number of samples, both related to difficulties handling long edges; these failures show that if we wanted to remove the Chow-Liu step, then we would at the very least have to make the tree metric reconstruction procedure more complex.

\paragraph{First cause of failure: infinite distances.} Since the tree metric reconstruction algorithm relies on reconstructing the evolutionary distance $\d(u,v) = \log(1/\mu_{uv})$, the most obvious point of failure is that the chosen root node $\rho$ may have correlation $\mu_{\rho u} = 0$ with all other nodes $u$, in which case the algorithm fails when run with population values for the correlations. Indeed, the evolutionary distance satisfies $\d(\rho,u) = \infty$ and the centroid metric constructed by the algorithm satisfies $\csf(u,v) = \infty$ for all $u,v \neq \rho$. The algorithm fails since it must subtract an infinite value from an infinite value in order to output the reconstruction. 

\paragraph{Second cause of failure: desteinerization and latent variables.} The failure of the tree metric reconstruction algorithm is not simply due to the pathological case when there are perfectly-uncorrelated vertices. We demonstrate the issue with a counterexample of a different nature, where pairwise correlations are all nonzero.  

We start by describing a distribution $P$ over $\{\pm 1\}$-valued random variables $X,Y_1,Y_2$ with $\E[X] = \E[Y_1] = \E[Y_2] = 0$ and $\E[Y_1 Y_2] = 1/4$, $\E[X Y_1] = \E[X Y_2] = \delta$. This distribution is not a tree Ising model, but it can be realized as the marginal law of a tree Ising model with an additional latent variable (i.e. Steiner node in the metric interpretation) $Z$:
\begin{enumerate}
    \item $X$ is sampled from a Rademacher distribution, i.e. uniform on $\{\pm 1\}$.
    \item $Z$ equals $X$ with probability $2\delta$, otherwise $Z$ is sampled independently from a Rademacher distribution.
    \item With probability $1/2$, $Y_1 = Z$ and otherwise $Y_1$ is sampled from a Rademacher distribution. $Y_2$ is generated in the same way.
\end{enumerate}
It's clear that $P$ is within $\loctvk2$ distance $\delta$ of a tree Ising model where $X$ is independent of $Y_1,Y_2$. We now explain why $P$ is a problematic example for the pure tree metric reconstruction algorithm.

Outside of the present work, reconstruction algorithms based on the evolutionary metric are largely used to reconstruct tree models with latent variables, see e.g. \cite{daskalakis2011evolutionary}, in which case it is considered a virtue that the output metric includes the Steiner nodes. However, in the present context it is a problem: when our algorithm runs the initial tree metric reconstruction step (\textsc{AdditiveMetricReconstruction}), the output will contain the Steiner node $Z$ which is at distance $\log(1/2)$ from $Y_1,Y_2$ and at distance $\log(1/\delta)$ from $X$. Essentially, the problem is that the Steiner node is not close to any of the observable nodes, and this is an issue as the final output cannot have latent variables. More precisely, there is no way to desteinerize this tree in the sense of either (1) deleting the Steiner node if it unecessary (i.e. of degree $2$) or (2) replacing the steiner node/latent variable with an observed variable while keeping the other distances in the tree unchanged (as in \textsc{Desteinerize}): replacing $Z$ by any of $Y_1,Y_2$, or $X$ will result in a tree Ising model with $\loctvk2$ distance $\Omega(1)$ from $P$. For example, replacing $Z$ by $Y_1$ results in a model with $\E[Y_1 Y_2] = 1/2$ instead of $1/4$.

\section{Optimal Lower bound for TV Learning}
\label{sec:TVlower}
In recent work, Devroye et al \cite{devroye2020minimax} studied minimax rates for learning various kinds of graphical models in Total Variation (TV) distance. They proved that for a tree Ising model on $n$ nodes, the minimax rate for reconstructing the tree model in TV from $m$ samples is upper bounded by $O\left(\sqrt{\frac{n \log(n)}{m}}\right)$ and posed the tightness of this result as an open question. In this Appendix, we resolve this open question by proving a matching information-theoretic lower bound, showing that the minimax rate is $\Theta\left(\sqrt{\frac{n \log n}{m}}\right)$ up to a universal constant. 

\paragraph{Construction.}
Recall from Stirling's formula that $\log(n!) \sim n\log(n)$. This motivates the following
simple construction for the lower bound:
\begin{enumerate}
    \item $\mathcal{S}$ is a family of permutations on $[n]$ to be specified later.
    \item Pick a permutation $\pi$ from family of permutations $\mathcal{S}$.
    \item Build an Ising model on the matching graph with covariance $\alpha/\sqrt{n}$ between vertices $i$ and $n + \pi(i)$ for every $i \in [n]$. (This corresponds to edge weight essentially $\alpha/\sqrt{n}$ since the correlation for edge weight $\beta$ is $\tanh(\beta) \approx \beta$)
\end{enumerate}
This yields a family of distributions $\{P_{\pi}\}_{\pi \in \mathcal{S}}$ on $2n$ nodes.

For the set of permutations, we choose a set which satisfies $\log |\mathcal{S}| = \Omega(n \log(n))$ and such that every set of permutations has Hamming distance at least $n/4$, where the Hamming distance is $|\{x: \pi(x) \ne \pi'(x)\}|$. In the survey of Quistorff \cite{quistorff2006survey}, %
such a result is given as (6) and attributed to Deza. 

\paragraph{Analysis.}
The proof reduces to the following claims:
\begin{enumerate}
    \item Every two elements have large total variation distance. More specifically, to distinguish models from $\pi_1,\pi_2$ look at statistic $\sum_{i = 1}^n X_i X_{\pi_1(i)}$. Under $P_{\pi_1}$ the expectation is $\alpha \sqrt{n}$ and variance is $\Theta(n)$, whereas under $P_{\pi_2}$ the expectation is less than $(3/4) \alpha \sqrt{n}$ and variance is $\Theta(n)$. Applying the Central Limit Theorem, this shows the total variation distance is $\Omega(\alpha)$.
    \item Every two elements have reasonably small KL. Recall by the Gibbs variational formula \cite{ellis2007entropy} that
    \[ KL(P_{\pi_1},P_{\pi_2}) = (\frac{1}{2} \E_{\pi_2}[ X^T J_{\pi_2} X] + H_{\pi_2}(X))  - (\frac{1}{2} \E_{\pi_1}[ X^T J_{\pi_2} X] + H_{\pi_2}(X)). \]
    Since the entropies are the same, this is just a difference of expectations. By similar reasoning to above, it is of order $\Theta(\alpha^2)$ (there are $n/4$ missing edges, and they each contribute $(\alpha/\sqrt{n}) \cdot (c \alpha/\sqrt{n})$).
    \item Any algorithm given $m \le C_2 n\log(n)/\alpha^2$ samples fails to reconstruct with probability at least $1/2$. Since between any two models the KL for $m$ samples is of order $m \alpha^2$ by tensorization and (2), and $\log |\mathcal{S}| \sim n \log(n)$, this follows directly from Fano's inequality (see e.g. \cite{rigollet2015high}).
\end{enumerate}
Combining these claims shows that $\sqrt{n\log(n)/m}$ is the tight rate for learning tree Ising models in TV distance. 

\printbibliography
\end{document}